\documentclass[11pt]{article}

\usepackage[english]{babel}
\usepackage[utf8x]{inputenc}
\usepackage[T1]{fontenc}

\normalsize

\usepackage[a4paper,top=1in,bottom=1in,left=1in,right=1in,marginparwidth=1in]{geometry}

\usepackage{natbib}
\usepackage{setspace}
\usepackage{amsmath}
\usepackage{amssymb} 
\usepackage{amsthm}
\usepackage{amsfonts, mathtools}
\usepackage{algorithm,algpseudocode}
\usepackage{bm}
\usepackage{comment}
\usepackage{graphicx}
\usepackage{multirow, booktabs}
\usepackage{caption}
\usepackage{subcaption}
\usepackage[colorinlistoftodos]{todonotes}
\usepackage{hyperref}
\hypersetup{colorlinks=true, allcolors=blue}

\usepackage{multirow}
\usepackage{textgreek}
\usepackage{adjustbox}
\usepackage[normalem]{ulem}

\theoremstyle{plain}
\newtheorem{theorem}{Theorem}

\newtheorem{lemma}{Lemma}

\theoremstyle{definition}

\newtheorem{assumption}{Assumption}

\theoremstyle{remark}
\newtheorem{remark}{Remark}

\definecolor{rose}{rgb}{1.0, 0.33, 0.64}
\renewcommand{\thefootnote}{\fnsymbol{footnote}}

\title{Learning Shortest Paths When Data is Scarce}

\author{Dmytro Matsypura, Yu Pan, Hanzhao Wang}
\date{\small Discipline of Business Analytics, The University of Sydney\\
January 2026}

\begin{document}
\maketitle
\onehalfspacing
\def\thefootnote{}\relax\footnotetext{ Correspondence to 
yu.pan@sydney.edu.au
}
\begin{abstract}
Digital twins and other simulators are increasingly used to support routing decisions in large-scale networks. However, simulator outputs often exhibit systematic bias, while ground-truth measurements are costly and scarce. We study a stochastic shortest-path problem in which a planner has access to abundant synthetic samples, limited real-world observations, and an edge-similarity structure capturing expected behavioral similarity across links. We model the simulator-to-reality discrepancy as an unknown, edge-specific bias that varies smoothly over the similarity graph, and estimate it using Laplacian-regularized least squares. This approach yields calibrated edge cost estimates even in data-scarce regimes. We establish finite-sample error bounds, translate estimation error into path-level suboptimality guarantees, and propose a computable, data-driven certificate that verifies near-optimality of a candidate route. For cold-start settings without initial real data, we develop a bias-aware active learning algorithm that leverages the simulator and adaptively selects edges to measure until a prescribed accuracy is met. Numerical experiments on multiple road networks and traffic graphs further demonstrate the effectiveness of our methods.
\end{abstract}

\section{Introduction}
\label{sec:intro}
%\dmytro{I suggest numbering Assumptions, Lemmas and Theorems separately, throughout. Otherwise, the numbering creates an impression that there are LOTS of assumptions and theorems.

%Also, there is a number of incomplete references with missing fields, like \cite{kipf2017gcn,hamilton2017inductive,velickovic2018gat,zhou2004learning}. They need to be reviewed.

%In Appendix B, I suggest having subsections for each proof. Currently, it is inconsistent with subsection B.4 containing several proofs. 

%I also placed a number of comments throughout the text.}

Shortest-path models are a workhorse of operations management, underpinning routing, logistics, and service-network decisions in freight transportation, last-mile delivery, ride-hailing, and emergency response. In many of these applications, firms repeatedly solve shortest-path problems on large infrastructure networks while rarely observing the relevant edge costs under the operational conditions of interest. To fill this gap, organizations increasingly rely on simulators and digital twins to generate synthetic edge cost data at scale and under counterfactual scenarios \citep{simdata_greasley_owen_2018,simdata_wang_et_al_2022,simdata_mahmoodi_et_al_2024}. The appeal is clear: synthetic data are abundant, cheap to generate, and easy to tailor to new planning scenarios. The challenge is that simulation models are inevitably misspecified—inputs are estimated from finite or outdated data, model structure is simplified, and exogenous shocks or behavioral shifts are hard to capture \citep{simdata_johnson_mollaghasemi_1994,simdata_goeva_lam_qian_zhang_2019}. As a result, a path that is optimal in the simulator can be substantially suboptimal in the real system, even when the simulator appears accurate on average.

This paper asks how an operations planner should combine a biased-but-cheap simulator with scarce-but-reliable real measurements when the goal is to find a shortest path in the real network. We study a stochastic shortest-path problem on a graph where each edge has an unknown mean cost. The planner observes (i) a small number of noisy real measurements on some edges, (ii) a large number of synthetic samples on every edge from a simulator whose mean costs may be systematically biased, and (iii) an edge-similarity structure that captures domain knowledge about which links are expected to behave similarly (for example, adjacent road segments, or edges with comparable physical or operational attributes). Real measurements from sensing, field experiments, or controlled operational trials are costly and often incomplete, whereas synthetic samples from a digital twin are essentially free. The central challenge is therefore to leverage simulation outputs without trusting them: scarce real data must be used to calibrate and interrogate the simulator in a way that is tailored to the downstream shortest-path decision.

Our approach treats the simulator as a structured, noisy prior for real edge costs. We model the difference between real and simulated mean costs on each edge as an unknown \emph{bias} term, and assume this bias varies smoothly over an auxiliary similarity graph on edges. We then estimate the bias by solving a convex Laplacian-regularized least-squares problem that anchors to the simulator while pooling information across similar edges. This estimator (i) provides calibrated mean-cost estimates even when some edges have no real observations, (ii) admits principled weighting based on sampling variances, and (iii) is computationally tractable on large networks. We further connect calibration accuracy to decision quality. Using the calibrated edge means as plug-in costs, we analyze the resulting estimated shortest path and derive path-level suboptimality bounds in terms of edgewise estimation errors. 

Finally, we address a cold-start setting in which no real data are initially available and the planner must decide what to measure. We propose an \emph{Active Estimated Shortest Path} (\textsc{A-ESP}) procedure that interleaves (i) bias-aware edge cost estimation with (ii) adaptive data collection guided by pathwise confidence intervals. \textsc{A-ESP} exploits the similarity graph to share information across edges and uses a simple variance-balancing sampling rule to prioritize informative measurements. We prove that \textsc{A-ESP} identifies the true shortest path with high probability and provide sample-complexity guarantees that quantify how problem difficulty depends on network size, path length, optimality gaps, simulator-bias smoothness, and noise levels.

Empirically, we evaluate our framework on multiple road-network datasets, including urban road graphs and sensor-based traffic graphs. Across these settings, our methods improve edge cost calibration and shortest-path performance when real data are scarce and simulator bias exhibits exploitable structure, while naturally converging to real-data performance when abundant real measurements make calibration unnecessary. Taken together, our results provide actionable guidance on when and how simulation-based digital twins can be safely and effectively used for routing decisions.

% The remainder of the paper is organized as follows. Section~\ref{sec:review} reviews the literature. Section~\ref{sec:setup} introduces the stochastic shortest-path model, the data structure, and the edge-similarity graph. Section~\ref{sec:edge_cost_est} presents our proposed estimator and its finite-sample analysis, and Section~\ref{sec:sp_greedy_edgewise} analyzes the suboptimality of the estimated shortest path. Section~\ref{sec:active} develops \textsc{A-ESP} and establishes its sample-complexity guarantees for the cold-start setting. Section~\ref{subsec:experiments} reports our empirical study, and the Appendix collects additional proofs and implementation details.
The remainder of the paper is organized as follows.
In Section~\ref{sec:review}, we review the related literature.
In Section~\ref{sec:setup}, we introduce the stochastic shortest-path model, the data structure, and the edge-similarity graph.
In Section~\ref{sec:edge_cost_est}, we present our estimator and its finite-sample analysis, and in Section~\ref{sec:sp_greedy_edgewise}, we analyze the suboptimality of the resulting estimated shortest path.
In Section~\ref{sec:active}, we develop \textsc{A-ESP} and establish its sample-complexity guarantees for the cold-start setting.
In Section~\ref{subsec:experiments}, we report our empirical study, and the Appendix collects additional proofs and implementation details.

\section{Related Work}
\label{sec:review}
\subsection{Graph-Based Learning}
\label{subsec:graph-learning}

Our modeling of edgewise simulator bias as a smooth function over an edge-similarity graph is closely related to the literature on graph-based learning.
One key component of our method is leveraging potentially biased synthetic data through Laplacian regularization.
A large body of work uses the graph Laplacian to encode a smoothness prior and regularize prediction or representation learning problems.
Early work on Laplacian eigenmaps shows how to embed high-dimensional data by constructing a neighborhood graph and using the graph Laplacian to preserve local topology in a low-dimensional representation \citep{belkin2003laplacian}.
In parallel, \citet{smola2003kernels} developed kernels and regularization operators directly on graphs, generalizing classical Tikhonov regularization to graph-structured domains.
Semi-supervised learning methods such as Gaussian fields and harmonic functions \citep{zhu2003gaussian} and local-global consistency algorithms \citep{zhou2004learning} treat unknown labels as a graph signal and learn them by minimizing a quadratic form involving the Laplacian.
Manifold regularization formalizes this idea by adding an intrinsic Laplacian penalty to standard kernel methods, i.e., learning in a reproducing kernel Hilbert space (RKHS), thereby exploiting the geometry of the marginal data distribution \citep{belkin2006manifold}.
\citet{ando2006learning} analyze transductive classification on graph vertices using a Laplacian penalty, with careful attention to normalization and spectral dimensionality reduction.
In contrast to these previous studies, we focus on adapting synthetic data to real data and analyze how key factors (e.g., the degree of bias) affect estimation of the unknown parameter; this analysis then informs our choice of a shortest path.
In our framework, the Laplacian penalty pools information and calibrates the synthetic and real datasets, yielding an estimation-with-optimization pipeline that is distinct from previous works.

Graph signal processing (GSP) provides a complementary viewpoint by interpreting unknown quantities on a network as signals defined on the vertices of a graph and designing filters in the spectral domain of the Laplacian \citep{shuman2013spg,ortega2018gsp}. Within this framework, piecewise-smooth signals are denoised or interpolated by penalizing graph-Laplacian energy, which is conceptually similar to the regularization we impose on edge biases. Building on these ideas, graph neural networks (GNNs) implement learnable, localized filters on graphs via spectral and spatial convolutions \citep{defferrard2016cnn,kipf2017gcn}. Subsequent architectures such as GraphSAGE \citep{hamilton2017inductive} and graph attention networks \citep{velickovic2018gat} extend this message-passing paradigm to large, dynamic, or inductive settings, while the graph network framework of \citet{battaglia2018relational} emphasizes relational inductive biases for reasoning over entities and their interactions. Compared with these expressive parametric models, our approach adopts a simpler convex optimization framework tailored to operations on a fixed network: rather than learning high-dimensional node representations, we regularize a scalar correction on each edge of a known similarity graph and analyze how this impacts downstream shortest-path decisions under real and synthetic data.

There is also a line of work on graph-structured bandits and online learning with feedback graphs that is related to our active querying of paths. Spectral bandit algorithms assume that rewards are smooth over a known graph and use the eigenstructure of the Laplacian to design exploration strategies and regret bounds \citep{valko2014spectral}. More general feedback-graph models allow the learner’s observation structure to be specified by a graph, capturing settings where pulling one arm reveals information about others \citep{alon2015feedback}. Networked and collaborative bandit formulations leverage a graph over users or tasks to share information and speed up learning \citep{cesa-bianchi2013gang}. Our active procedure similarly exploits smoothness on an edge-similarity graph, but differs in two key respects: (i) the reward-relevant quantity is a path cost induced by edgewise corrections to a biased simulator rather than a stand-alone arm reward, and (ii) the graph structure comes from domain knowledge about edge similarity rather than from observed feedback alone. This combination of graph-regularized estimation, path-centric objectives, and real-synthetic data integration appears distinct from existing graph-learning and graph-bandit formulations.

\subsection{Operations Management with Simulation}
\label{subsec:simdata-om}

Simulation has long been a core methodology in operations management (OM). Early survey work documents how simulation models have been used to study supply chains, production systems, and service operations, typically as stand-alone ``virtual laboratories'' fed by expert-elicited parameters and limited empirical data \citep{kleijnen2005supplyChain,henderson2006simulationHandbook}. \cite{shafer2004empirical} review empirically grounded simulation studies in OM and highlight that most models are calibrated to case-specific data but then treated as ground truth for evaluating alternative policies. In supply-chain contexts, \cite{kleijnen2005supplyChain} emphasizes the role of discrete-event simulation among other paradigms and discusses methodological issues such as validation, sensitivity analysis, and robustness when simulation outputs are used to support managerial decisions. In all of these settings, simulation output effectively becomes synthetic performance data on which OM researchers and practitioners optimize schedules, capacity levels, or routing policies, yet the statistical properties of this synthetic data are rarely modeled explicitly.

Optimization via simulation provides a systematic framework for turning such synthetic data into prescriptive decisions. Classic reviews by \cite{fu1994optimization} and \cite{fu2005simulationOptimization} synthesize algorithms for simulation optimization, ranking-and-selection, and metaheuristic search when the system can only be evaluated through noisy simulation experiments. Within OM, these methods underpin applications such as call-center staffing, inventory control, and network design, where the simulator is assumed to be a high-fidelity representation of the real system and the main challenge is to allocate computational budget across candidate designs. More recent work pushes towards explicitly data-driven simulation: \cite{tannock2007dataDriven} model a supply chain using data-driven simulation to capture empirical demand and lead-time distributions, while \cite{desantis2023edCalibration} use simulation-based optimization to calibrate service-time parameters of an emergency-department discrete-event simulation so that key performance indicators match those observed in hospital data. In parallel,  \cite{hu2023dynamicDataDriven,hu2024dataAssimilation} argue that streaming observational data should continually update simulation state and parameters, treating the simulator as a digital twin whose fidelity is maintained online rather than through static, off-line calibration.

Digital-twin thinking has also begun to reshape simulation practice in logistics and production systems. \cite{agalianos2020desDigitalTwin} review the integration of simulation with sensor-driven digital twins for warehouse and logistics operations, identifying the need for architectures that keep simulation models synchronized with real-time operational data while still supporting fast scenario analysis. Across these lines of work, simulation models generate abundant synthetic observations that are used to evaluate routing, capacity, or scheduling decisions, but the synthetic data are rarely combined formally with sparse ground-truth measurements when optimizing over network-structured decisions. The framework developed in this paper contributes to this OM-focused simulation literature by treating the simulator as a potentially biased but inexpensive generator of synthetic edge cost data on a network. We explicitly model simulator bias as a smooth latent field over the edge-similarity graph, estimate it from limited real edge observations using Laplacian regularization, and propagate the resulting uncertainty to shortest-path decisions.

\subsection{Data-Driven Optimization}
\label{subsec:data-driven-optimization}

Our work also contributes to the broader literature on data-driven and prescriptive optimization, which uses observed data as the primary input to optimization models rather than assuming a fully specified stochastic structure. \citet{bertsimas_kallus_2020_prescriptive} develop a general prescriptive analytics framework that learns decision policies as functions of contextual covariates, establishing asymptotic optimality relative to an oracle with full distributional knowledge. \citet{bertsimas_koduri_2022_rkhs} propose a reproducing-kernel-Hilbert-space approach that directly approximates optimal objective values and decision rules from data, with finite-sample guarantees. In retail and supply chain settings, \citet{qi_mak_shen_2020_retail_review} survey data-driven methods for assortment, fulfillment, and inventory decisions, while \citet{ban_rudin_2019_big_data_newsvendor} work on the data-driven newsvendor problem to illustrate how machine learning can be embedded into classical inventory models. More recently, \citet{gupta_kallus_2022_data_pooling} show that pooling data across many seemingly unrelated stochastic optimization problems can strictly improve performance relative to solving each problem in isolation. These papers share our goal of using data to directly inform operational decisions, but they typically assume that available data faithfully reflect the target environment (for example, all relevant outcomes are observed and measurement bias is limited), whereas our setting explicitly starts from a potentially biased simulator (and observed real data if available).

A complementary line of work embeds statistical uncertainty about the data-generating process directly into the optimization formulation via distributionally robust or data-driven stochastic optimization. \citet{bertsimas_gupta_kallus_2018_ddro} design uncertainty sets for robust optimization from hypothesis tests, yielding solutions with finite-sample probabilistic guarantees. A rich stream of papers uses Wasserstein balls or related ambiguity sets to construct data-driven distributionally robust models, including the seminal work of \citet{esfahani_kuhn_2018_wasserstein_dro} and subsequent extensions to risk-averse formulations such as \citet{zhao_guan_2018_wasserstein_risk_averse}. \citet{rahimian_mehrotra_2022_dro_review} survey these frameworks and clarify their connections to classical robust optimization, chance constraints, and statistical learning. Conceptually, our use of an imperfect simulator together with limited real data is closest in spirit to this literature: both aim to guard against model misspecification and sampling error. 

A third related stream couples prediction and optimization even more tightly by training predictive models to be directly tailored to the downstream decision problem. The Smart ``Predict, then Optimize'' framework of \citet{elmachtoub_grigas_2022_spo} introduces a regret-based loss that measures the decision error induced by a prediction, and \citet{elbalghiti_elmachtoub_grigas_tewari_2023_pto_bounds} provide generalization guarantees for this predict-then-optimize paradigm. In machine-learning and AI applications, \citet{wilder_dilkina_tambe_2019_decision_focused} advocate \emph{decision-focused learning}, in which combinatorial optimization problems are embedded inside the training loop so that models are trained to perform well on the ultimate decision task rather than on a generic prediction metric. In parallel, the policy-learning literature in econometrics and statistics, exemplified by \citet{athey_wager_2021_policy_learning}, seeks to learn treatment-assignment or pricing policies directly from observational data with causal guarantees. Our approach is complementary to these decision-focused frameworks: we likewise care about downstream decision quality rather than pure predictive accuracy, but instead of treating the predictive model as a black box, we exploit the structure of the underlying graph and show how to calibrate the potentially biased simulation data with real data so that it yields provably accurate predictions for the shortest path problem.

\section{Problem Setup}
\label{sec:setup}
We consider the (stochastic) shortest-path problem on a graph \(G=(V,E)\), where $V$ denotes the set of nodes and $E$ denotes the set of edges. We denote a source node \(v_\text{src}\in V\) and a sink node \(v_\text{sink} \in V\) of the graph $G$ as the start node and the terminal node we are interested in. Each edge \(e\in E\) has a random cost
\[
C_e \;=\; \mu_e + \varepsilon_e,
\]
where \(\mu_e\) is an unknown mean cost and \(\varepsilon_e\) is a zero-mean noise term. Let \(\mathcal{P}\) denote the set of all simple \(v_\text{src}\text{--}v_\text{sink}\) paths (no repeated vertices). The goal is to find a path with the smallest expected cost:
\[
P^\star \in \arg\min_{P\in\mathcal{P}} \mu(P),
\quad \text{where } \mu(P) \coloneqq \sum_{e\in P} \mu_e .
\]

\subsection{Calibrating Shortest Paths with Real and Synthetic Data}
Although the means \(\mu_e\) are unknown, we may have observed real samples on each edge (e.g., measured travel times collected from vehicles or deliveries between two cities). Specifically, for each \(e\in E\) we have a set of samples
\[
\mathcal{S}_e=\{c_{e,i}\}_{i=1}^{n_e}, \qquad n_e\ge 0,
\]
where each observation satisfies \(c_{e,i}=\mu_e+\varepsilon_{e,i}\) and the \(\{\varepsilon_{e,i}\}\) are zero-mean noise terms. The case \(n_e=0\) indicates no real data are available for edge \(e\).

In addition, we have access to a set of synthetic samples produced by a simulator (e.g., a digital twin). For each \(e\in E\),
\[
\mathcal{S}'_e=\{c'_{e,i}\}_{i=1}^{n'_e}, \qquad
c'_{e,i}=\mu'_e+\varepsilon'_{e,i},
\]
where \(\mu'_e\) is the simulator’s (unknown) mean and \(\varepsilon'_{e,i}\) is a zero-mean noise term. In general, the simulator may be biased so that \(\mu'_e \neq \mu_e\). Because synthetic data are typically cheaper to generate, we commonly have \(n'_e \gg n_e\).

This work studies how to leverage the limited real data \(\{\mathcal{S}_e\}_{e\in E}\) together with the abundant synthetic data \(\{\mathcal{S}'_e\}_{e\in E}\) to (approximately) recover the shortest path \(P^\star\).

\subsection{Edge Similarity}
We assume access to an edge-similarity matrix \(W\in \mathbb{R}^{|E|\times |E|}\) for the graph \(G\), where \(W_{e,e'}\) quantifies the similarity between edges \(e\in E\) and \(e'\in E\). In settings such as workflow graphs with semantic node labels or spatial road networks, one may compute similarities from edge feature vectors \(x_e\) (e.g., traffic intensity, length). A simple choice is a kernel-valued similarity \(W_{e,e'}=\mathrm{k}(x_e,x_{e'})\) with, for example, the Gaussian kernel \(\mathrm{k}(x_e,x_{e'})=\exp(-\| x_e-x_{e'}\|_2^2)\). As an alternative, \(W\) can be derived from structural heuristics (e.g., inverse-degree-based weights). We will assume \(W\) is entrywise nonnegative and (without loss of generality) symmetric.

\section{Edge Cost Estimation using Laplacian Regularization}
\label{sec:edge_cost_est}
In this section, we develop an algorithm and accompanying analysis to estimate the simulator bias from the available data (equivalently, the unknown edge means \(\mu_e\)).

\subsection{Parameter Estimation}
We estimate the \emph{real} edge means \(\mu_e\) by calibrating abundant synthetic information with real observations. For each edge \(e\), let
\[
\bar c_e := \frac{1}{n_e}\sum_{i=1}^{n_e} c_{e,i}
\qquad\text{and}\qquad
\bar c'_e := \frac{1}{n'_e}\sum_{i=1}^{n'_e} c'_{e,i}
\]
denote the empirical means of the real and synthetic samples, respectively (when \(n_e=0\), \(\bar c_e\) is undefined and the corresponding data-fidelity term in \eqref{eqn:b_hat} below is dropped by setting \(w_e=0\)).

%denote the empirical means of the real and synthetic samples, respectively (when \(n_e=0\), \(\bar c_e\) is undefined and the corresponding data-fidelity term below \dmytro{below where?} will be dropped via a zero weight).

A naive estimator uses only real data, i.e., \(\bar c_e\) as an estimate of \(\mu_e\). However, limited or missing real samples (\(n_e\) small or zero) can yield high error. 

\paragraph{Leveraging synthetic data.}
When synthetic samples are plentiful (e.g., inexpensive to generate), the law of large numbers implies \(\bar c'_e \to \mu'_e\), providing an accurate estimate of the simulator mean. Directly using \(\bar c'_e\) to estimate \(\mu_e\) can be biased because, in general, \(\mu'_e \neq \mu_e\). Let the (unknown) bias for edge $e\in E$ be
\[
b^\star_e := \mu_e - \mu'_e .
\]
Our goal is to estimate \(b^\star_e\) through an estimator $\hat b_e$ and then form the calibrated estimator
\[
\hat\mu_e \;:=\; \bar c'_e + \hat b_e .
\]

If \(n_e\) is sufficiently large, an unbiased plug-in estimate of the bias is \(\hat b_e=\bar c_e-\bar c'_e\) (equivalently, \(\bar c_e\) itself is already a good estimator of \(\mu_e\)). However, for edges with small \(n_e\), this direct estimate is unreliable. 
To share information across edges, we use the similarity matrix \(W\) and estimate the bias vector \(b=(b_e)_{e\in E}\) by solving the Laplacian-regularized least-squares problem
\begin{equation}
\label{eqn:b_hat}
\hat{b}\;\in\;\arg\min_{b\in\mathbb{R}^{|E|}}
\;\sum_{e\in E} w_e\,\bigl(b_e-(\bar c_e-\bar c'_e)\bigr)^2
\;+\;\lambda\, b^\top L b,
\end{equation}
where \(L=\mathrm{diag}\!\bigl(\sum_{e'} W_{e,e'}\bigr)-W\) is the graph Laplacian associated with \(W\), and
\(
b^\top L b=\tfrac12\sum_{e,e'} W_{e,e'}(b_e-b_{e'})^2.
\)
The weights \(w_e\ge 0\) encode the reliability of the real-data estimate on edge \(e\): we take \(w_e=0\) when \(n_e=0\) and typically let \(w_e\) increase with \(n_e\).
We elaborate on the choice of \(w_e\) below.
The calibrated edge-mean estimates are then \(\hat\mu_e=\bar c'_e+\hat b_e\), which can subsequently be used as edge weights to compute an (approximately) shortest path.

% To share information across edges, we use the similarity matrix \(W\) and estimate the bias vector \(b=(b_e)_{e\in E}\) via the graph-regularized objective
% \[
% \hat{b}\;\in\;\arg\min_{b\in\mathbb{R}^{|E|}}
% \;\sum_{e\in E} w_e\,\bigl(b_e-(\bar c_e-\bar c'_e)\bigr)^2
% \;+\;\frac{\lambda}{2}\sum_{e,e'\in E} W_{e,e'}\,(b_e-b_{e'})^2 .
% \]
% Here, the weights \(w_e\ge 0\) encode the reliability of the real-data estimate on edge \(e\) (we take \(w_e=0\) when \(n_e=0\); for \(n_e>0\), \(w_e\) can increase with \(n_e\) to assign more ``importance'' in the estimation for the edges with more real samples). We will elaborate \dmytrorev{on} the choice of $w_e$ later. \(\lambda>0\) controls the strength of the smoothness prior that similar edges (large \(W_{e,e'}\)) should have similar biases (small $(b_e-b_{e'})^2$).

% Equivalently, using the Laplacian regularization,
% \dmytro{I suggest writing this as an optimization problem, rather than the solution to an optimization problem.}
% \begin{equation}
% \label{eqn:b_hat}
% \hat{b}\;\in\;\arg\min_{b\in\mathbb{R}^{|E|}}
% \;\sum_{e\in E} w_e\,\bigl(b_e-(\bar c_e-\bar c'_e)\bigr)^2
% \;+\;\lambda\, b^\top L b,
% \end{equation}
% where \(L=\mathrm{diag}\!\bigl(\sum_{e'} W_{e,e'}\bigr)-W\) is the graph Laplacian associated with \(W\), and
% \(
% b^\top L b=\tfrac12\sum_{e,e'} W_{e,e'}(b_e-b_{e'})^2.
% \)
% The calibrated edge-mean estimates are then \(\hat\mu_e=\bar c'_e+\hat b_e\), which can subsequently be used as edge weights to compute an (approximately) shortest path.

\subsubsection{Discussion}
We highlight three practical aspects of \eqref{eqn:b_hat}.

\begin{itemize}
\item \textbf{Choice of the weights \(w_e\).}
A statistically natural choice is to weight each data-fit term by the inverse variance of the plug-in target \(\bar c_e-\bar c'_e\). Under the standard independence assumption for the real and synthetic noises,
\[
\mathrm{Var}(\bar c_e-\bar c'_e)\;=\;\frac{\mathrm{Var}(\varepsilon_{e,i})}{n_e}\;+\;\frac{\mathrm{Var}(\varepsilon'_{e,i})}{n'_e},
\]
so we set
\[
w_e \;\approx\; \frac{1}{\,\widehat{\mathrm{Var}}(\bar c_e-\bar c'_e)\,}
\;\approx\;
\frac{1}{\,\widehat{\mathrm{Var}}(\varepsilon_{e,i})/n_e \;+\; \widehat{\mathrm{Var}}(\varepsilon'_{e,i})/n'_e\,},
\]
with the empirical variances $\widehat{\mathrm{Var}}(\cdot)$ estimated from \(\mathcal{S}_e\) and \(\mathcal{S}'_e\) by the usual sample-variance formulas. Intuitively, larger variance means we trust the empirical difference \(\bar c_e-\bar c'_e\) less as an estimate of the true bias \(\mu_e-\mu'_e\), so we place less weight on that edge. In the extreme case \(n_e=0\), the variance of \(\bar c_e-\bar c'_e\) is effectively infinite and we take \(w_e=0\), i.e., the real-data term for edge \(e\) is dropped. To avoid unstable, overly large weights when \(n_e\) is very small, it is helpful to cap \(w_e\) from above (equivalently, impose a small positive uncertainty on the estimated variance).

\item \textbf{Choice of the regularization strength \(\lambda\).}
The parameter \(\lambda\) controls how aggressively we enforce that similar edges (large \(W_{e,e'}\)) have similar biases \(b_e\). When \(\lambda=0\), the problem decouples across edges: for edges with \(n_e>0\), \(\hat b_e\) reduces to the plug-in estimate \(\bar c_e-\bar c'_e\) and \(\hat\mu_e=\bar c_e\); for edges with \(n_e=0\), the objective contains no information and \(\hat\mu_e\) is essentially unidentifiable without additional structure. As \(\lambda\) increases, the solution is shrunk toward vectors that are nearly constant on each connected component of \(W\); if \(W\) is connected, this corresponds to learning (approximately) a single global simulator-to-real shift shared across edges. An equivalent constrained-form view is
\[
\min_{b \in\mathbb{R}^{|E|}} \;\sum_{e\in E} w_e\,\bigl(b_e-(\bar c_e-\bar c'_e)\bigr)^2 
\quad\text{s.t.}\quad b^\top L b \le B^2,
\]
where \(\lambda\) corresponds to a Lagrange multiplier for a constraint and \(B^2\) decreases with \(\lambda\). The constant $B^2$ reflects the smoothness of the bias distribution, 
% \sout{and details will be discussed in the analysis later} 
which we discuss in detail in Section~\ref{subsec:edge_estimate_analysis}. This further indicates that \(\lambda\) trades off variance (from noisy \(\bar c_e-\bar c'_e\)) and bias (from enforcing smoothness). We analyze the effect of \(\lambda\) on prediction error in the next subsection, and provide practical tuning strategies in the Appendix \ref{appx:lambda_tuning}.

\item \textbf{Role of the regularizer and the synthetic mean.}
Writing \(\hat\mu_e=\bar c'_e+\hat b_e\) and optimizing directly over \(\tilde\mu_e=\bar c'_e+b_e\), \eqref{eqn:b_hat} is equivalent to
\[
\hat{\mu}\;\in\;\arg\min_{\tilde{\mu}\in\mathbb{R}^{|E|}}
\;\sum_{e\in E} w_e\,\bigl(\tilde{\mu}_e-\bar c_e\bigr)^2
\;+\;\lambda\,(\tilde{\mu}-\bar c')^\top L\,(\tilde{\mu}-\bar c').
\]
Thus, the regularizer explicitly pulls \(\tilde\mu\) toward the simulator mean \(\bar c'\) along directions deemed smooth by \(W\). If \(\bar c'\) carries no meaningful information (e.g., treat \(\bar c'=0\)), the procedure reduces to Laplacian smoothing of the real-data means across edges, enabling imputations on edges with \(n_e=0\) via their neighbors. At the other extreme, if the simulator is perfectly calibrated (\(\bar c'=\mu\)), then \((\tilde{\mu}-\bar c')^\top L(\tilde{\mu}-\bar c')\) is minimized when \(\tilde\mu\) is aligned with \(\mu\); any positive data weights \(w_e\) anchor these constants so that \(\hat\mu\) collapses to \(\bar c'=\mu\). The analysis below further quantifies how the (unknown) bias \(b^\star=\mu-\mu'\) and the geometry of \(W\) (or $L$) together determine the estimation error.
\end{itemize}
\subsubsection{Analysis}
\label{subsec:edge_estimate_analysis}
%In this section, we derive the optimizer of \eqref{eqn:b_hat} \dmytro{this only makes sense if (1) is written as an optimization problem.} and analyze its estimation error as a function of the bias, the similarity matrix, the sample sizes, and the hyperparameter $\lambda$. All proofs are deferred to Appendix \ref{appx:proofs}.
In this subsection, we analyze the solution to \eqref{eqn:b_hat} and derive estimation-error bounds as a function of the bias, the similarity matrix, the sample sizes, and the hyperparameter \(\lambda\). All proofs are deferred to Appendix~\ref{appx:proofs}.

We begin with compact notation. Define $y\in\mathbb{R}^{|E|}$ with entries $y_e\coloneqq \bar c_e-\bar c'_e$ as the empirical bias vector (we can set $\bar{c}_e$ as an arbitrary number if $n_e=0$; this case will not influence the estimation since we will add a zero weight, i.e., $w_e=0$), and let
$M\coloneqq\mathrm{diag}(w_e)\in\mathbb{R}^{|E|\times|E|}$ denote the diagonal weight matrix. With this notation, \eqref{eqn:b_hat} becomes
\[
\hat b \in \arg\min_{b\in\mathbb{R}^{|E|}} \ (b-y)^\top M(b-y) + \lambda\, b^\top L b.
\tag{\ref{eqn:b_hat} revisited}
\]

%The problem admits a unique closed form solution. \dmytro{(1 revisited) is a solution to a problem, not a problem.}
The following lemma shows that, under a mild observability condition, this objective is strictly convex and admits a unique minimizer in closed form.

\begin{lemma}[Uniqueness and closed form]\label{lem:uniqueM}
For any $\lambda>0$, if each connected component of the edge‑adjacency graph (which induces $W$) contains at least one edge with $w_e>0$ and $W$ is symmetric with nonnegative entries, the objective is strictly convex and admits the unique minimizer
\[
\hat b \;=\; (M+\lambda L)^{-1} M\,y.
\]
\end{lemma}

Lemma \ref{lem:uniqueM} implies that \eqref{eqn:b_hat} can be solved efficiently via the closed form expression above. We now state the assumptions used to analyze the estimation error of $\hat b$ relative to $b^\star$.

% \dmytro{The previous sentence and the one after are repetitive. Let's pick one. I suggest the one above.}

% \paragraph{Assumptions for estimation.}

% We impose the following assumptions throughout the analysis.

\begin{assumption}[Noise and weights]\label{ass:noiseM}
For each edge $e\in E$, the aggregated measurement noise
$\xi_e\coloneqq y_e-(\mu_e-\mu'_e)$ is independent and centered sub-Gaussian with variance proxy $\nu_e^2$.
The weights satisfy
\[
\kappa_-\,\nu_e^{-2}\ \le\ w_e\ \le\ \kappa_+\,\nu_e^{-2},\qquad \text{and }w_e=0 \text{ if } n_e=0,
\]
for constants $0<\kappa_-\le \kappa_+<\infty$.
\end{assumption}

\noindent
Assumption~\ref{ass:noiseM} requires the weights $w_e$ to reflect the inverse noise level. Specifically, suppose that for edge $e$ we observe $n_e$ i.i.d.\ real samples $c_{e,i}=\mu_e+\varepsilon_{e,i}$ and $n'_e$ i.i.d.\ synthetic samples $c'_{e,i}=\mu'_e+\varepsilon'_{e,i}$ with centered sub-Gaussian noise proxies $\sigma_e^2$ and $\sigma_e'^2$. Then $\xi_e$ is centered sub-Gaussian with proxy
$\nu_e^2 \asymp \sigma_e^2/n_e+\sigma_e'^2/n'_e$. Hence larger sample sizes justify larger weights $w_e$. Intuitively, more observations for edge $e$ yield a more reliable estimate, so the corresponding weight should be larger.

\begin{assumption}[Bias smoothness]\label{ass:smoothM}
%Let $b^\star\in\mathbb{R}^{|E|}$ denote the unknown bias \dmytrorev{vector}, \sout{where} \dmytrorev{with entries} $b^\star_e=\mu_e-\mu'_e$.
Let $b^\star\in\mathbb{R}^{|E|}$ denote the unknown bias vector, with entries $b^\star_e=\mu_e-\mu'_e$.
We assume Laplacian smoothness:
\[
\|b^\star\|_{L}^2 \;\coloneqq\; {b^\star}^\top L b^\star \ \le\ B^2,
\]
for some finite $B>0$.
\end{assumption}

\noindent
Assumption~\ref{ass:smoothM} bounds the true bias under the Laplacian seminorm induced by $L$. We later show how the level $B$ influences the estimation error for $\mu_e$. In the special case where $b^\star$ is constant across edges, we have $\|b^\star\|_{L}=0$. In particular, if $b^\star\equiv 0$, the bias vanishes. Since the Laplacian seminorm is invariant to global shifts, adding a constant to all entries of $b^\star$ does not change $\|b^\star\|_{L}$. In our setting, $b^\star$ collects edgewise biases. The Laplacian penalty measures how these biases vary across the similarity graph, not their absolute level. A uniform offset represents the same bias everywhere and therefore has zero ``roughness.'' The absolute level is instead anchored by the data-fit term $(b-y)^\top M(b-y)$ in our estimation, i.e., matching bias $b$ with the observed $y$.

\paragraph{Main results.}
We now present our main estimation error bounds. We begin with notation:
\[
H\coloneqq M^{-1/2}LM^{-1/2}\succeq 0,
\qquad
S_\lambda\coloneqq (I+\lambda H)^{-1},
\]
and, for each edge $e$ with positive weight ($E^+\coloneqq\{e\in E:\ w_e>0\}$),
\[
\alpha_e(\lambda)\coloneqq \|S_\lambda M^{-1/2}e_e\|_2,
\qquad
\alpha_\infty(\lambda)\coloneqq \max_{e\in E^+}\alpha_e(\lambda),
\qquad
w_{\min}\coloneqq \min_{e\in E^+} w_e,
\]
where $e_e$ denotes the $e$th standard basis vector. The matrix $H=M^{-1/2}LM^{-1/2}$ is the graph Laplacian $L$ expressed in \emph{noise units}: directions that are rough on the similarity graph and poorly supported by data (small $w_e$) have larger eigenvalues. The operator $S_\lambda=(I+\lambda H)^{-1}$ acts as a graph smoother; it diffuses information across adjacent edges while damping high-frequency (nonsmooth) components in proportion to $\lambda$. Consequently, $\alpha_e(\lambda)=\|S_\lambda M^{-1/2}e_e\|_2$ measures the \emph{leverage} of noise on edge $e$ after smoothing; smaller values indicate that perturbations at $e$ have weaker global impact on $\hat b$. Note here we assume $M$ invertible, i.e., $w_e>0$ for all $e\in E$. If not, we can use its Moore-Penrose inverse \citep{penrose1955generalized} of $M^{1/2}$ or add a small scalar $\delta>0$ to its diagonal to keep $w_e>0$ for all $e \in E$.

\begin{theorem}[High-probability estimation bound]\label{thm:edgewise}
Under Assumptions~\ref{ass:noiseM} and \ref{ass:smoothM}, for any $e\in E^+$ and any $\delta\in(0,1)$,
\[
\mathbb{P}\!\left(
|\hat b_e - b^\star_e|
\;\le\;
\underbrace{\tfrac{\sqrt{\lambda}}{2}\,B\,w_e^{-1/2}}_{\text{bias}}
\;+\;
\underbrace{\,\sqrt{\kappa_+}\,\alpha_e(\lambda)\,\sqrt{2\log(2/\delta)}}_{\text{variance}}
\right)\;\ge\;1-\delta.
\]
Furthermore, we have a worst-case bound
\[
\mathbb{P}\!\left(
\max_{e\in E^+}|\hat b_e - b^\star_e|
\;\le\;
\tfrac{\sqrt{\lambda}}{2}\,B\,w_{\min}^{-1/2}
\;+\;
\,\sqrt{\kappa_+}\,\alpha_\infty(\lambda)\,\sqrt{2\log\!\bigl(2|E|/\delta\bigr)}
\right)\;\ge\;1-\delta.
\]
\end{theorem}

Because
\[
\hat\mu_e-\mu_e
\;=\;
\bigl(\hat b_e-b^\star_e\bigr)\;+\;\bigl(\bar c'_e-\mu'_e\bigr),
\]
Theorem~\ref{thm:edgewise} directly controls the first term, while the second typically contributes an additional $\mathcal{O}\!\bigl(\sigma'_e/\sqrt{n'_e}\bigr)$ that is negligible when synthetic samples are plentiful. Therefore, Theorem~\ref{thm:edgewise} yields a bound for the cost estimation error $|\hat\mu_e-\mu_e|$.

To interpret these terms, the ``bias'' component $\sqrt{\lambda}\,B/(2\sqrt{w_e})$ captures the effect of the bias level $B$ (that is, the quality of the synthetic data): smoother $b^\star$ (smaller $B$) reduces this contribution. Under Assumption~\ref{ass:noiseM} with i.i.d.\ sampling, $w_e\asymp (\sigma_e^2/n_e+\sigma_e'^2/n'_e)^{-1}$, increasing either $n_e$ or $n'_e$ enlarges $w_e$. This both decreases the bias coefficient $w_e^{-1/2}$ and tightens the variance through $w_{\min}$ and $\alpha_\infty(\lambda)$. In the common regime with abundant synthetic data, $w_e\propto n_e$, so well-sampled edges strongly anchor their neighborhoods.

In addition, writing the spectral decomposition $H=U\mathrm{diag}(\eta_k)U^\top$ with $k$-th eigenvalue $\eta_k\ge 0$ gives $S_\lambda=U\mathrm{diag}\bigl((1+\lambda\eta_k)^{-1}\bigr)U^\top$ and motivates the effective dimension
\[
d_{\mathrm{eff}}(\lambda)\;\coloneqq\;\mathrm{tr}\bigl(S_\lambda^2\bigr)\;=\;\sum_k (1+\lambda\eta_k)^{-2},
\]
for which, by definition,  the worst-case leverage satisfies
\[
\alpha_\infty(\lambda)^2 \;\le\; \mathrm{tr}\!\bigl(S_\lambda^2 M^{-1}\bigr)\;\le\; \frac{d_{\mathrm{eff}}(\lambda)}{\,w_{\min}\,}.
\]
Hence the variance terms in Theorem~\ref{thm:edgewise} shrink as $d_{\mathrm{eff}}(\lambda)$ decreases (stronger smoothing or a tighter graph) and as $w_{\min}$ grows (more data on at least one edge per component). In contrast, the bias terms scale as $\sqrt{\lambda}\,B/\sqrt{w_e}$ (edgewise) or $\sqrt{\lambda}\,B/\sqrt{w_{\min}}$ (uniform), which reveals the classical bias-variance trade-off: larger $\lambda$ enforces smoother estimates when the true bias is smooth (small $B$), whereas excessive smoothing inflates the bias when $B$ is large.

\textbf{Choice and influence of $\lambda$.} Balancing the uniform terms in Theorem~\ref{thm:edgewise} suggests selecting $\lambda$ that solves
\[
\frac{\sqrt{\lambda}}{2}\,B\,w_{\min}^{-1/2}
\;\approx\;
\sqrt{\frac{\kappa_+\,d_{\mathrm{eff}}(\lambda)}{w_{\min}}}\,\sqrt{\log\!\bigl(2|E|/\delta\bigr)}.
\]
However, because this is a (conservative) upper bound, a data-driven choice can be preferable in practice (see Appendix~\ref{appx:lambda_tuning}).

\subsection{Comparison with real-only error bounds.}
Compare our high-probability edgewise bound for $\hat\mu_e$,\[
|\hat\mu_e-\mu_e|
\;\le\;
\underbrace{\frac{\sqrt{\lambda}}{2\sqrt{w_e}}\,B}_{\text{smoothness bias}}
\;+\;
\underbrace{\sqrt{\kappa_+}\,\alpha_e(\lambda)\,\sqrt{2\log(2/\delta)}}_{\text{variance after graph smoothing}}
,
\]
with the standard plug-in bound for the \emph{real-only} estimator $\tilde\mu_e:=\bar c_e$ (e.g., from standard concentration inequality for sub-Gaussian % \dmytro{sub-Gaussian what?}
random variables \citep{vershynin2018high}):
\[
|\tilde\mu_e-\mu_e|
\;\le\;
\underbrace{\sqrt{\kappa_+}\,w_e^{-1/2}\,\sqrt{2\log(2/\delta)}}_{\text{variance without smoothing}},\qquad
\text{where }~w_e\asymp\Big(\sigma_e^2/n_e\Big)^{-1}.
\]
Since $S_\lambda=(I+\lambda H)^{-1}$ is a contraction, $\alpha_e(\lambda)=\|S_\lambda M^{-1/2}e_e\|\le w_e^{-1/2}$, so smoothing never \emph{increases} the variance factor. Our estimator strictly improves the bound over real-only whenever
\[
\frac{\sqrt{\lambda}}{2\sqrt{w_e}}\,B
\;\le\;
\sqrt{\kappa_+}\,\sqrt{2\log(2/\delta)}\Big(w_e^{-1/2}-\alpha_e(\lambda)\Big),
\]
i.e., when the simulator bias is \emph{smooth} (small $B$) and the similarity graph with regularization strength $\lambda$ meaningfully reduces leverage $\alpha_e(\lambda)$ below $w_e^{-1/2}$. Intuitively, the Laplacian couples edges so that poorly observed ones ``borrow strength'' from well-observed neighbors, shrinking the variance term, and the cost is a controllable smoothness bias that grows only like $\sqrt{\lambda}\,B/\sqrt{w_e}$. Thus, a small bias cap $B$ can benefit the bias term and the variance term. Conversely, if real data are abundant on every edge and $b^\star$ is highly irregular across $W$, then $\alpha_e(\lambda)\approx w_e^{-1/2}$ and the real-only bound becomes competitive, suggesting a small $\lambda$.

\section{From Edge Cost Estimates to Shortest Paths}
\label{sec:sp_greedy_edgewise}

Having obtained calibrated edge-mean estimates $\hat{\mu}_e$, we now turn to our goal: finding the shortest $v_{\text{src}}$--$v_{\text{sink}}$ path. The most natural and straightforward approach is to treat our estimates as the true edge costs and apply a standard shortest-path algorithm. In this section, we formalize this simple plug-in procedure, analyze its performance, and develop a practical, data-driven certificate to quantify the quality of the returned path.

Specifically, given the calibrated bias \(\hat b_e\), we set the estimated weight of each edge to \(\hat \mu_e = \bar c'_e + \hat b_e\). We then run a standard shortest path routine on these estimates to obtain an estimated shortest path, as detailed in Algorithm \ref{alg:esp_real}. Thus, we treat \(\{\hat \mu_e\}_{e\in E}\) as the true costs and apply any standard algorithm, such as Dijkstra's algorithm \citep{dijkstra2022note} for nonnegative weights or the Bellman--Ford algorithm \citep{bellman1958routing} otherwise, to obtain the shortest path.
\begin{algorithm}[H]
\caption{Estimated Shortest Path (ESP)}
 \label{alg:esp_real}
 \begin{algorithmic}[1]
 \Require Graph \(G=(V,E)\), source \(v_\text{src}\), sink \(v_\text{sink}\); real samples \(\{\mathcal S_e\}_{e\in E}\); synthetic samples \(\{\mathcal S'_e\}_{e\in E}\); similarity matrix \(W\) (with Laplacian \(L\)); regularization \(\lambda>0\).
 \State Compute edge-mean estimates $\hat{\mu}_e$ using Laplacian regularization.
 \State Compute a shortest path \(\widehat P\) from \(v_\text{src}\) to \(v_\text{sink}\) using the estimated means \(\hat \mu_e\) as edge weights. For example, use Dijkstra's algorithm if all \(\hat \mu_e \ge 0\) or Bellman--Ford otherwise.
 %\State Let \(\widehat P \in \arg\min_{P\in\mathcal{P}} \sum_{e\in P} \hat \mu_e\) be the returned path.
 \State \textbf{return} \(\widehat P\).
 \end{algorithmic}
 \end{algorithm}

\subsection{Performance Analysis}
In this subsection, we analyze the suboptimality gap of Algorithm \ref{alg:esp_real}, defined as the difference between the true cost of the estimated path, \(\mu(\widehat P)=\sum_{e\in \widehat P}\mu_e\), and the true cost of the optimal path, \(\mu(P^\star)=\sum_{e\in P^\star}\mu_e\).

Throughout the analysis we assume access to abundant synthetic data, so \(n'_e \to \infty\). Consequently \(\bar c'_e \to \mu'_e\). The same arguments extend to the finite synthetic data setting by adding an explicit bound on the sampling error \(|\bar c'_e - \mu'_e|\), for example via standard concentration results \citep{vershynin2018high}. Under \(n'_e \to \infty\) we have
\[
\hat\mu_e - \mu_e \;=\; (\hat b_e - b^\star_e) + (\bar c'_e - \mu'_e) \;=\; \hat b_e - b^\star_e.
\]
Hence the high probability bound on the per edge estimation error \(|\hat b_e - b^\star_e|\) from Theorem \ref{thm:edgewise} provides the basis for a path-level guarantee.

% Specifically, we assume there exists an estimation error bound $\beta_e\geq 0$ for each edge $e$ such that
% \begin{equation}
% \label{eq:edgewise_bound_generic}
% |\hat\mu_e - \mu_e| \;\le\; \beta_e \qquad \text{for all } e\in E,
% \end{equation}
% then Theorem \ref{thm:edgewise} combined with a union bound implies that with probability at least \(1-\delta\) for all $e\in E$ we can set \dmytro{what does this mean? are we making a probabilistic choice of $\beta$?}
% \[
% \beta_e \;=\; \tfrac{\sqrt{\lambda}}{2}\,B\,w_e^{-1/2}
% \;+\;
% \sqrt{\kappa_+}\,\alpha_e(\lambda)\,\sqrt{2\log\!\bigl(2|E|/\delta\bigr)}.
% \]
Specifically, suppose we have edgewise radii \(\beta_e\ge 0\) such that
\begin{equation}
\label{eq:edgewise_bound_generic}
|\hat\mu_e - \mu_e| \;\le\; \beta_e \qquad \text{for all } e\in E.
\end{equation}
In our setting, Theorem~\ref{thm:edgewise} and a union bound imply that the choice
\[
\beta_e \;\coloneqq\; \tfrac{\sqrt{\lambda}}{2}\,B\,w_e^{-1/2}
\;+\;
\sqrt{\kappa_+}\,\alpha_e(\lambda)\,\sqrt{2\log\!\bigl(2|E|/\delta\bigr)}
\]
satisfies \eqref{eq:edgewise_bound_generic} simultaneously for all edges \(e\in E\) with probability at least \(1-\delta\).

Let \(L_{\max} = \max_{P \in \mathcal P} |E(P)|\) denote the maximum number of edges in any path from $v_\text{src}$ to $v_\text{sink}$ in the original graph \(G\), where $E(P) \subseteq E$ represents the set of the edges on path $P$. The following result bounds the suboptimality gap.

\begin{theorem}[Path Suboptimality from Edgewise Errors]
\label{thm:two_path_bound}
Assume \eqref{eq:edgewise_bound_generic} holds. The path \(\widehat P\) returned by Algorithm \ref{alg:esp_real} satisfies
\[
\mu(\widehat P)-\mu(P^\star)
\;\le\;
\sum_{e\in E(P^\star)} \beta_e \;+\; \sum_{e\in E(\widehat P)} \beta_e.
\]
In particular, if there exists $\bar\beta$ such that \(\beta_e\le \bar\beta\) for all \(e\), then
\(\mu(\widehat P)-\mu(P^\star) \le 2\,\bar\beta\,L_{\max}\).
\end{theorem}

Theorem \ref{thm:two_path_bound} provides a performance guarantee for Algorithm \ref{alg:esp_real}. First, it translates edgewise estimation errors \(\{\beta_e\}\) into a bound on the decision error of the shortest path, which formalizes the intuition that better estimation yields better decisions. In the extreme case \(\beta_e \equiv 0\) for all edges, the estimated path \(\widehat P\) coincides with the true optimal path. Second, the guarantee shows that the quality of the estimated shortest path is not equally sensitive to errors on all edges. Instead, it depends mainly on the errors along the edges that lie on \(P^\star\) and on \(\widehat P\). For instance, if the edges on \(P^\star\) have large weights \(w_e\) that reflect abundant, high quality real data, the associated error bounds \(\beta_e\) are small, which in turn yields a tight suboptimality bound. In practice, this suggests that data collection is most valuable when directed toward the most critical or frequently traversed segments of the graph, which motivates the next section on active learning for shortest paths with synthetic data.
\section{Active Shortest Path Learning with a Biased Simulator}
\label{sec:active}

We now study a \emph{cold-start} setting in which there are initially no real observations on any edge,
\[
n_e=0 \quad \forall\, e\in E,
\]
while a (possibly biased) simulator provides abundant synthetic samples \(\{c'_{e,i}\}_{i=1}^{n'_e}\) with \(n'_e \gg 1\).
This situation arises, for example, when (i) a digital twin exists for an infrastructure network but field measurements have not yet been collected; (ii) a severe distribution shift (e.g., work zones, seasonal effects, or demand changes) renders historical logs stale and necessitates re-measurement; or (iii) real queries correspond to expensive interventions or measurements (e.g., field tests or controlled experiments) and must therefore be used sparingly.
Our goal is to \emph{adaptively} combine the biased but cheap simulator with a small number of real measurements to quickly identify an optimal (or near-optimal) $v_\text{src}$--$v_\text{sink}$ path under the \emph{true} mean costs.

\subsection{Setup}

The setting mirrors Section~\ref{sec:setup} except that at time \(t=0\) there is no real data. We use $n_e(t)$ to denote the number of collected real observations for edge $e\in E$ at time $t$. Initially, \(n_e(0)=0\) for all \(e\in E\).
For each edge \(e\), a simulator supplies a (large) batch of synthetic samples
\[
\mathcal{S}'_e=\{c'_{e,i}\}_{i=1}^{n'_e},
\]
which we use both to estimate the simulator mean and to calibrate its bias as in Section~\ref{sec:edge_cost_est}.
At each round \(t\ge1\), an adaptive algorithm chooses an edge \(e_t\in E\) based on the filtration generated by all past real observations, the synthetic samples \(\{\mathcal{S}'_e\}_{e\in E}\), and the similarity matrix \(W\), and then observes a fresh real sample
\[
c_{e_t,\,n_{e_t}(t)+1} \;=\; \mu_{e_t} + \varepsilon_{e_t,\,n_{e_t}(t)+1},
\]
where \(\varepsilon_{e_t,\cdot}\) is the real noise.
We write
\[
n_e(t)\;=\;\sum_{s=1}^t \bm 1\{e_s=e\},\qquad
\bar c_e(t)\;=\;\frac{1}{n_e(t)}\sum_{i=1}^{n_e(t)} c_{e,i},
\]
for the real sample count and running real-sample mean on edge \(e\) (with the convention that \(\bar c_e(t)\) is undefined when \(n_e(t)=0\)).
Let $T$ denote either a fixed budget of real queries or a (data-dependent) stopping time, and let \(\widehat{P}_T\in\mathcal{P}\) be the path returned by the algorithm.
The target is to guarantee, for a prescribed \(\delta\in(0,1)\),
\[
\Pr\!\big[\widehat{P}_T \in \arg\min_{P\in\mathcal{P}} \mu(P)\big] \;\ge\; 1-\delta,
\]
while exploiting the synthetic data (despite bias) and the similarity information in \(W\) to guide both estimation and sampling.

\subsection{Algorithm}

We first extend the ``static'' notation of Section~\ref{sec:edge_cost_est} to the adaptive setting.
At round $t$, we form a diagonal matrix of \emph{fidelity weights} \(M_t=\mathrm{diag}(w_e(t))\) as in Section~\ref{sec:edge_cost_est}, with \(w_e(t)=0\) if \(n_e(t)=0\).
Let \(\bar c'_e\) denote the (accurate) synthetic mean for edge \(e\) obtained from the simulator, and let \(\bar c(t)\) be the vector of real sample means with entries \(\bar c_e(t)\) when \(n_e(t)>0\) and an arbitrary value (e.g., zero) when \(n_e(t)=0\).
We maintain the real-minus-synthetic discrepancy
\[
y_t \coloneqq \bar c(t)-\bar c',
\]
and estimate the bias vector via the regularized smoother
\[
\hat b_t \coloneqq (M_t+\lambda L)^{-1} M_t\, y_t,
\]
where \(L\) is the Laplacian from similarity \(W\) and \(\lambda>0\) controls the amount of smoothing across edges.
The bias-calibrated edge means are
\[
\hat\mu_e(t)=\bar c'_e+\hat b_{t,e},
\qquad 
\hat\mu(P,t)=\sum_{e\in E(P)}\hat\mu_e(t)\quad\text{for any path }P.
\]

To quantify uncertainty, we reuse the pathwise confidence construction from Section~\ref{sec:edge_cost_est}, now in a time-uniform form.
For confidence level \(\delta\in(0,1)\) and each round $t$, let
\[
S_\lambda(M_t)\coloneqq\bigl(I+\lambda\,M_t^{-1/2} L M_t^{-1/2}\bigr)^{-1},\qquad
\alpha_e(\lambda;M_t)\coloneqq\bigl\|S_\lambda(M_t)M_t^{-1/2}e_e\bigr\|_2,
\]
where \(e_e\) denotes the \(e\)-th standard basis vector in \(\mathbb{R}^{|E|}\).
With problem constants \(B,\kappa_+\) from Section~\ref{subsec:edge_estimate_analysis}, we define the per-edge confidence radius
\begin{equation}
\beta_e(t)=\sqrt{\frac{\lambda}{2}}\,\frac{B}{\sqrt{w_e(t)}}\;+\;
\sqrt{\kappa_+}\,\alpha_e(\lambda;M_t)\,\sqrt{\,2\log\!\Bigl(\frac{2|E|\,\pi^2 t^2}{3\delta}\Bigr)} ,
\label{eq:edge-ci-active}
\end{equation}
with the convention that \(\beta_e(t)=+\infty\) whenever \(w_e(t)=0\).
We extend this to paths by
\[
\beta(P,t)\coloneqq \sum_{e\in E(P)}\beta_e(t).
\]
The resulting pathwise lower and upper confidence bounds are
\[
\mathrm{LCB}(P,t)\coloneqq \hat\mu(P,t)-\beta(P,t),\qquad
\mathrm{UCB}(P,t)\coloneqq \hat\mu(P,t)+\beta(P,t).
\]

The \emph{Active Estimated Shortest Path} algorithm (\textsc{A-ESP}) uses these confidence intervals to maintain a high-probability certificate for the optimal path and to decide which edge to query next.
At each round we compare the current empirical best path
\[
\widehat P_t\in\arg\min_{P\in\mathcal P} \hat\mu(P,t)
\]
with the most optimistic challenger
\[
\widetilde P_t\in\arg\min_{P\in\mathcal P\setminus\{\widehat P_t\}} \mathrm{LCB}(P,t).
\]
If the upper confidence bound of \(\widehat P_t\) already lies below the lower confidence bound of \(\widetilde P_t\),
then the ordering between all paths is certified and we stop; otherwise we take another real measurement.

The key algorithmic design choice is how to select the next edge to query when the certificate is not yet tight.
For our theoretical analysis we consider a simple global \emph{variance-balancing} rule.
We keep track of sub-Gaussian variance proxies \(\{\sigma_e^2\}_{e\in E}\) for the real noise on each edge and, at each round, query the edge with the largest ratio
\[
\frac{\sigma_e^2}{n_e(t)},
\]
with the convention \(\sigma_e^2/0:=+\infty\) (so edges that have never been queried are sampled first). Intuitively, this sampling rule focuses on the most uncertain edge. The complete algorithm is summarized below.

\begin{algorithm}[H]
\caption{\textsc{A-ESP}: Active Estimated Shortest Path}
\label{alg:a-esp}
\begin{algorithmic}[1]
\Require Graph \(G=(V,E)\), source \(v_\text{src}\), sink \(v_\text{sink}\); simulator for synthetic costs; Laplacian \(L\); regularization \(\lambda>0\); confidence \(\delta\in(0,1)\); variance proxies \(\{\sigma_e^2\}_{e\in E}\) for the real-data noise.
\State \textbf{Initialize:} For each edge \(e\in E\), obtain a very accurate synthetic mean \(\bar c'_e\) (by running the simulator many times).
Compute a simulator-best path \(P^{\mathrm{sim}}\in\arg\min_{P\in\mathcal P}\sum_{e\in E(P)}\bar c'_e\), and query one real sample on each edge \(e\in E(P^{\mathrm{sim}})\).
Initialize the real sample counts \(n_e(1)\in\{0,1\}\) and empirical means \(\bar c_e(1)\) accordingly, and set \(t\gets 1\).
\While{true}
  \Statex {\itshape \# Bias estimation (Laplacian calibration).}
  \State Form the diagonal weight matrix \(M_t=\mathrm{diag}(w_e(t))\) and the vector \(y_t=\bar c(t)-\bar c'\).
  \State Compute \(\hat b_t=(M_t+\lambda L)^{-1}M_t y_t\) and set \(\hat\mu_e(t)=\bar c'_e+\hat b_{t,e}\).

  \Statex {\itshape \# Edgewise confidence radii.}
  \State Compute \(\beta_e(t)\) for all edges via~\eqref{eq:edge-ci-active}.

  \Statex {\itshape \# Candidate paths (using the definitions of \(\hat\mu(\cdot,t)\), \(\mathrm{LCB}(\cdot,t)\), and \(\mathrm{UCB}(\cdot,t)\) above).}
  \State \(\widehat P_t\in\arg\min_{P\in\mathcal P} \hat\mu(P,t)\) \hfill (empirical best)
  \State \(\widetilde P_t\in\arg\min_{P\in\mathcal P\setminus\{\widehat P_t\}} \mathrm{LCB}(P,t)\) \hfill (most optimistic challenger)

  \Statex {\itshape \# Stopping rule.}
  \If{\(\mathrm{UCB}(\widehat P_t,t)\le \mathrm{LCB}(\widetilde P_t,t)\)}
    \State \textbf{return} \(\widehat P_t\).
  \EndIf

  \Statex {\itshape \# Query selection (global greedy variance balancing).}
  \State Choose \(e_t \in \arg\max_{e\in E} \sigma_e^2/n_e(t)\) \emph{(with the convention \(\sigma_e^2/0=+\infty\); break ties arbitrarily)}.
  \State Query one real sample on \(e_t\); update \(\bar c_{e_t}(t{+}1)\), \(n_{e_t}(t{+}1)=n_{e_t}(t)+1\), and \(w_{e_t}(t{+}1)\) as required by Assumption~\ref{ass:noiseM_new}.
  Keep \(n_e(t{+}1)=n_e(t)\) and \(w_e(t{+}1)=w_e(t)\) for all \(e\neq e_t\).
  \State Set \(t\gets t+1\).
\EndWhile
\end{algorithmic}
\end{algorithm}

For intuition, it is useful to note that the unknown true-cost difference between \(\widehat P_t\) and \(\widetilde P_t\) is completely determined by the edges in their symmetric difference
\(D_t = E(\widehat P_t)\triangle E(\widetilde P_t)\), since contributions from edges shared by both paths cancel.
The certificate gap
\[
\mathrm{UCB}(\widehat P_t,t)-\mathrm{LCB}(\widetilde P_t,t)
\;=\; \bigl[\hat\mu(\widehat P_t,t)-\hat\mu(\widetilde P_t,t)\bigr] \;+\; \beta(\widehat P_t,t)+\beta(\widetilde P_t,t)
\]
is therefore controlled by the uncertainty on edges that participate in such disagreements.
The global greedy rule balances variance across all edges and, in particular, across those that repeatedly appear in disagreement sets; the guarantees below show that this suffices to drive the certificate gap to zero at a near-optimal rate.

\paragraph{On unknown noise proxies.}
The sampling rule uses sub-Gaussian proxies \(\{\sigma_e^2\}_{e\in E}\) for the real-data noise.
In our analysis we assume that \(\sigma_e^2\) are known upper bounds on the (conditional) variances.
In practice these proxies can be obtained from a short preliminary phase, from domain knowledge, or by maintaining conservative empirical estimates.
Over-estimating the variances only makes the confidence radii larger and the allocation more conservative, so the correctness guarantees below remain valid (up to constants), while the resulting sample-complexity bound continues to scale with the aggregate variance \(\Sigma^2=\sum_{e\in E}\sigma_e^2\).

\subsection{Guarantees and Sample Complexity}
\label{subsec:active-guarantees}

% \dmytro{I suggest writing ``In this section, we analyze ...'' because the section doesn't do anything, it's us who perform the analysis. We should probably rewrite the last paragraph of the introduction in a similar spirit.}

% This section analyzes Algorithm~\ref{alg:a-esp} in the adaptive regime where the real counts \(n_e(t)\) evolve over time while the simulator sample sizes \(n'_e\) are fixed. All proofs are deferred to Appendix \ref{appx:proofs}. We work with the following adaptive analogue of Assumption~\ref{ass:noiseM}. 

In this section, we analyze Algorithm~\ref{alg:a-esp} in the adaptive regime where the real counts \(n_e(t)\) evolve over time while the simulator sample sizes \(n'_e\) are fixed. All proofs are deferred to Appendix \ref{appx:proofs}. We work with the following adaptive analogue of Assumption~\ref{ass:noiseM}. 

\begin{assumption}[Adaptive noise and weight calibration]
\label{ass:noiseM_new}
At round $t$, the maintained edge statistic is the real--minus--synthetic mean
\(y_{t,e}\) with conditional sub-Gaussian proxy
\[
\nu_e(t)^2\ \le\ \frac{\sigma_e^2}{n_e(t)}\;+\;\frac{\sigma_e'^2}{n'_e},
\qquad n_e(t)=\sum_{s=1}^t \bm 1\{e_s=e\}.
\]
The \emph{predictable} fidelity weights \(w_e(t)\) satisfy, for fixed constants \(1\le\kappa_-\le\kappa_+<\infty\),
\[
\kappa_-\,\frac{1}{\nu_e(t)^2}\ \le\ w_e(t)\ \le\ \kappa_+\,\frac{1}{\nu_e(t)^2},
\]
with \(w_e(t)=0\) if \(n_e(t)=0\).
\end{assumption}

\noindent
Assumption~\ref{ass:noiseM_new} requires that (i) the real edge costs and simulator discrepancy are uniformly sub-Gaussian with proxies \(\sigma_e^2\) and \(\sigma_e'^2\); and (ii) the weights \(w_e(t)\) track the inverse conditional variance up to universal factors.
In words, edges that are measured more accurately (either because they have more real samples or because the simulator is more reliable on them) carry larger weight in the smoother.

Our first result shows that the edgewise confidence intervals from \eqref{eq:edge-ci-active} yield valid pathwise certificates and a correct stopping rule, regardless of how the edges are selected.

\begin{theorem}[Anytime pathwise confidence and correctness]
\label{thm:aesp-correctness}
Under Assumptions~\ref{ass:smoothM}--\ref{ass:noiseM_new} and $n'_e$ is large enough that we ignore simulator mean error, the event
\[
\mathcal E_\delta\ :=\ \Bigl\{\forall\, t\ge1,\ \forall\, P\in\mathcal P:\ \mu(P)\in\bigl[\mathrm{LCB}(P,t),\mathrm{UCB}(P,t)\bigr]\Bigr\}
\]
occurs with probability at least \(1-\delta\).
Consequently, on \(\mathcal E_\delta\), if Algorithm~\ref{alg:a-esp} 
stops at time $t$ because \(\mathrm{UCB}(\widehat P_t,t)\le \mathrm{LCB}(\widetilde P_t,t)\), then \(\widehat P_t\) equals the true optimum \(P^\star\).
\end{theorem}

Theorem~\ref{thm:aesp-correctness} certifies that \textsc{A-ESP} is correct whenever it stops and that its pathwise certificates are valid uniformly over all rounds and all paths.

We now provide a sample-complexity guarantee for the global greedy variant in Algorithm~\ref{alg:a-esp}. For the sample-complexity result we further work in a regime where the residual simulator variance is negligible compared to the real-data noise, so that the Monte Carlo noise from the simulator is dominated by the variability in real costs.

\begin{assumption}[Negligible simulator variance]
\label{ass:negligible-sim}
Under Assumption~\ref{ass:noiseM_new}, the simulator sample sizes \(n'_e\) are large enough that, for all edges \(e\) and all rounds $t$ with \(n_e(t)\ge1\),
\[
\nu_e(t)^2\ \le\ \frac{\sigma_e^2}{n_e(t)},
\qquad\text{equivalently}\qquad
\frac{\sigma_e'^2}{n'_e}\ \ll\ \frac{\sigma_e^2}{n_e(t)}.
\]
In particular, we may drop the floor term \(\sigma_e'^2/n'_e\) in the bound of Assumption~\ref{ass:noiseM_new} and regard the weights as satisfying \(w_e(t)\asymp n_e(t)/\sigma_e^2\).
This corresponds to running the simulator many times per edge before the active phase so that its Monte Carlo noise is dominated by real-data variability.
\end{assumption}

Let \(P^\star\) be an optimal path and \(P^{(2)}\) a second-best path (i.e., the path with second smallest real expected cost) and define the optimality gap
\[
\Delta\;\coloneqq\;\min_{P\neq P^\star}\bigl[\mu(P)-\mu(P^\star)\bigr]\;=\;\mu(P^{(2)})-\mu(P^\star)\;>0.
\]
Write \(L_{\max} = \max_{P\in\mathcal P} |E(P)|\) for the maximum number of edges on any $v_\text{src}$--$v_\text{sink}$ path, \(|E|\) for the total number of edges, and \(\Sigma^2=\sum_{e\in E}\sigma_e^2\) for the aggregate real-data variance.

Under the negligible-simulator regime of Assumption~\ref{ass:negligible-sim}, we obtain the following bound.

\begin{theorem}[Sample complexity for A-ESP with global greedy sampling]
\label{thm:aesp-simple-sc}
Under Assumptions~\ref{ass:smoothM}, \ref{ass:noiseM_new}, and~\ref{ass:negligible-sim}, there exists a universal constant \(C>0\) such that, with probability at least \(1-\delta\), Algorithm~\ref{alg:a-esp} with the global greedy sampling rule returns the optimal path \(P^\star\) by or before
\begin{equation}
T_\delta\ \le\
C\cdot
\frac{|E|\,L_{\max}}{\kappa_-\,\Delta^2}\,
\Bigl(\lambda B^2 + \kappa_+ \log\frac{2|E|}{\delta}\Bigr)\,
\Sigma^2.
\end{equation}
\end{theorem}

At a high level, the proof combines (i) the pathwise confidence guarantee from Theorem~\ref{thm:aesp-correctness}, (ii) the decomposition of the edgewise confidence radii \(\beta_e(t)\) into a smoothing bias term of order \(\sqrt{\lambda}\,B/\sqrt{w_e(t)}\) and a variance term controlled by \(\alpha_e(\lambda;M_t)\), and (iii) a greedy variance-balancing argument showing that the global rule \(e_t\in\arg\max_e \sigma_e^2/n_e(t)\) drives the average variance \(\sum_{e\in E}\sigma_e^2/n_e(t)\) down at the rate \(O(|E|\Sigma^2/t)\).
A detailed proof is given in Appendix~\ref{appx:active-proofs}.

Similarly to Theorem~\ref{thm:edgewise}, Theorem~\ref{thm:aesp-simple-sc} reveals several interpretable trade-offs:
\begin{itemize}
\item \textbf{Optimality gap \(\Delta\).}
Larger gaps reduce the required number of real queries: distinguishing a clearly superior path from its runner-up is easier than separating many nearly tied alternatives.
\item \textbf{Smoothness and regularization \((B,\lambda)\).}
The parameter \(B\) captures how smoothly the simulator-to-real bias varies across the edge-similarity graph, while \(\lambda\) controls how aggressively we smooth.
Stronger smoothing (larger \(\lambda\)) reduces the variance term but increases the bias term, and the combination \(\lambda B^2\) in the bound reflects this bias-variance trade-off.
\item \textbf{Real-data variance \(\Sigma^2\).}
Lower real-data noise leads to fewer required queries.
The factor \(\Sigma^2\) shows that \textsc{A-ESP} automatically adapts to heterogeneous edge variances: harder edges (with larger \(\sigma_e^2\)) attract more samples, but the overall difficulty still scales with the aggregate variance of the network.
\item \textbf{Graph and path complexity \((|E|,L_{\max})\).}
The dependence on both the total number of edges \(|E|\) and the maximum path length \(L_{\max}\) reflects the fact that each certificate concerns an entire path, yet sampling is performed at the edge level.
Longer paths and denser graphs require more samples to control the uncertainty on all relevant edges.
\item \textbf{Weight calibration \((\kappa_-,\kappa_+)\).}
Tighter calibration (i.e., smaller spread between \(\kappa_-\) and \(\kappa_+\)) leads to sharper bounds.
The constant \(\kappa_-\) controls how faithfully the weights \(w_e(t)\) approximate the inverse noise level, while \(\kappa_+\) scales the variance term inside \(\beta_e(t)\).
\end{itemize}

\begin{remark}[Non-unique optimal paths]
\label{rem:nonunique-opt}
The definition of the ``optimality gap''
\[
\Delta\coloneqq \min_{P\neq P^\star}\bigl[\mu(P)-\mu(P^\star)\bigr]
\]
implicitly assumes that the optimal path is unique. Indeed, if there exist multiple optimal paths, then there is some \(P\neq P^\star\) with \(\mu(P)=\mu(P^\star)\), and hence \(\Delta=0\).
In this case, any guarantee (or stopping rule) that requires \emph{strict separation} between \(P^\star\) and all other paths cannot hold with finite samples: no confidence intervals can certify that one optimal path is strictly better than another when their means coincide.

\smallskip
\noindent
\textbf{Implication for Algorithm~\ref{alg:a-esp}.}
The stopping rule in Algorithm~\ref{alg:a-esp} compares \(\widehat P_t\) to a single challenger \(\widetilde P_t\) and certifies a strict ordering. When the optimal path is not unique, the challenger \(\widetilde P_t\) may itself be optimal, so the condition
\(\mathrm{UCB}(\widehat P_t,t)\le \mathrm{LCB}(\widetilde P_t,t)\)
need not ever be satisfied (even though returning \emph{any} element of \(\mathcal P^\star\) is acceptable).

A standard remedy is to adopt an \((\varepsilon,\delta)\)-PAC objective: return a path whose expected cost is within \(\varepsilon\) of optimal.
This can be achieved by relaxing the stopping rule to
\[
\mathrm{UCB}(\widehat P_t,t)\ \le\ \mathrm{LCB}(\widetilde P_t,t)\ +\ \varepsilon,
\]
in which case, in the event \(\mathcal E_\delta\) from Theorem~\ref{thm:aesp-correctness}, the returned path satisfies
\(\mu(\widehat P_t)\le \mu^*+\varepsilon\)
without requiring the algorithm to distinguish between multiple optimal solutions, where $\mu^*$ is the optimal shortest-path expected cost.
\end{remark}

\section{Experiments}
\label{subsec:experiments}

We now evaluate our proposed edge cost estimation algorithm and active shortest-path learning algorithm through numerical experiments. Our goal is to compare our methods with natural baselines and to understand how performance varies with the noise level in real observations, the number of real observations (for Algorithm~\ref{alg:esp_real}), and the smoothness and magnitude of simulator bias.

\subsection{Experimental Setup}

\paragraph{Dataset Description.}
Our experiments are based on road networks derived from real traffic data.
\begin{itemize}
    \item 
\textbf{Urban road networks (\texttt{netzschleuder}) \citep{tiago_p_peixoto_2020_7839981}}.  We first select three representative urban road networks from \texttt{netzschleuder}: \texttt{Barcelona}, \texttt{Brasilia}, and \texttt{Irvine2}. As illustrated in Figure~\ref{fig:simul-topo}, these graphs exhibit markedly different topological patterns, which allow us to test the robustness of our methods across network structures. These datasets do not have edge costs and we generate a set of them (for both the ``real'' data and the synthetic data) as detailed below.
\item \textbf{Traffic graphs (\texttt{METR-LA} and \texttt{PEMS-BAY}) \citep{j49q-ch56-25}}.
To evaluate performance on truly observed traffic, we additionally use the \texttt{METR-LA} and \texttt{PEMS-BAY} datasets, whose road graphs are shown in Figure~\ref{fig:real-topo}. Each edge corresponds to a traffic sensor that reports a time series of speeds as the edge cost in the dataset. We defer the detailed description of this dataset to Appendix \ref{appx:experiment}.
\end{itemize}

\begin{figure}[!htbp]
    \centering
    \begin{subfigure}{0.32\textwidth}
        \centering
        \includegraphics[width=\linewidth]{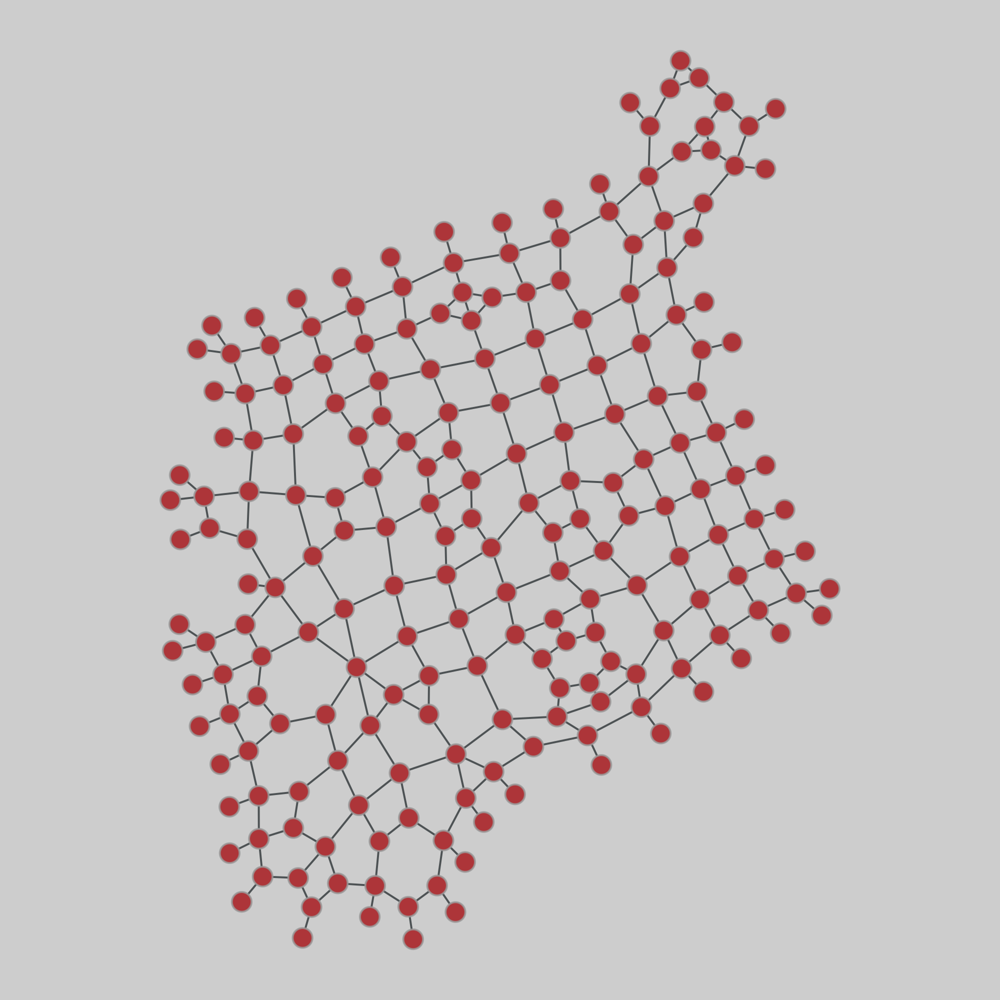}
        \caption{Barcelona}
    \end{subfigure}
    \begin{subfigure}{0.32\textwidth}
        \centering
        \includegraphics[width=\linewidth]{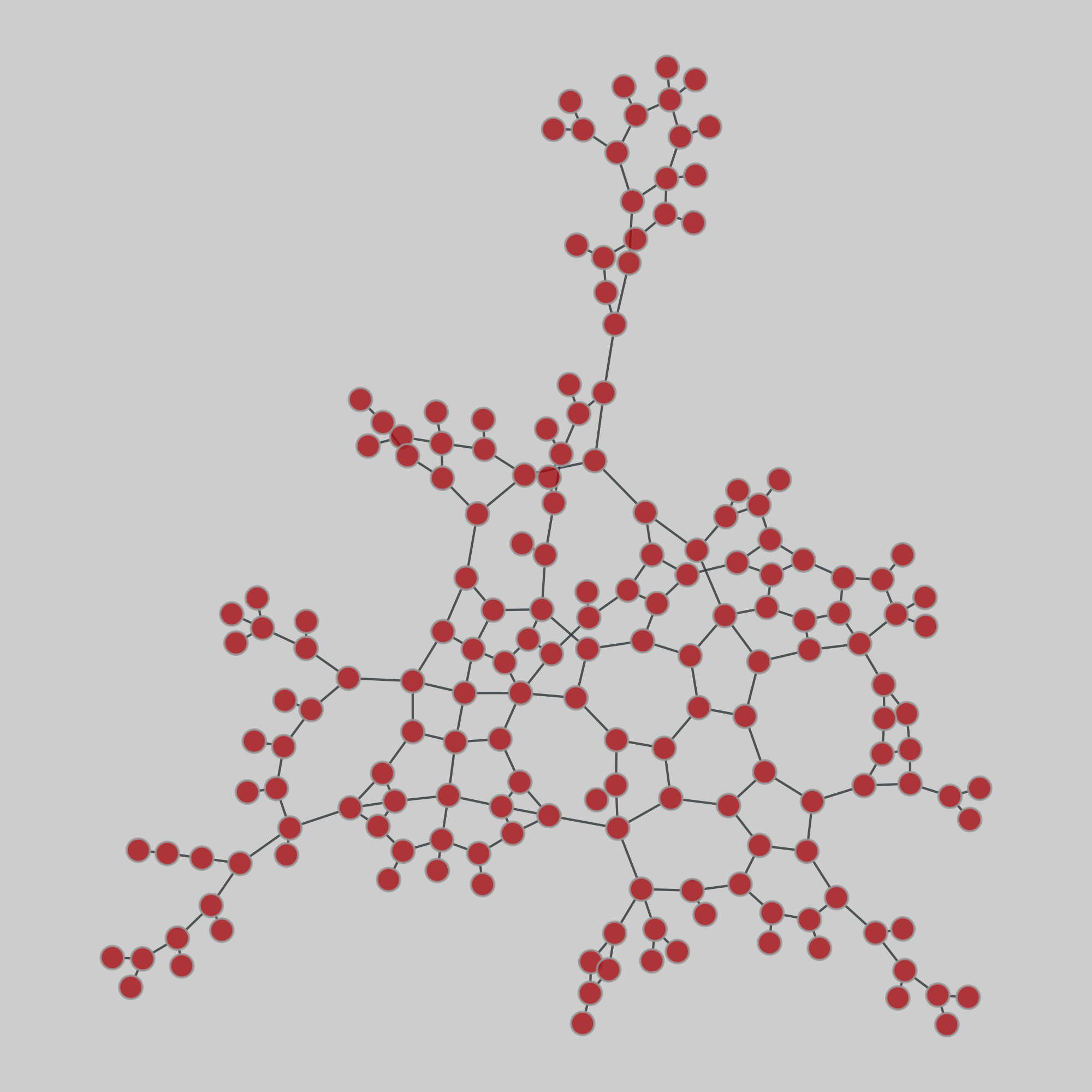}
        \caption{Brasilia}
    \end{subfigure}
    \begin{subfigure}{0.32\textwidth}
        \centering
        \includegraphics[width=\linewidth]{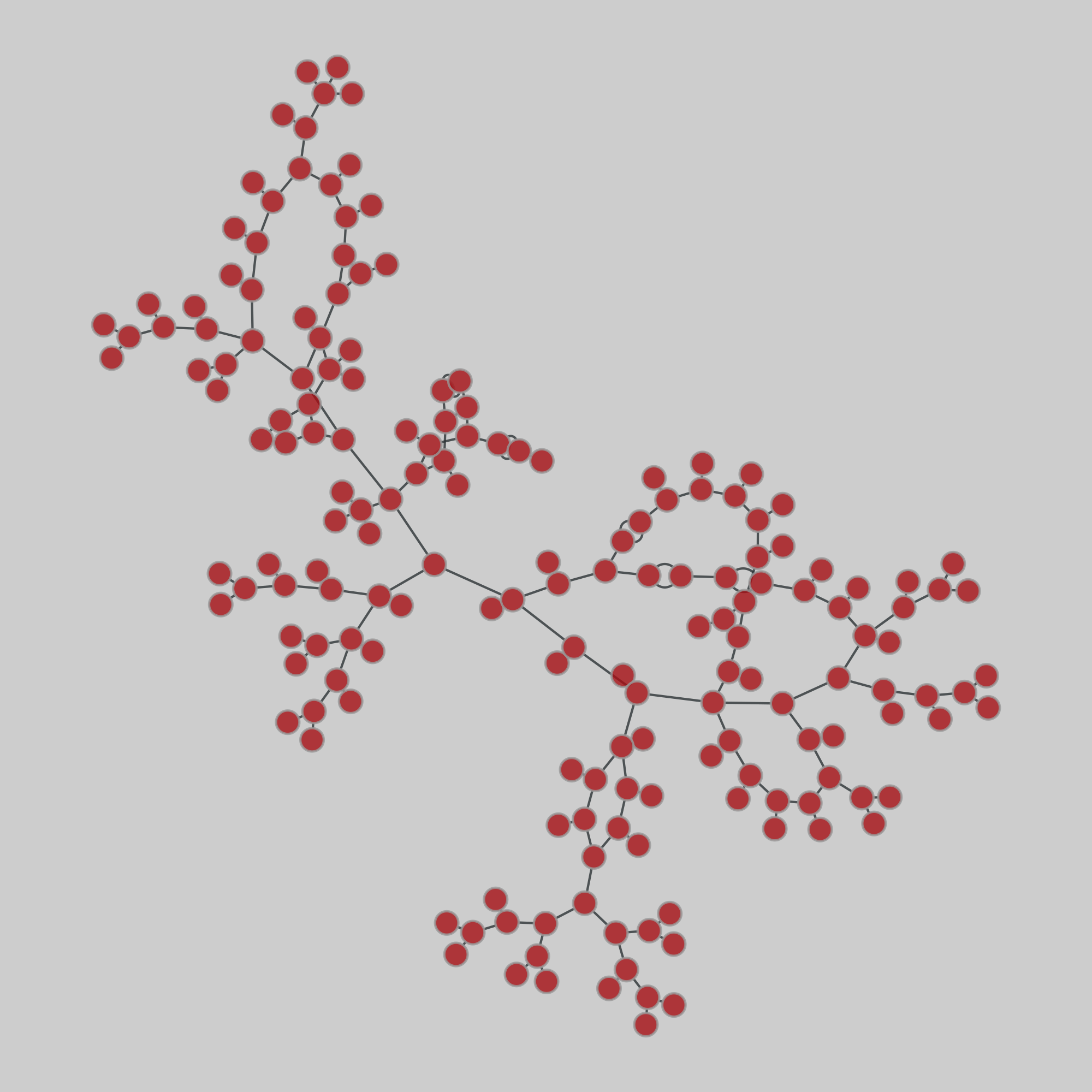}
        \caption{Irvine2}
    \end{subfigure}
    \caption{Topology of the three urban road networks from \texttt{netzschleuder}. Barcelona contains many short loops, Irvine2 is approximately tree-like, and Brasilia combines both patterns.}
    \label{fig:simul-topo}
\end{figure}

\begin{figure}[!htbp]
    \centering
    \begin{subfigure}{0.35\textwidth}
        \centering
        \includegraphics[width=\linewidth]{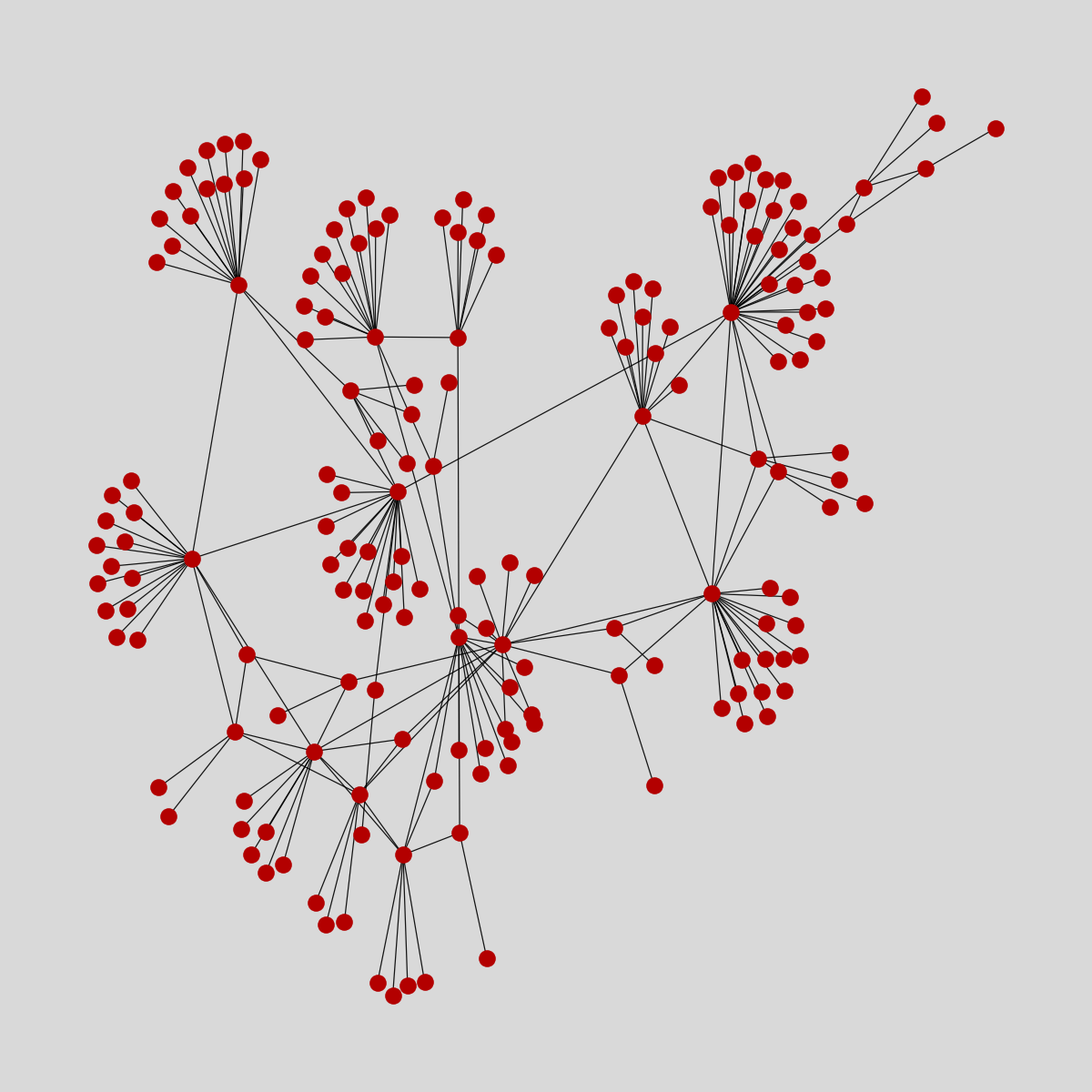}
        \caption{METR-LA}
    \end{subfigure}
    \begin{subfigure}{0.35\textwidth}
        \centering
        \includegraphics[width=\linewidth]{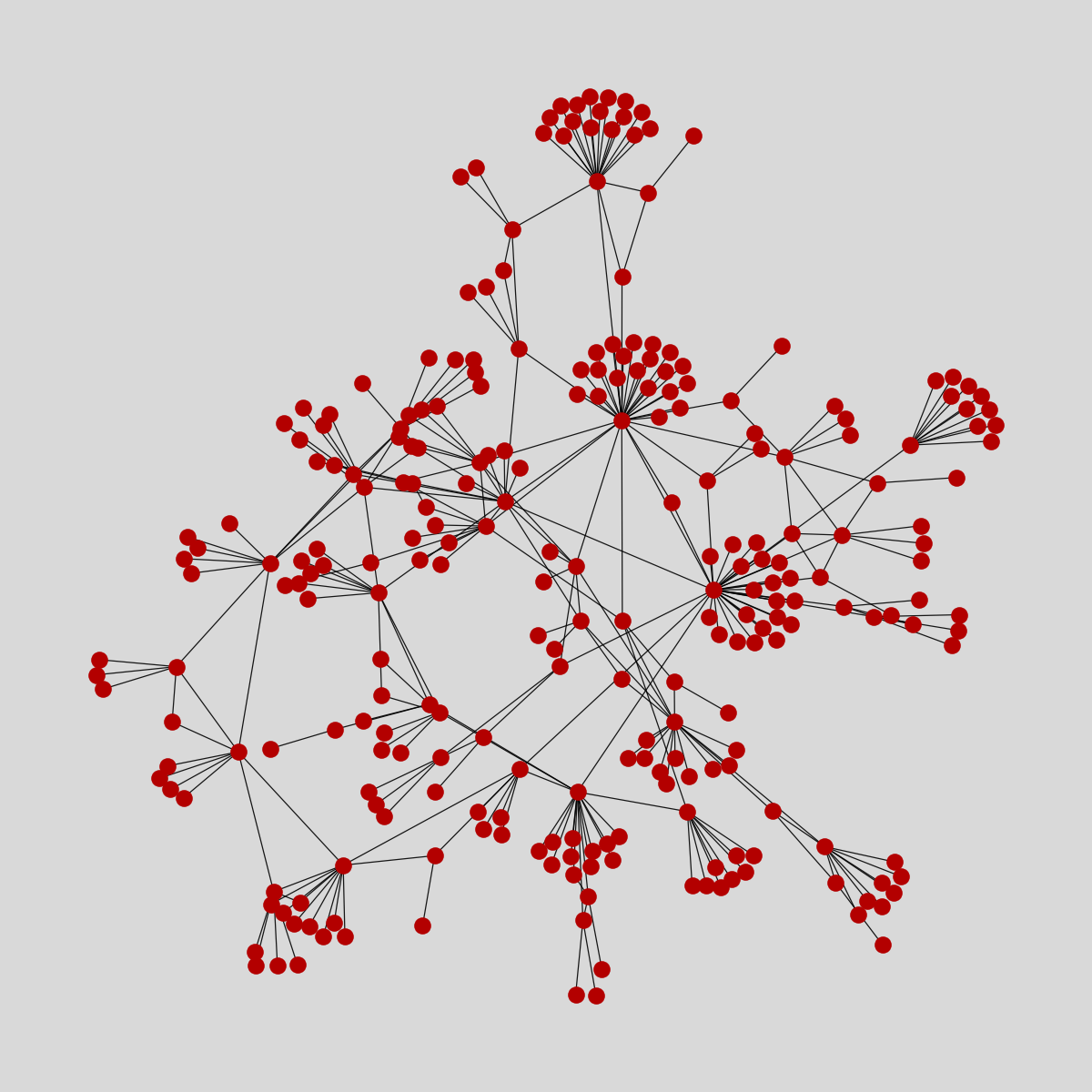}
        \caption{PEMS-BAY}
    \end{subfigure}
    \caption{Topology of the two real traffic graphs used in our case study. Each edge represents a sensor location that records traffic flow information.}
    \label{fig:real-topo}
\end{figure}

\paragraph{Data Generation.}
On the three \texttt{netzschleuder} graphs, each edge’s real cost $\mu_e$ is modeled as an independent Gaussian random variable $\mathcal{N}(50,10^2)$. We then construct a smooth simulator bias vector $b^\star$ in two steps. First, we solve
\[
    (I + \rho L)\,\tilde{b}^\star = \xi, \qquad \xi \sim \mathcal{N}(0,I_{|E|}),
\]
where $L$ is the Laplacian of the edge-similarity graph and $\rho>0$ controls the degree of smoothing. Second, we rescale to meet a prescribed Laplacian seminorm,
\[
    b^\star \;=\; \frac{B}{\sqrt{(\tilde{b}^\star)^\top L \tilde{b}^\star}}\,\tilde{b}^\star,
\]
so that $\|b^\star\|_L = B$. The parameters $\rho$ and $B$ therefore govern the smoothness and magnitude of the bias, respectively. Intuitively, $(I + \rho L)^{-1}$ acts like a low-pass filter on the white-noise vector $\xi$, so that for larger $\rho$ nearby edges in the similarity graph tend to share similar bias values, while $B$ simply scales this smooth bias pattern up or down, making the simulator weakly or strongly misspecified overall. In our experiments we test different choices of $\rho$ and $B$ to see their influence on the algorithm's performance. The choice of the similarity matrix $W$ in $L$ is the Heat-kernel (diffusion) similarity, which is defined in Appendix \ref{appx:experiment}.

Given $\mu_e$ and $b^\star_e$, we define the simulator mean as $\mu'_e = \mu_e - b^\star_e$. Since simulator evaluations are assumed inexpensive, we treat the simulator means as known by drawing enough samples to accurately estimate $\mu'_e$. For real observations, unless otherwise specified, each edge has $n_e = 20$ i.i.d.\ samples $c_{e,i} \sim \mathcal{N}(\mu_e,30^2)$, and $25\%$ of the edges are unobservable (i.e., $n_e = 0$ and no real samples are available) in the edge cost estimation experiments.

For the \texttt{METR-LA} and \texttt{PEMS-BAY} datasets, we construct paired synthetic and real data directly from the raw traffic time series. We select two three-hour windows per day with markedly different congestion patterns (6--9\,am and 3--6\,pm). For each edge, we take the empirical mean and variance of speeds in the afternoon window as the ``real'' cost distribution, and those from the morning window as the ``simulator'' distribution. This construction preserves the heterogeneity and spatial correlation present in real traffic while inducing a realistic, unconstrained bias pattern between simulator and real costs. Throughout the case study of edge cost estimation, we cap the number of real samples at 20 per edge and label $50\%$ of the edges as unobservable. More details on the real data generation can be found in Appendix \ref{appx:experiment}.

\paragraph{Baselines.}
In the edge cost estimation experiments, we compare our estimator with three natural baselines:
\begin{itemize}
    \item \textbf{SIM (Synthetic Only).} Uses the simulator means only, $\hat\mu^{\mathrm{SIM}}_e \equiv \bar c'_e$, and ignores all real observations.
    \item \textbf{REAL (Real Only with Synthetic Fallback).} Uses the empirical mean of real observations when available and falls back to the simulator otherwise:
    \[
        \hat\mu^{\mathrm{REAL}}_e =
        \begin{cases}
        \bar c_e, & n_e > 0,\\[2pt]
        \bar c'_e, & n_e = 0.
        \end{cases}
    \]
    \item \textbf{CONST (Global Shift Calibration).} Fits a single global shift $\gamma$ by weighted least squares,
    \[
        \gamma
        = \arg\min_{\tilde{\gamma}\in \mathbb{R}} \sum_e w_e\,(\bar c_e - \bar c'_e - \tilde{\gamma})^2
        = \frac{\sum_e w_e(\bar c_e - \bar c'_e)}{\sum_e w_e},
    \]
    and sets $\hat\mu = \bar c' + \gamma\,\mathbf 1$.
\end{itemize}

In the active learning experiments, edge costs are always estimated by our proposed calibration method. We compare \textsc{A-ESP} (Algorithm~\ref{alg:a-esp}) with a simple \textbf{Random} baseline that selects the next edge to query uniformly at random over all edges, rather than targeting the currently noisiest edge. Both methods are initialized with the same prior information: each edge on the simulator-optimal path $P^{\mathrm{sim}}$ is queried once for a real sample.

\subsection{Results: Edge Cost Estimation}

We first study edge cost estimation on the \texttt{netzschleuder} graphs.  Figure~\ref{fig:edge-real-obs} reports the Root Mean Squared Error (RMSE) of the calibrated edge means $\hat\mu$ as  we vary the fraction of edges with real observations while keeping the per-edge sample size fixed at $n_e=20$, thereby mimicking different total numbers of real observations/samples.
Figure~\ref{fig:edge-smoothness} sweeps over the bias smoothness parameter $\rho$ and magnitude $B$, controlling both the scale and spatial variation of the simulator-to-real bias $b^\star$. Figure~\ref{fig:edge-real-noise}  varies the noise level of real observations by changing the standard deviation of the sampling distribution. Finally, Figure~\ref{fig:edge-pathwise} expands the estimation granularity to the path-level, which illustrates the suboptimality by comparing the true cost gap between the estimated shortest path and the ground-truth one.  Unless otherwise stated, we always conduct 5 independent runs. The reported numbers are the means over the conducted runs with shaded areas indicating the standard deviation.

\begin{figure}[!htbp]
    \centering
    \begin{subfigure}[t]{0.32\textwidth}
        \centering
        \includegraphics[width=\linewidth]{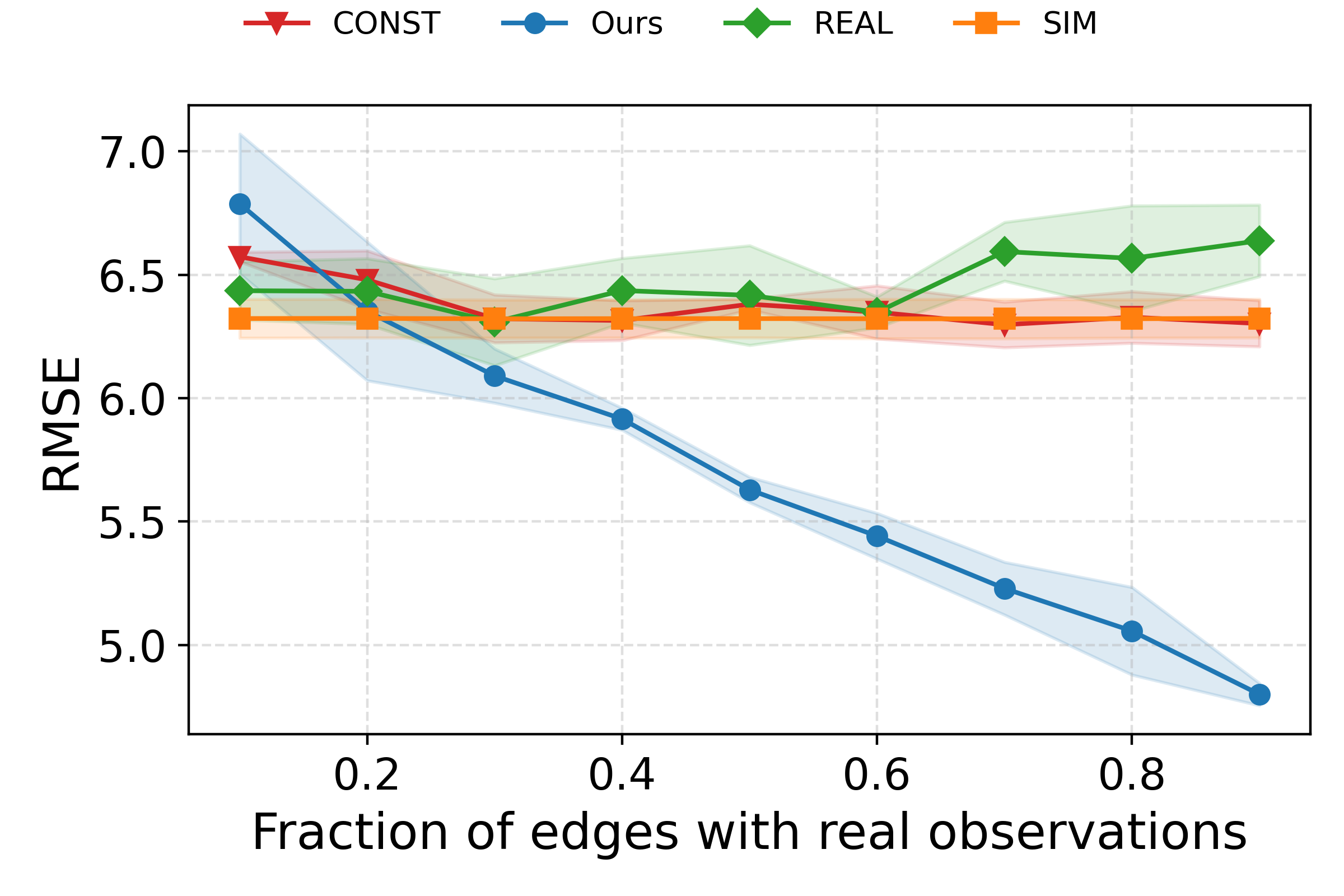}
        \caption{Barcelona}
    \end{subfigure}
    \hfill
    \begin{subfigure}[t]{0.32\textwidth}
        \centering
        \includegraphics[width=\linewidth]{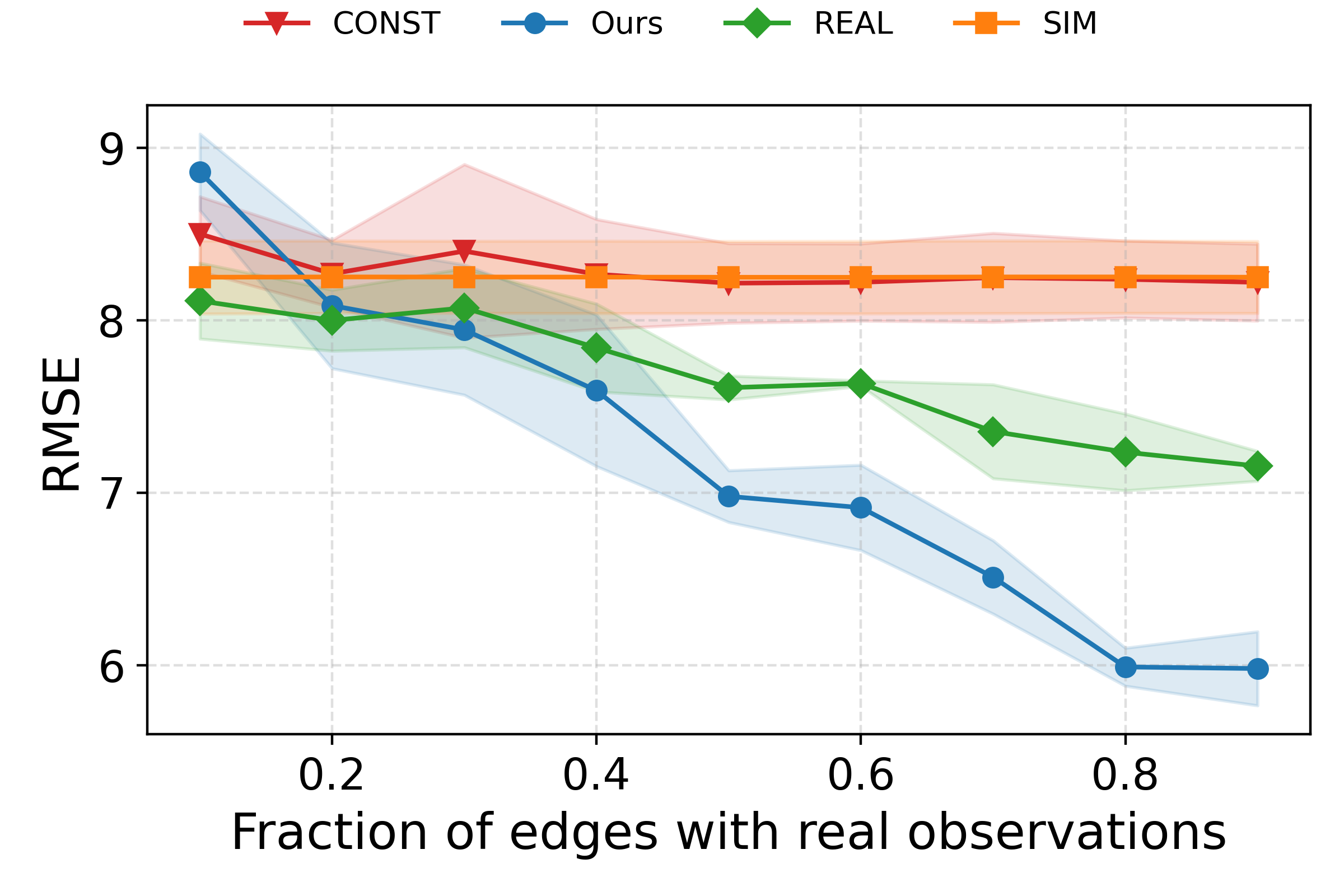}
        \caption{Brasilia}
    \end{subfigure}
    \hfill
    \begin{subfigure}[t]{0.32\textwidth}
        \centering
        \includegraphics[width=\linewidth]{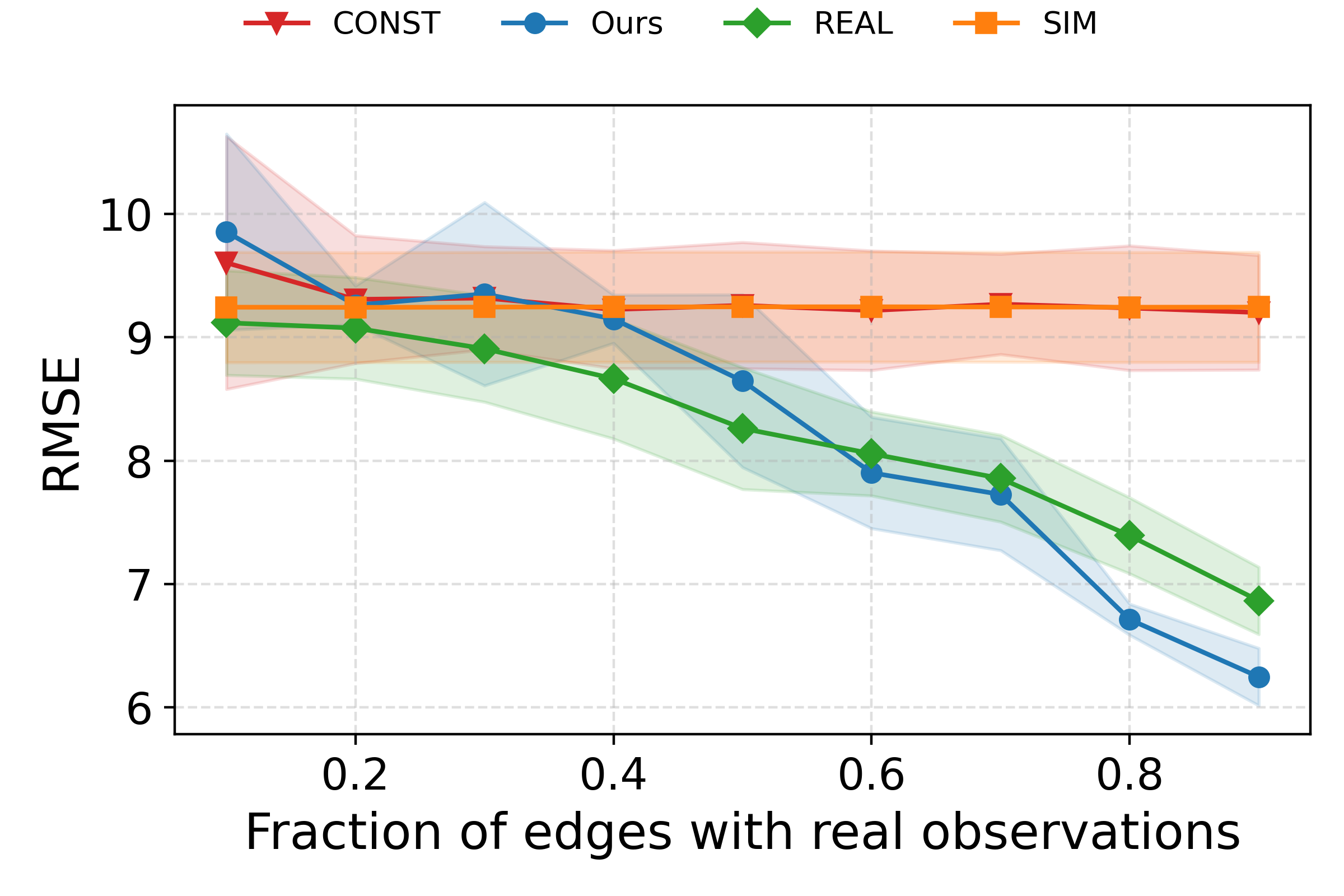}
        \caption{Irvine2}
    \end{subfigure}
    \caption{Effect of the number of real observations on the RMSE of $\hat\mu$ for the three \texttt{netzschleuder} graphs (left to right: Barcelona, Brasilia, Irvine2). The horizontal axis is the fraction of observable edges and the number of real samples per observable edge is fixed at $n_e = 20$, thus varying the observable fraction changes the total number of real samples. Results are averaged over multiple seeds (lower is better).}
    \label{fig:edge-real-obs}
\end{figure}

\begin{figure}[!htbp]
    \centering
    \begin{subfigure}{\textwidth}
        \centering
        \includegraphics[width=0.88\linewidth]{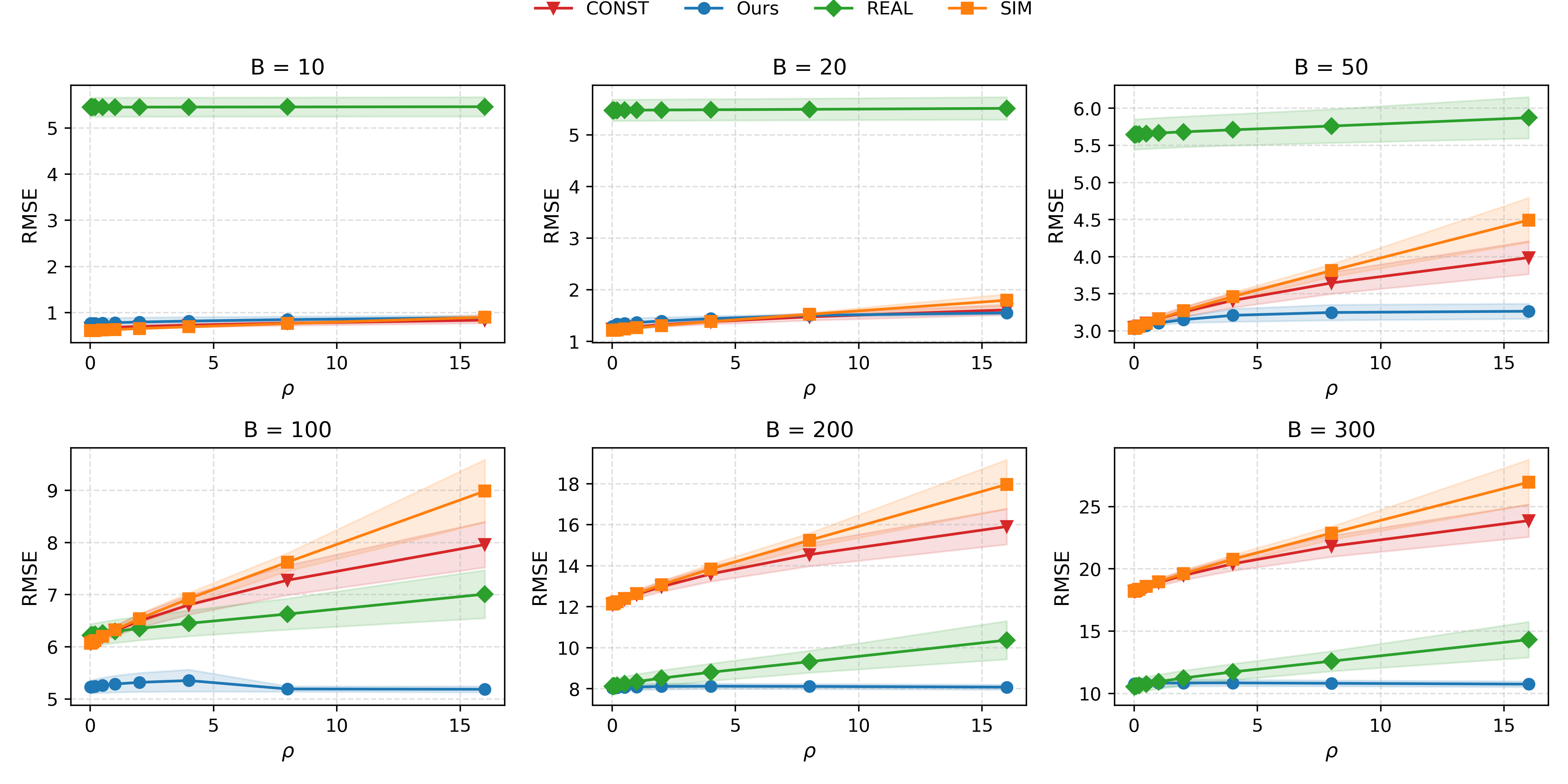}
        \caption{Barcelona}
    \end{subfigure}
    \medskip
    \begin{subfigure}{\textwidth}
        \centering
        \includegraphics[width=0.88\linewidth]{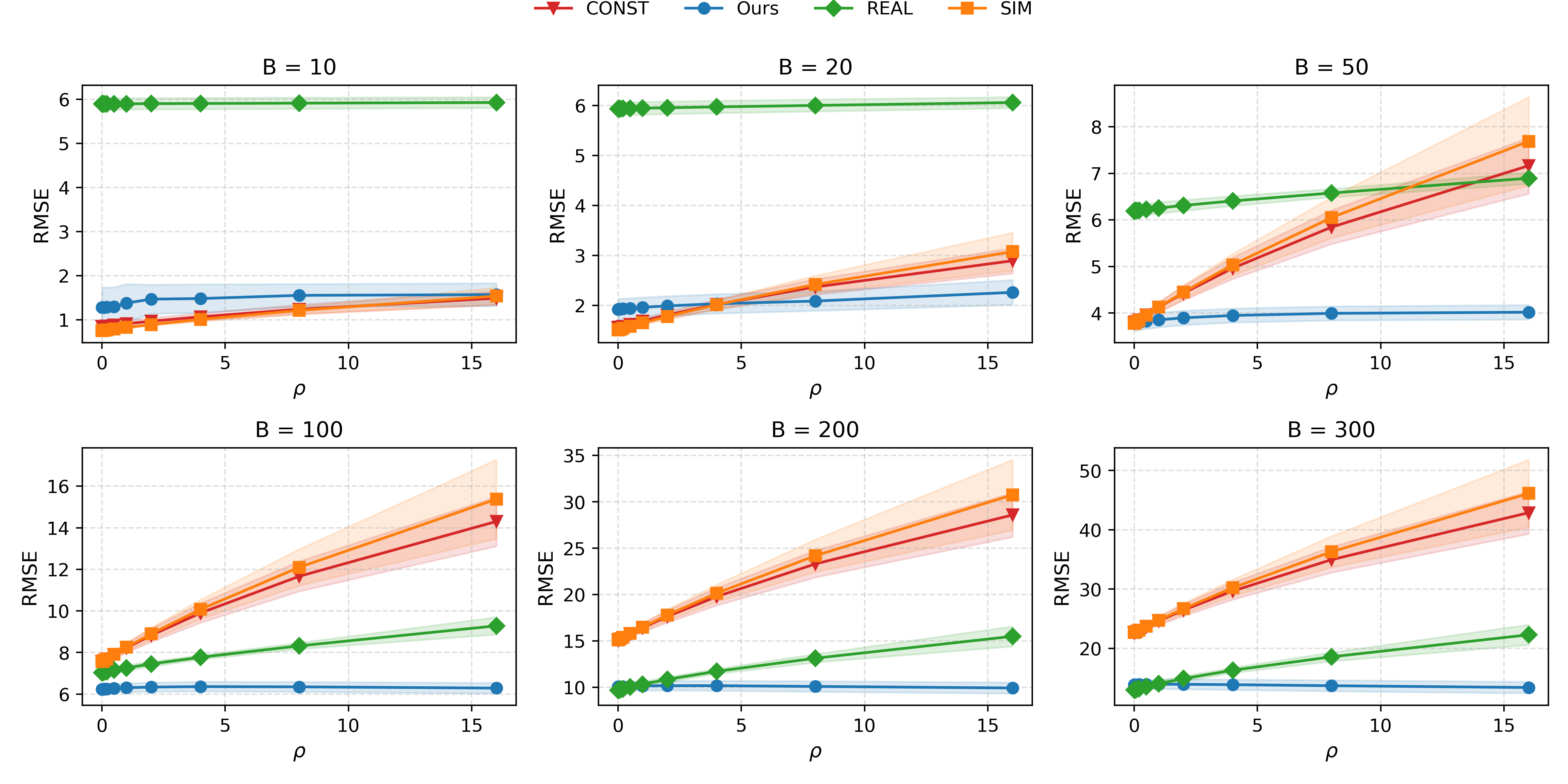}
        \caption{Brasilia}
    \end{subfigure}
    \medskip
    \begin{subfigure}{\textwidth}
        \centering
        \includegraphics[width=0.88\linewidth]{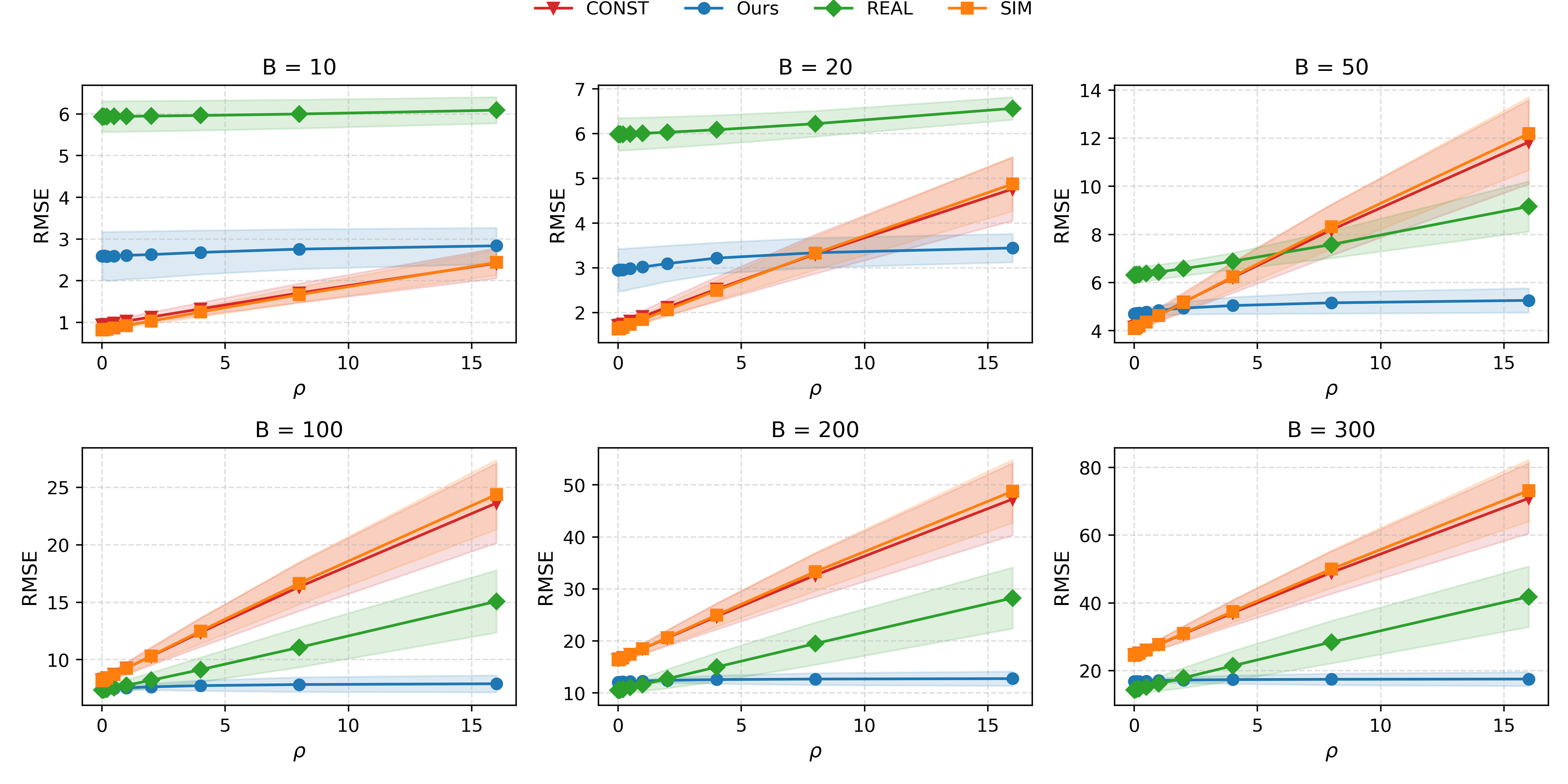}
        \caption{Irvine2}
    \end{subfigure}
    \caption{RMSE of $\hat\mu$ across simulator-bias smoothness and magnitude on the three \texttt{netzschleuder} graphs. Each panel sweeps the Laplacian seminorm target $\| b^\star\|_L$ (via $B$) and the smoothing parameter $\rho$ used to generate $b^\star$. Results are averaged over seeds (lower is better).}
    \label{fig:edge-smoothness}
\end{figure}

\begin{figure}[!htbp]
    \centering
    \begin{subfigure}[t]{0.32\textwidth}
        \centering
        \includegraphics[width=\linewidth]{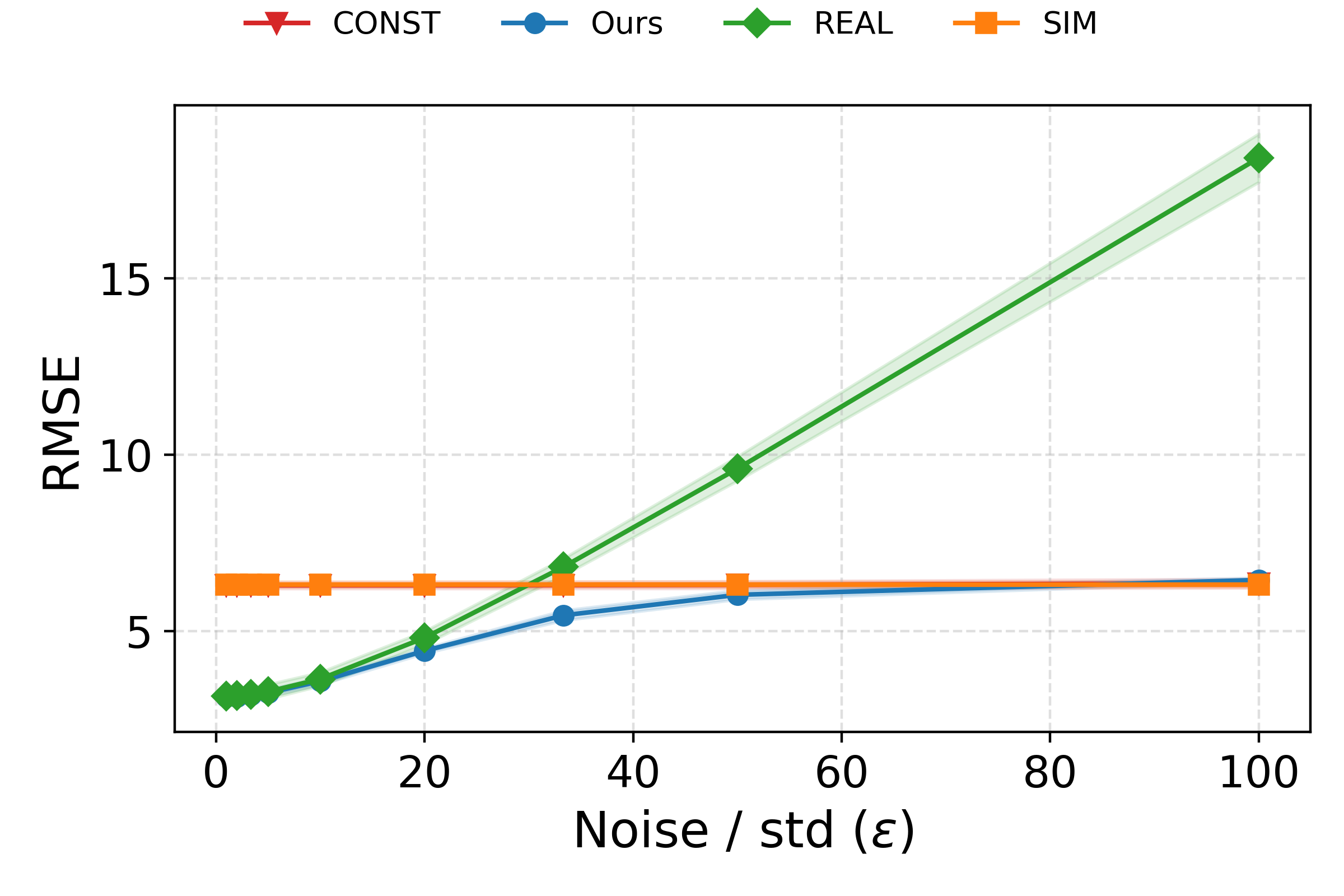}
        \caption{Barcelona}
    \end{subfigure}
    \hfill
    \begin{subfigure}[t]{0.32\textwidth}
        \centering
        \includegraphics[width=\linewidth]{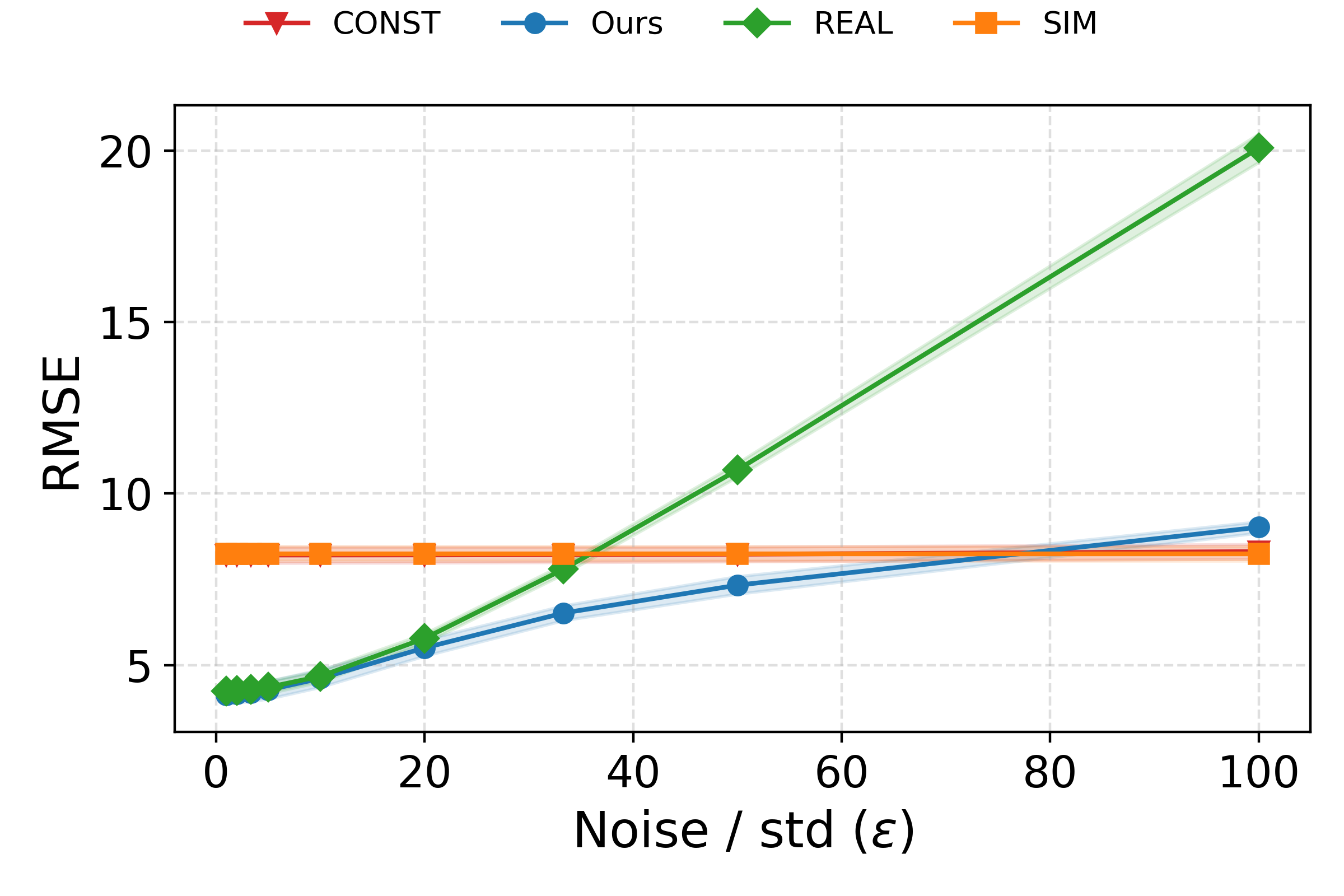}
        \caption{Brasilia}
    \end{subfigure}
    \hfill
    \begin{subfigure}[t]{0.32\textwidth}
        \centering
        \includegraphics[width=\linewidth]{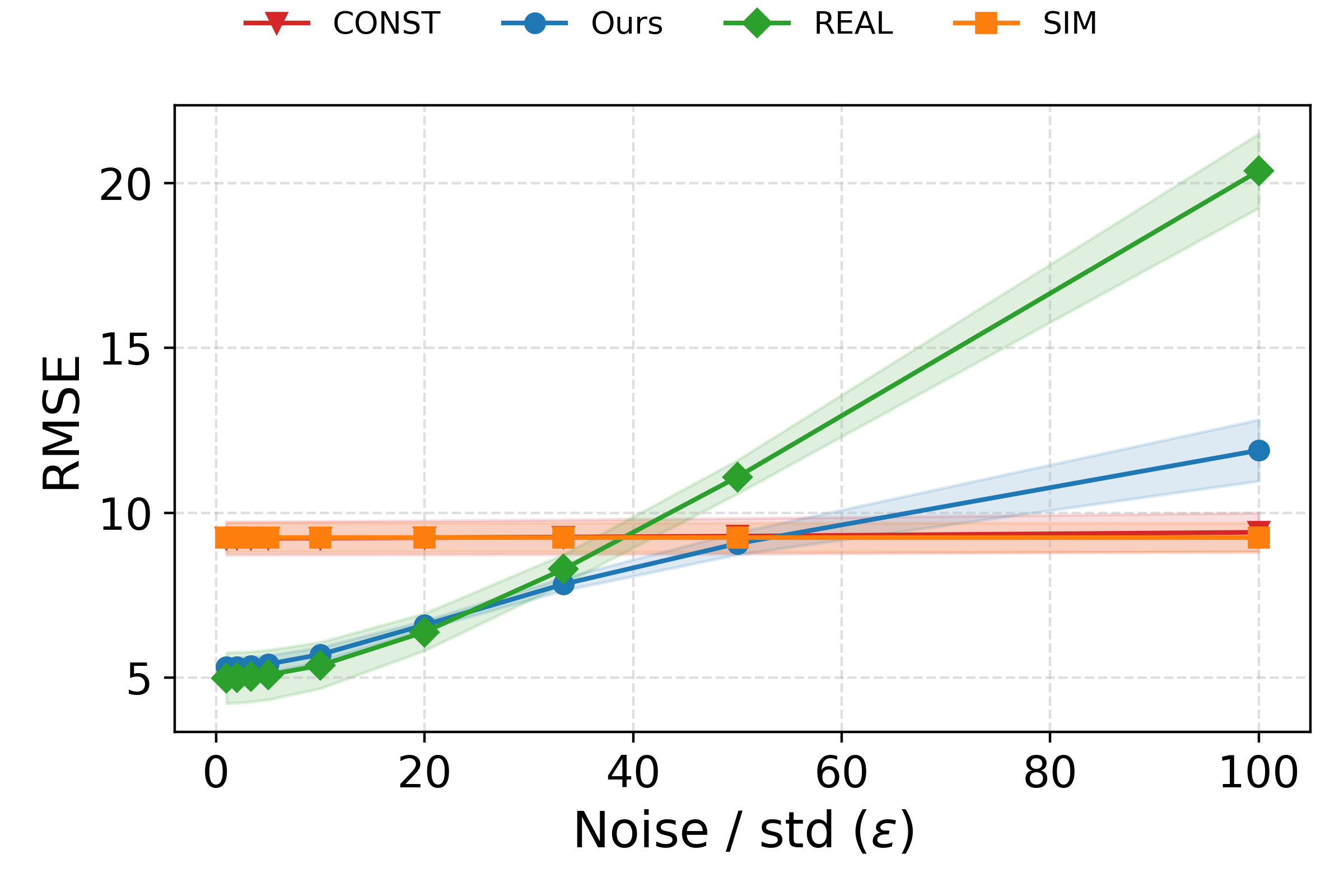}
        \caption{Irvine2}
    \end{subfigure}
    \caption{Effect of the real-observation noise level on the RMSE of the calibrated edge means $\hat\mu$ for the three \texttt{netzschleuder} graphs (left to right: Barcelona, Brasilia, Irvine2). Curves are averaged over multiple random seeds (lower is better). }
    \label{fig:edge-real-noise}
\end{figure}
\begin{figure}[!htbp]
    \centering
    \begin{subfigure}[t]{0.32\textwidth}
        \centering
        \includegraphics[width=\linewidth]{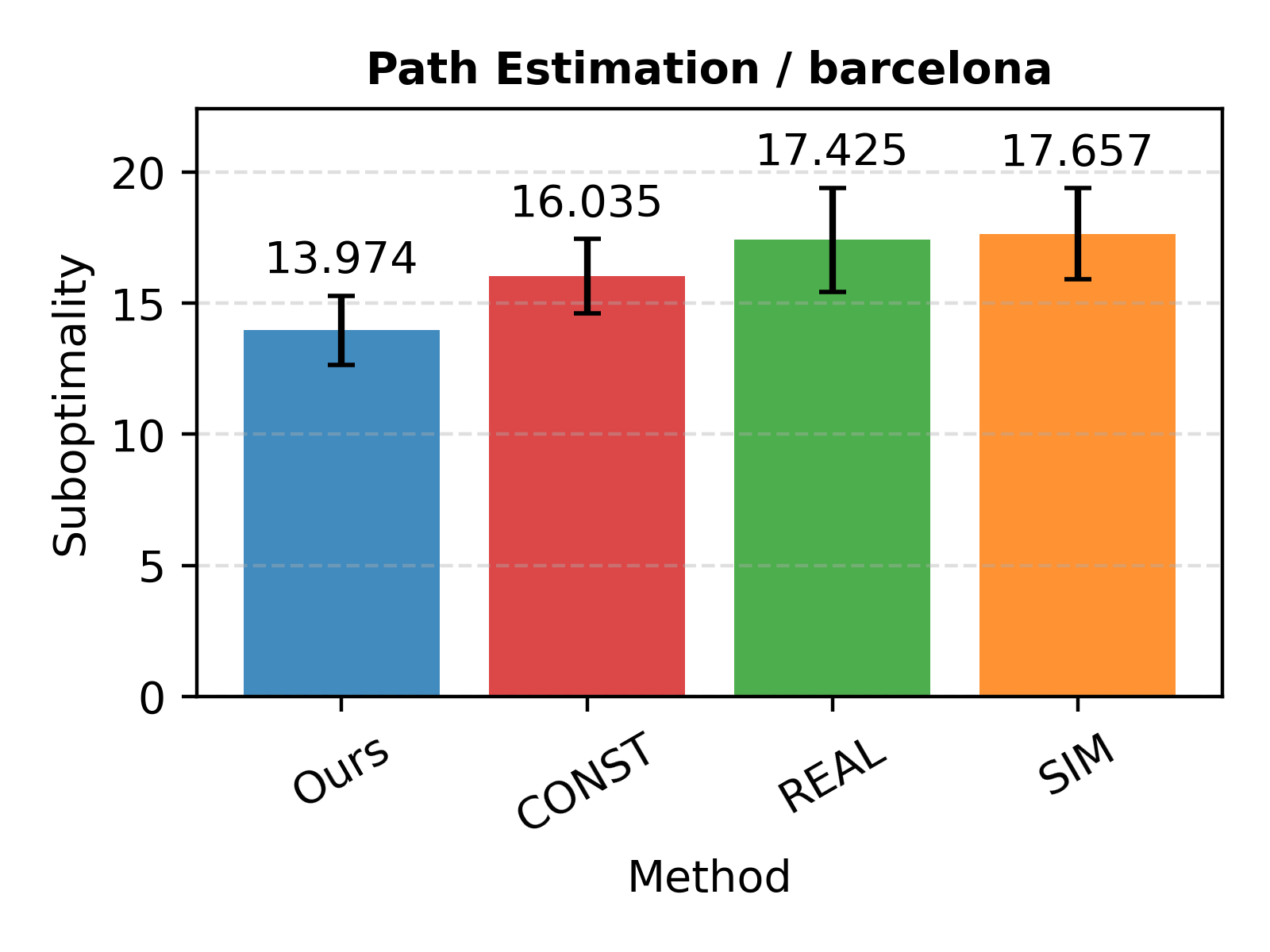}
        \caption{Barcelona}
    \end{subfigure}
    \hfill
    \begin{subfigure}[t]{0.32\textwidth}
        \centering
        \includegraphics[width=\linewidth]{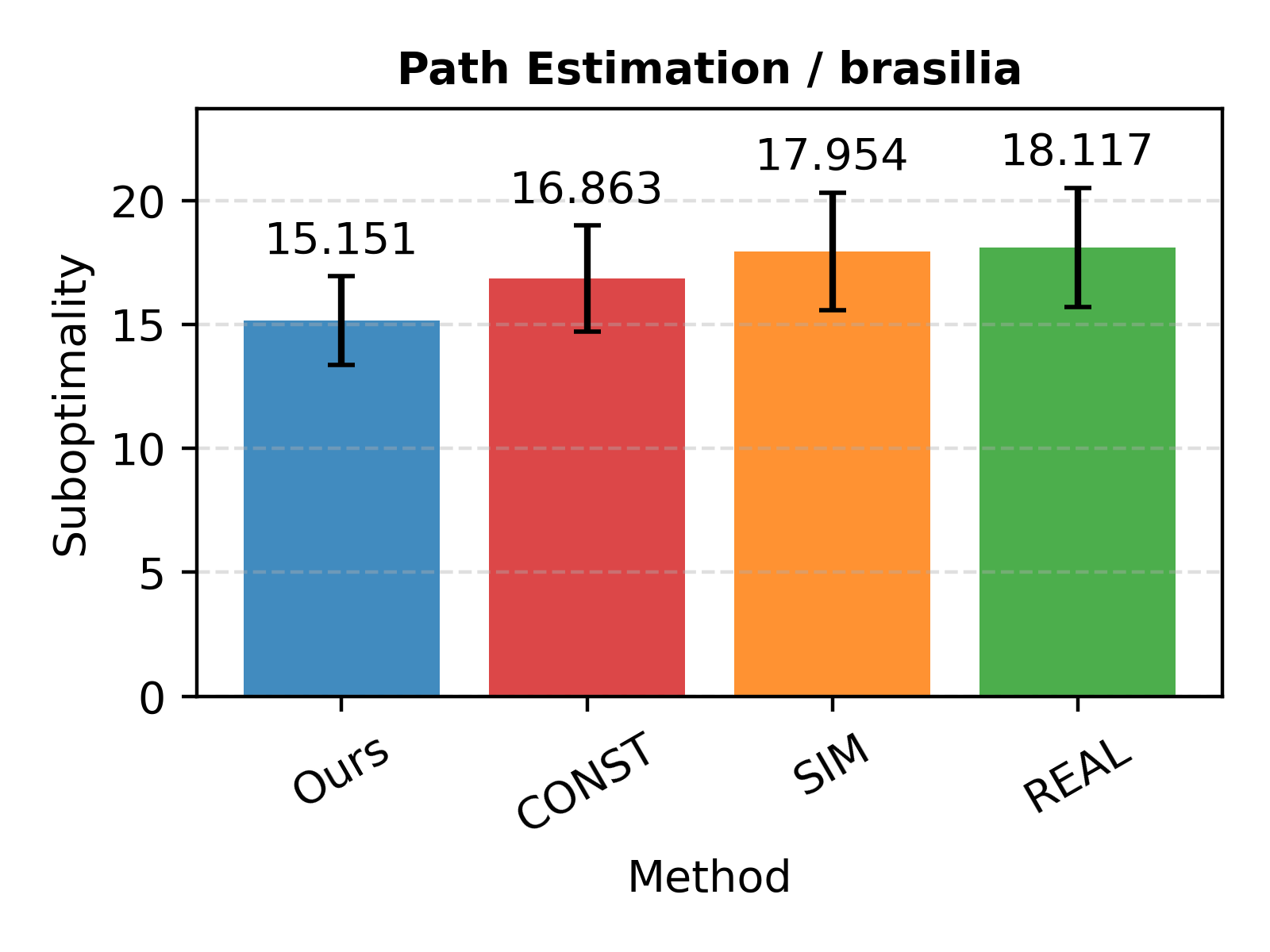}
        \caption{Brasilia}
    \end{subfigure}
    \hfill
    \begin{subfigure}[t]{0.32\textwidth}
        \centering
        \includegraphics[width=\linewidth]{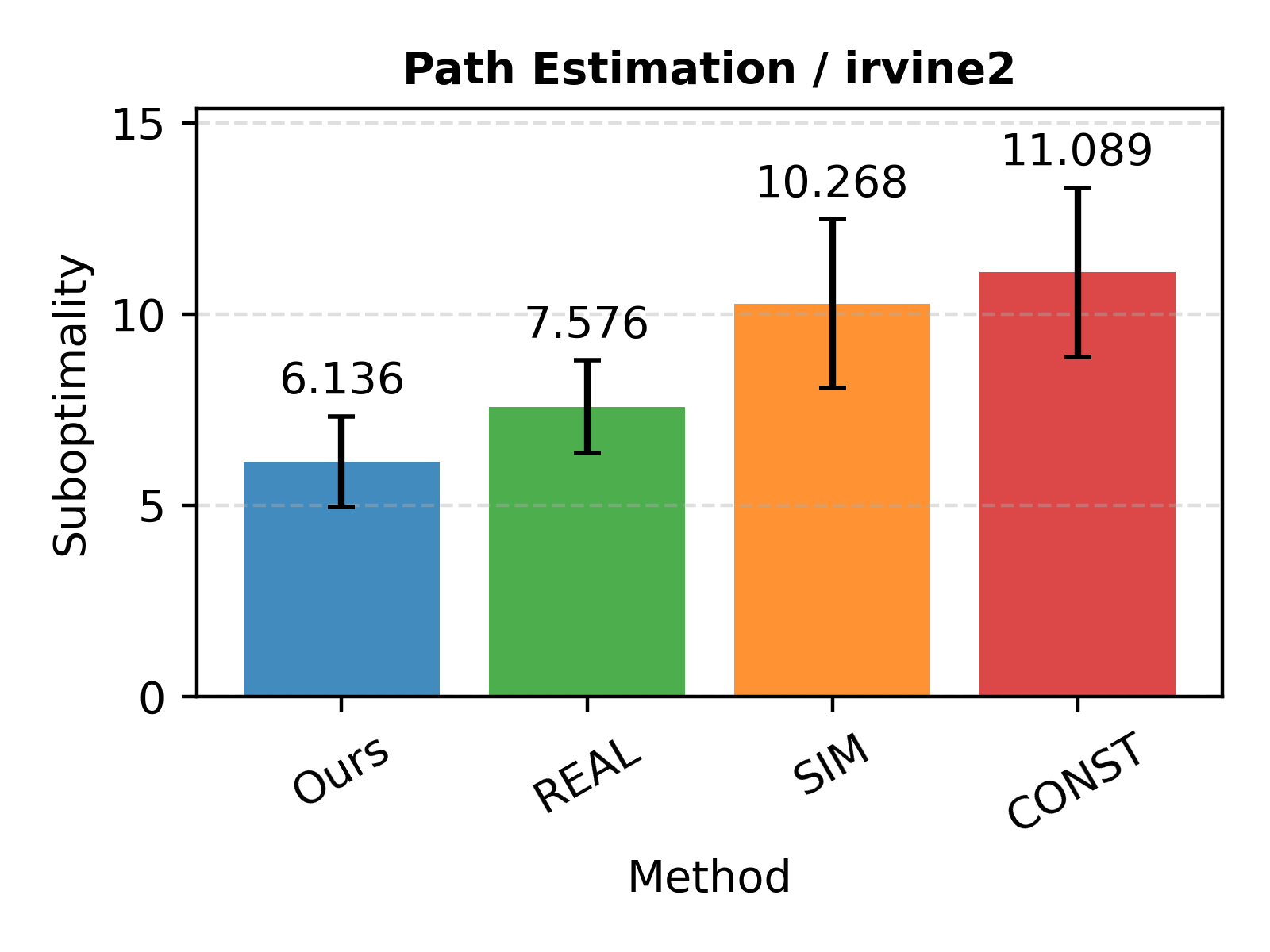}
        \caption{Irvine2}
    \end{subfigure}
    \caption{Path-level performance evaluation for the three \texttt{netzschleuder} graphs (left to right: Barcelona, Brasilia, Irvine2). Each random seed selects a new source-target pair to perform shortest path estimation. Curves are averaged over multiple random seeds (lower is better).}
    \label{fig:edge-pathwise}
\end{figure}

From Figures~\ref{fig:edge-real-obs}--\ref{fig:edge-pathwise}, we highlight several key observations:
\begin{itemize}
    \item \textbf{Value of structural regularization under limited coverage.} Figure~\ref{fig:edge-real-obs} shows that increasing the fraction of observable edges does not uniformly benefit REAL and CONST, since each edge still has few real samples. In contrast, our method propagates information across neighboring edges through the similarity graph and therefore continues to improve as either the coverage or the sample size increases. This desirable property also helps our method stand out in the path-level estimation, as illustrated by Figure~\ref{fig:edge-pathwise}.
     \item \textbf{Sensitivity of naive baselines to bias magnitude and smoothness.} Methods that rely exclusively on either synthetic data (SIM) or real data with a global shift (CONST), without any structural regularization, are highly sensitive to the magnitude and smoothness of the simulator bias as shown in Figure \ref{fig:edge-smoothness}. As the bias magnitude $B$ increases, the discrepancy between simulator and real costs grows, and synthetic-only predictions deteriorate markedly. When the bias is highly uneven, there exist edges with small bias where synthetic information remains reasonably accurate; however, as the bias becomes smoother (larger $\rho$), the simulator incurs a comparable shift on every edge, so synthetic information becomes systematically misleading. Our Laplacian-regularized estimator is much more robust to these changes.
      \item \textbf{Benefit of combining real and synthetic data under high noise.} When the real-data noise is large and the per-edge query budget is limited, purely real estimators (REAL) become unreliable, while purely simulator-based estimators (SIM) ignore valuable information from real observations. In these regimes, our method, which fuses real data, synthetic data, and structural information, achieves substantially lower RMSE. In particular, in the extremely noisy settings of Figure~\ref{fig:edge-real-noise}, the synthetic data act as a stable anchor, and our estimator leverages them to shrink estimation error rapidly as more edges become observable.
\end{itemize}

We next turn to the \texttt{METR-LA} and \texttt{PEMS-BAY} case study. Figure~\ref{fig:edge-baseline} reports baseline performance under the original data distribution described above. Figures~\ref{fig:edge-casestudy-ratio} and~\ref{fig:edge-casestudy-number} then perform a sensitivity analysis by varying, respectively, the fraction of observable edges and the number of real samples per observable edge. Figure~\ref{fig:edge-casestudy-path} also performs the path-level evaluation and computes the performance gap between the estimated and ground-truth shortest paths.

\begin{figure}[!htbp]
    \centering
    \begin{subfigure}{\textwidth}
        \centering
        \includegraphics[width=0.9\linewidth]{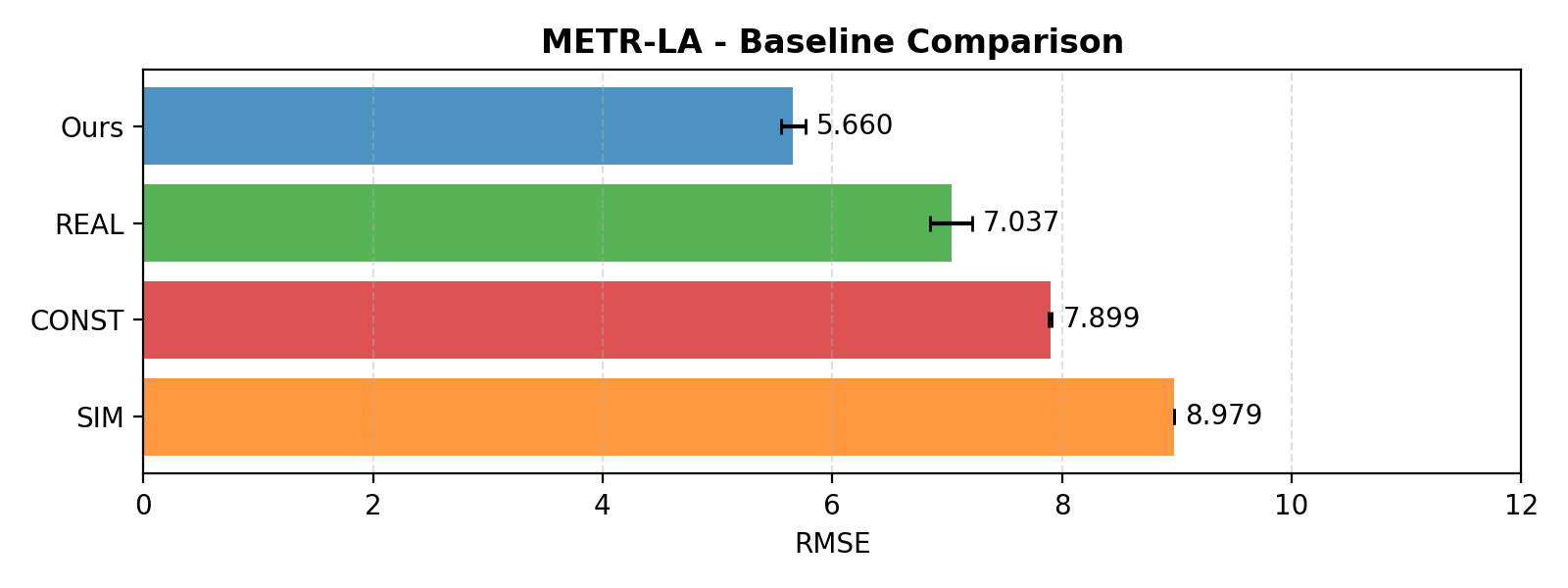}
        \caption{METR-LA}
    \end{subfigure}
    \medskip
    \begin{subfigure}{\textwidth}
        \centering
        \includegraphics[width=0.9\linewidth]{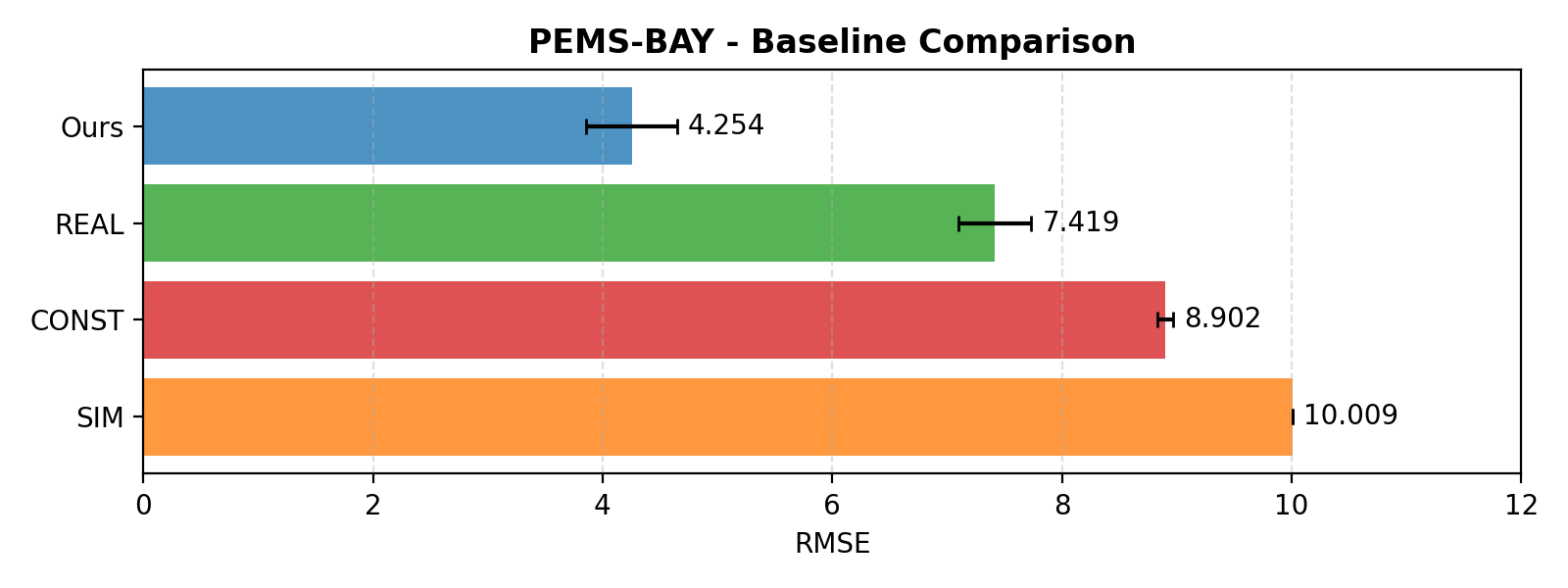}
        \caption{PEMS-BAY}
    \end{subfigure}
    \caption{Baseline comparison on METR-LA and PEMS-BAY under the original data distribution: RMSE of the calibrated edge means $\hat\mu$ for all methods, averaged over seeds (lower is better).}
    \label{fig:edge-baseline}
\end{figure}

\begin{figure}[!htbp]
    \centering
    \begin{subfigure}{0.49\textwidth}
        \centering
        \includegraphics[width=\linewidth]{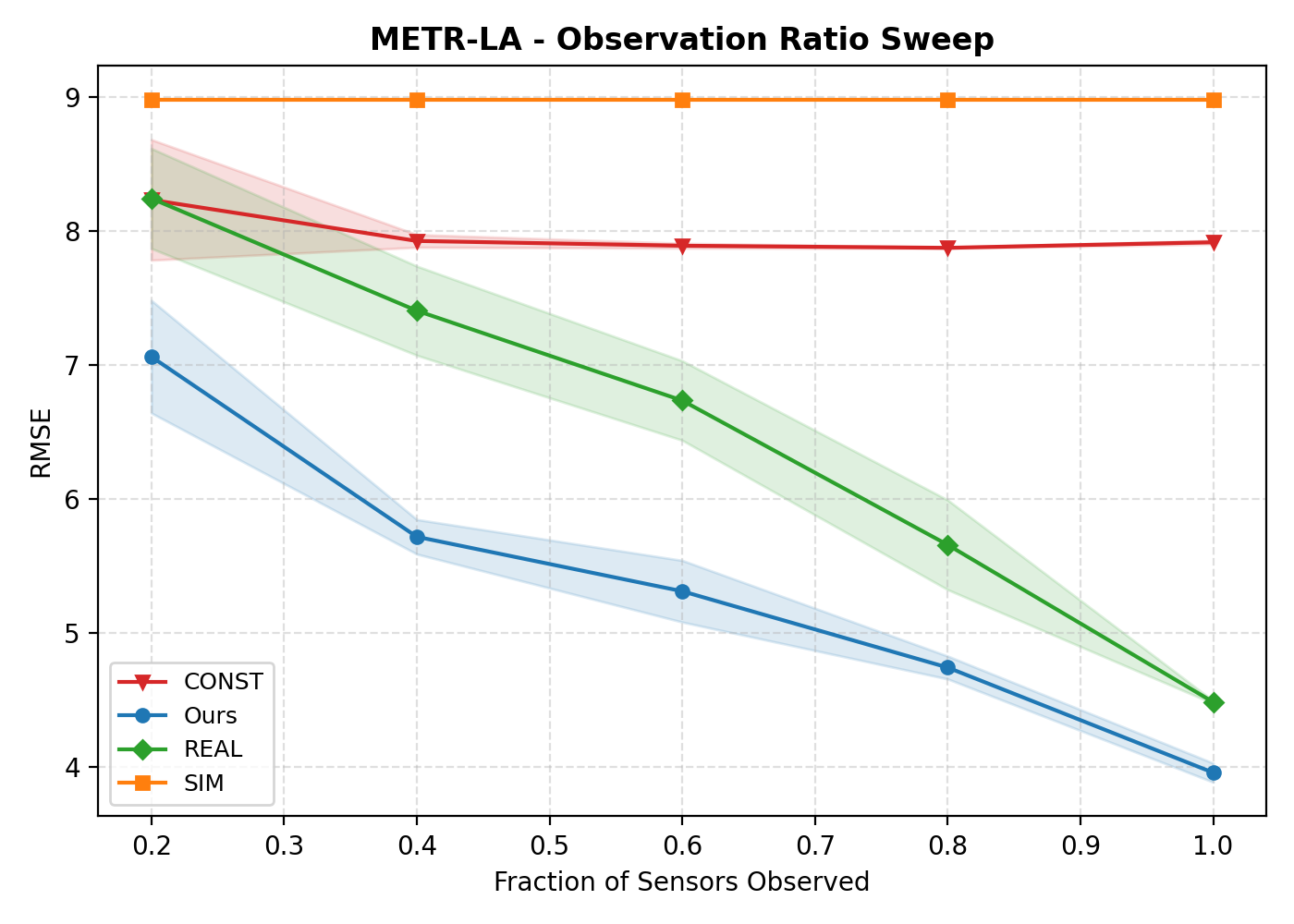}
        \caption{METR-LA}
    \end{subfigure}
    \begin{subfigure}{0.49\textwidth}
        \centering
        \includegraphics[width=\linewidth]{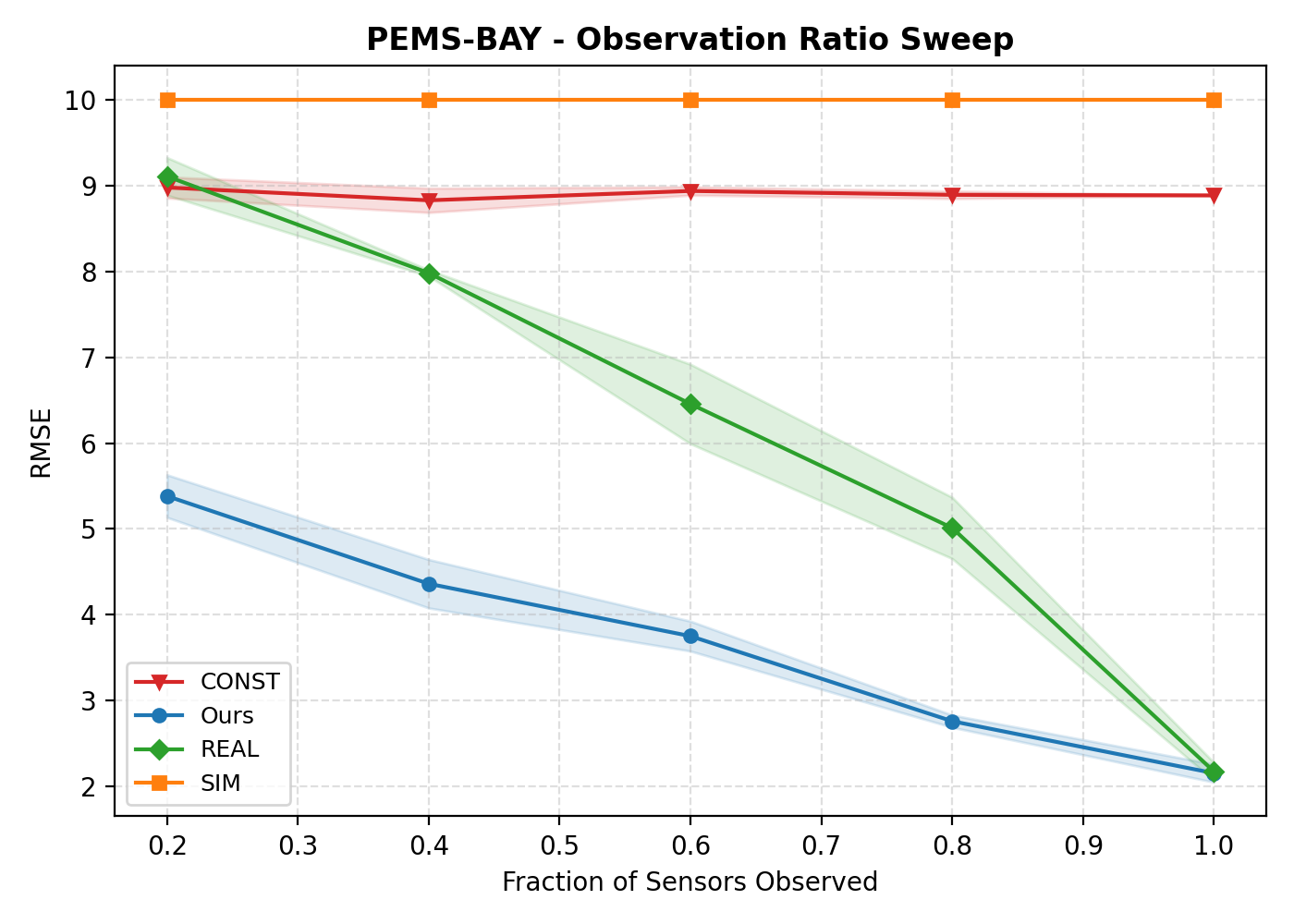}
        \caption{PEMS-BAY}
    \end{subfigure}
    \caption{Effect of the fraction of edges with real observations on RMSE of $\hat\mu$ for METR-LA and PEMS-BAY. The horizontal axis is the fraction of observable edges; the number of real samples per observable edge is fixed at $n_e = 20$. Results are averaged over seeds (lower is better).}
    \label{fig:edge-casestudy-ratio}
\end{figure}

\begin{figure}[!htbp]
    \centering
    \begin{subfigure}{\textwidth}
        \centering
        \includegraphics[width=0.8\linewidth]{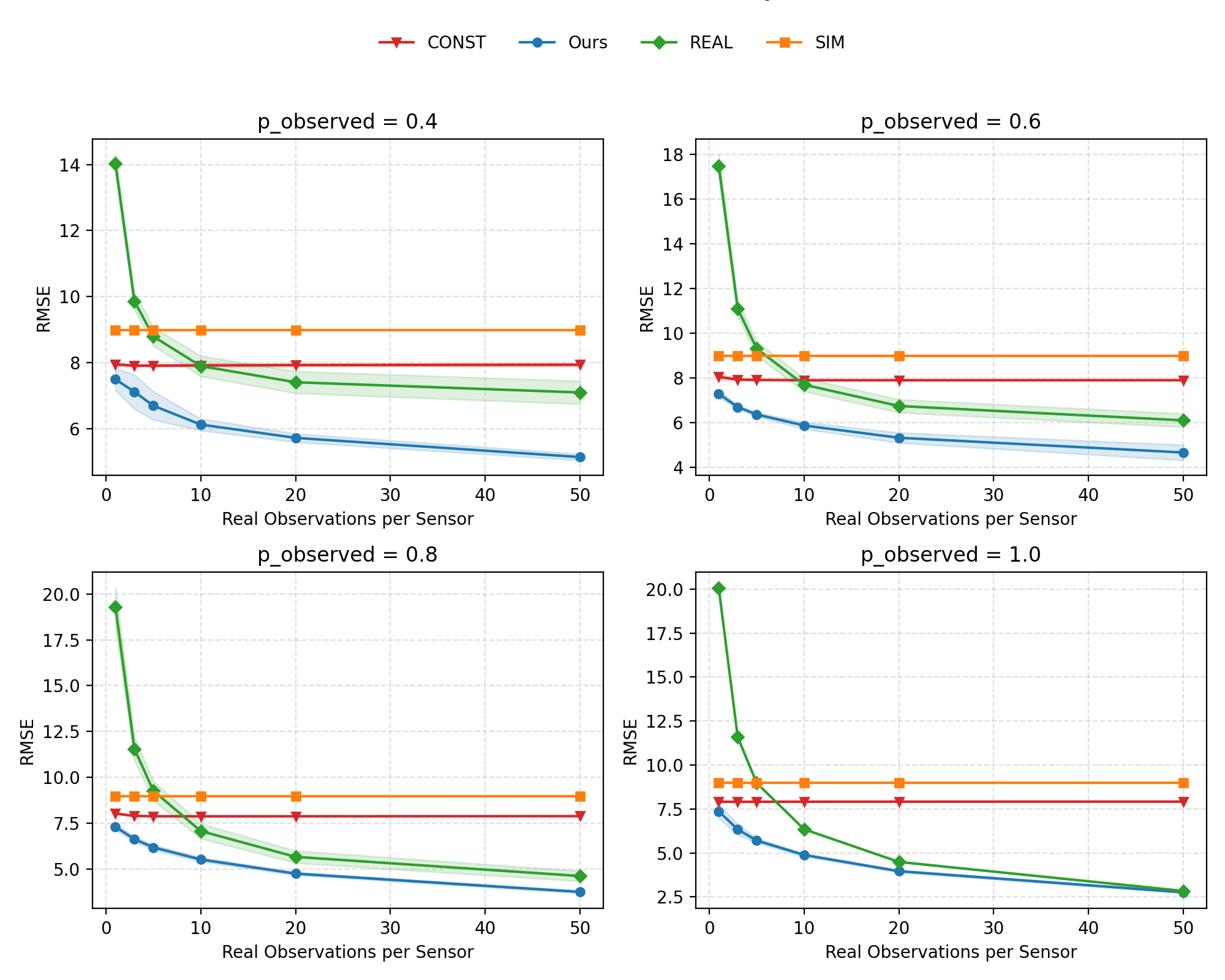}
        \caption{METR-LA}
    \end{subfigure}
    \medskip
    \begin{subfigure}{\textwidth}
        \centering
        \includegraphics[width=0.8\linewidth]{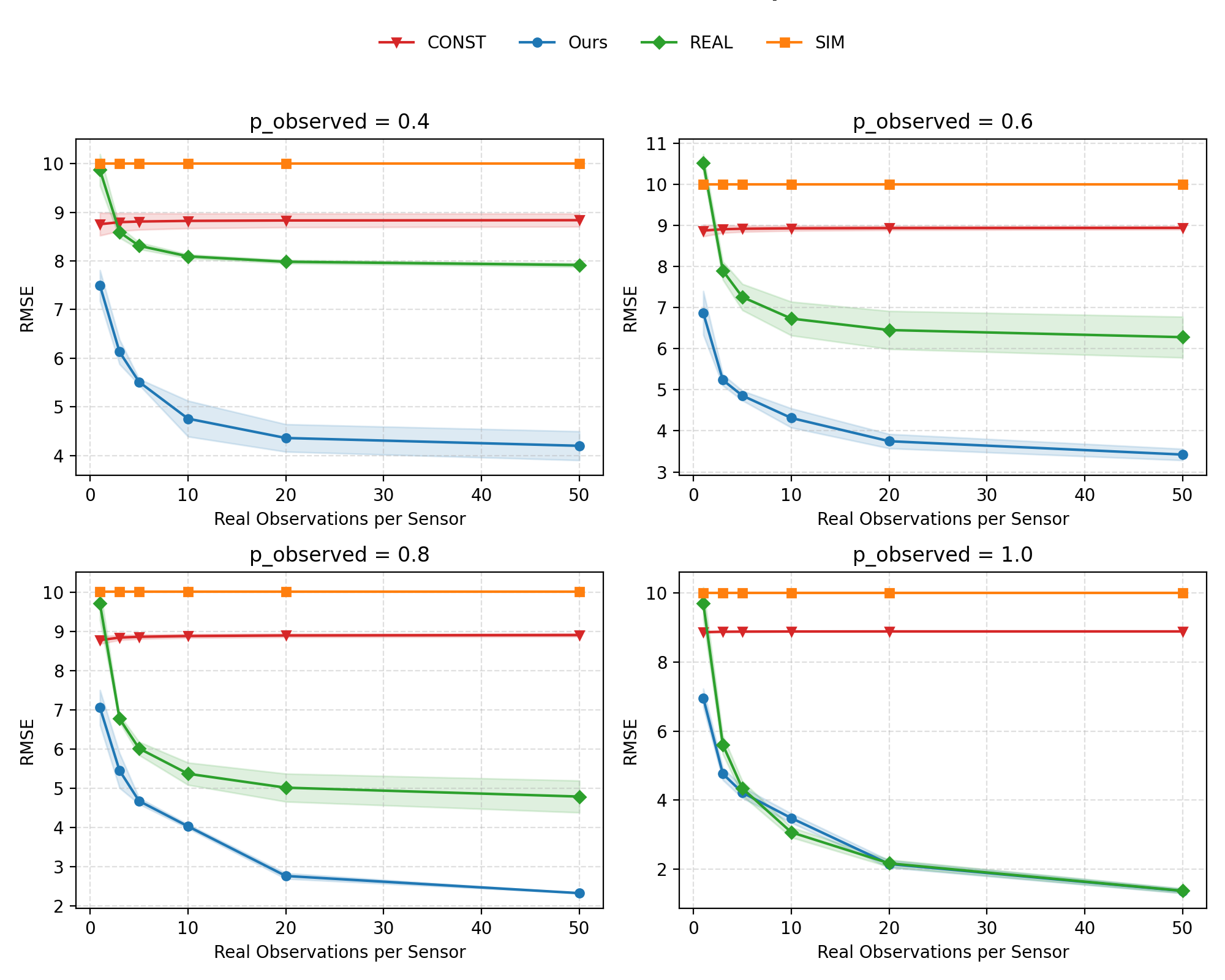}
        \caption{PEMS-BAY}
    \end{subfigure}
    \caption{Effect of the number of real samples per observable edge on RMSE of $\hat\mu$ for METR-LA and PEMS-BAY, where we use $p_\text{observed}$ to indicate the fraction of observable edges (e.g., $p_\text{observed}=0.4$ means $40\%$ edges have real samples, while others have none). Each curve corresponds to a different fraction of observable edges. Results are averaged over seeds (lower is better).}
    \label{fig:edge-casestudy-number}
\end{figure}

\begin{figure}[!htbp]
    \centering
    \begin{subfigure}{0.45\textwidth}
        \centering
        \includegraphics[width=\linewidth]{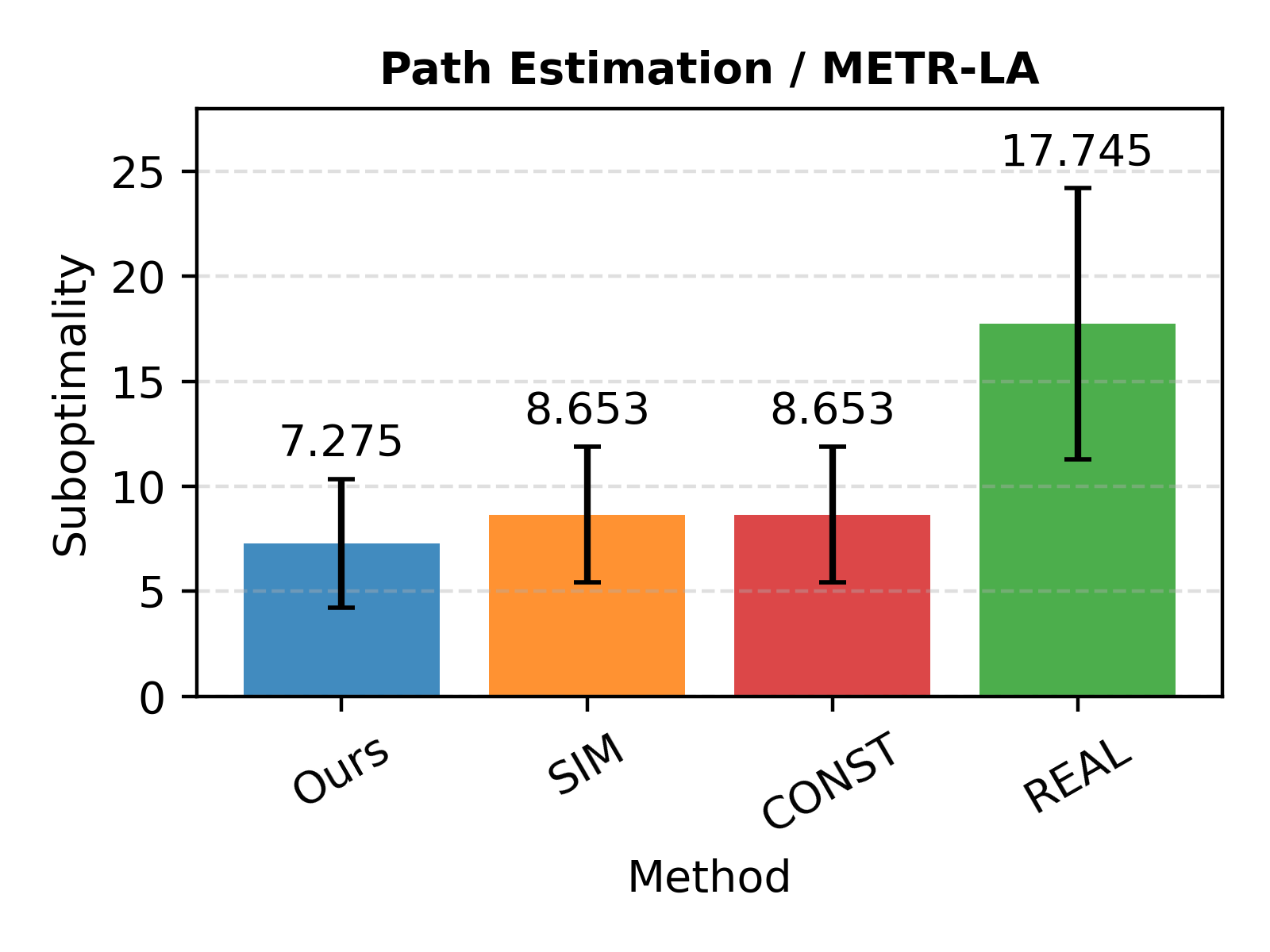}
        \caption{METR-LA}
    \end{subfigure}
    \begin{subfigure}{0.45\textwidth}
        \centering
        \includegraphics[width=\linewidth]{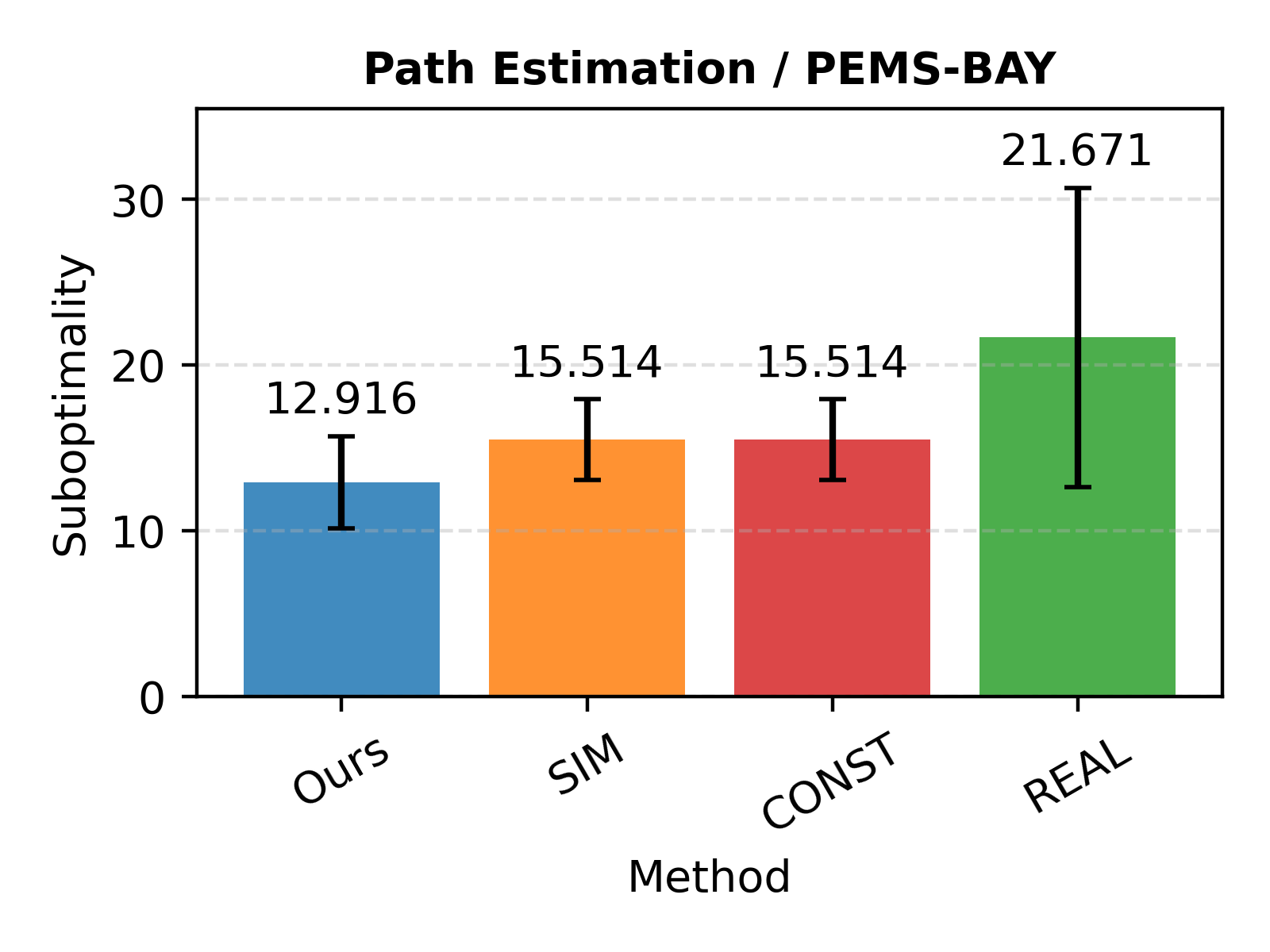}
        \caption{PEMS-BAY}
    \end{subfigure}
    \caption{Path-level performance evaluation for the METR-LA and PEMS-BAY under the original data distribution, averaged over seeds (lower is better).}
    \label{fig:edge-casestudy-path}
\end{figure}

The case study reinforces the conclusions from the synthetic graphs. Under the original data distribution (Figure~\ref{fig:edge-baseline}) where the smoothness constraint is implicitly determined by the dataset, our method achieves the lowest RMSE on both networks. In Figures~\ref{fig:edge-casestudy-ratio} and~\ref{fig:edge-casestudy-number}, our estimator smoothly interpolates between the synthetic and real regimes: when only a small fraction of edges has real data, it leans more heavily on the simulator while still borrowing structural information to outperform SIM, REAL, and CONST; as the quantity and coverage of real observations increase, it converges toward the real-only solution and can even improve upon REAL by regularizing noisy edges through the similarity graph. And the advantages in estimating single edges are further reflected in Figure~\ref{fig:edge-casestudy-path}, where our method achieves an overall better performance.

\subsection{Results: Shortest Path Estimation}

We then evaluate the active estimated shortest path algorithm \textsc{A-ESP} (Algorithm~\ref{alg:a-esp}) on METR-LA and PEMS-BAY. For each random seed, we sample a source-target pair such that there are at least ten distinct $v_\text{src}$--$v_\text{sink}$ paths in the underlying graph. All experiments start from the same initialization described above (one initial real query on each edge of $P^{\mathrm{sim}}$). Figure~\ref{fig:easp-path-cost} shows how the estimated cost of the current best path evolves as the number of real queries increases for a representative source-target pair. 
% Figure~\ref{fig:easp-optimal-ratio} reports, as a function of the confidence level used in the stopping rule, the empirical probability that the returned path coincides with the true shortest path. 
Figure~\ref{fig:easp-num-query} summarizes the average number of real queries required to recover the optimal path.  Unless otherwise stated, we always conduct 5 independent runs. The reported numbers are the means over the conducted runs with shaded areas indicating the standard deviation.

The results highlight the efficiency and reliability of \textsc{A-ESP}. Figure~\ref{fig:easp-path-cost} shows that, for a fixed query budget, \textsc{A-ESP} drives down the estimated path cost more quickly than the Random baseline by repeatedly querying the currently most uncertain edge. 
% Figure~\ref{fig:easp-optimal-ratio} illustrates the role of the confidence parameter: higher confidence levels yield a larger probability of returning the true shortest path but require more queries, capturing the trade-off between statistical guarantees and sampling effort. 
Consistent with this picture, Figure~\ref{fig:easp-num-query} demonstrates that, across confidence levels, \textsc{A-ESP} requires substantially fewer real queries than the Random baseline to achieve the same probability of identifying the optimal path.

\begin{figure}[!htbp]
    \centering
    \includegraphics[width=\linewidth]{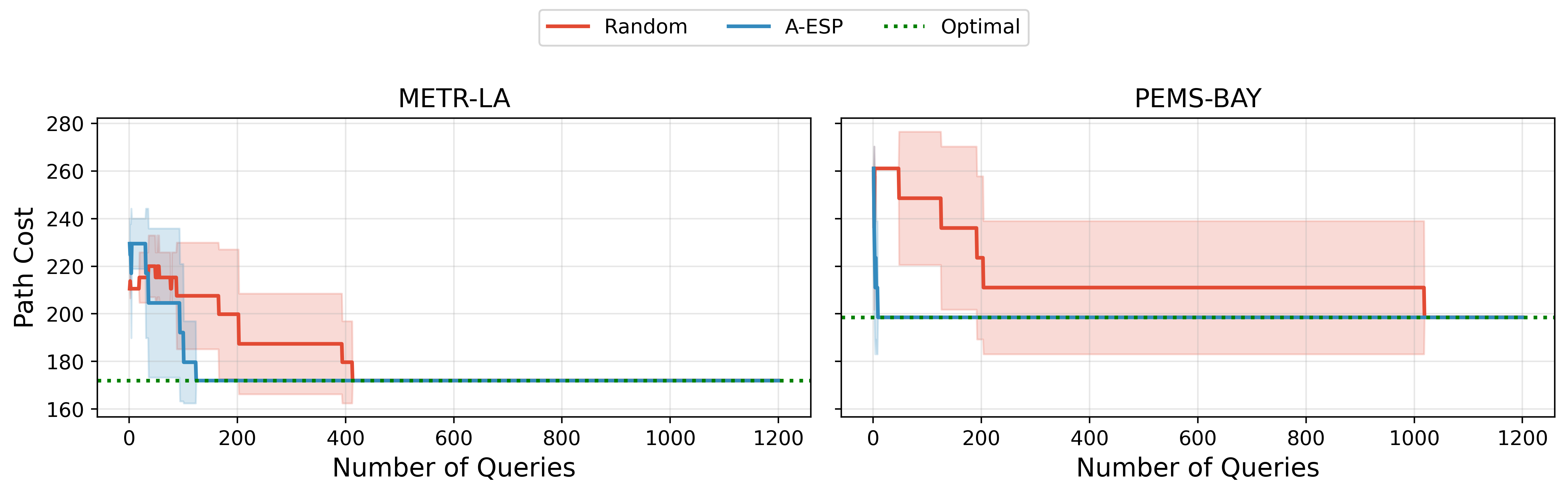}
    \caption{Evolution of the estimated shortest-path cost as the number of real queries increases, comparing \textsc{A-ESP} with the Random baseline on METR-LA and PEMS-BAY. Curves are averaged over multiple runs (lower is better).}
    \label{fig:easp-path-cost}
\end{figure}

% \begin{figure}[t]
%     \centering
%     \includegraphics[width=\linewidth]{figs/aesp/sp_stopping_combined_optimal_ratio.png}
%     \caption{Probability that the path returned by \textsc{A-ESP} coincides with the true shortest path as a function of the confidence level used in the stopping rule, on METR-LA and PEMS-BAY.}
%     \label{fig:easp-optimal-ratio}
% \end{figure}

\begin{figure}[!htbp]
    \centering
    \includegraphics[width=\linewidth]{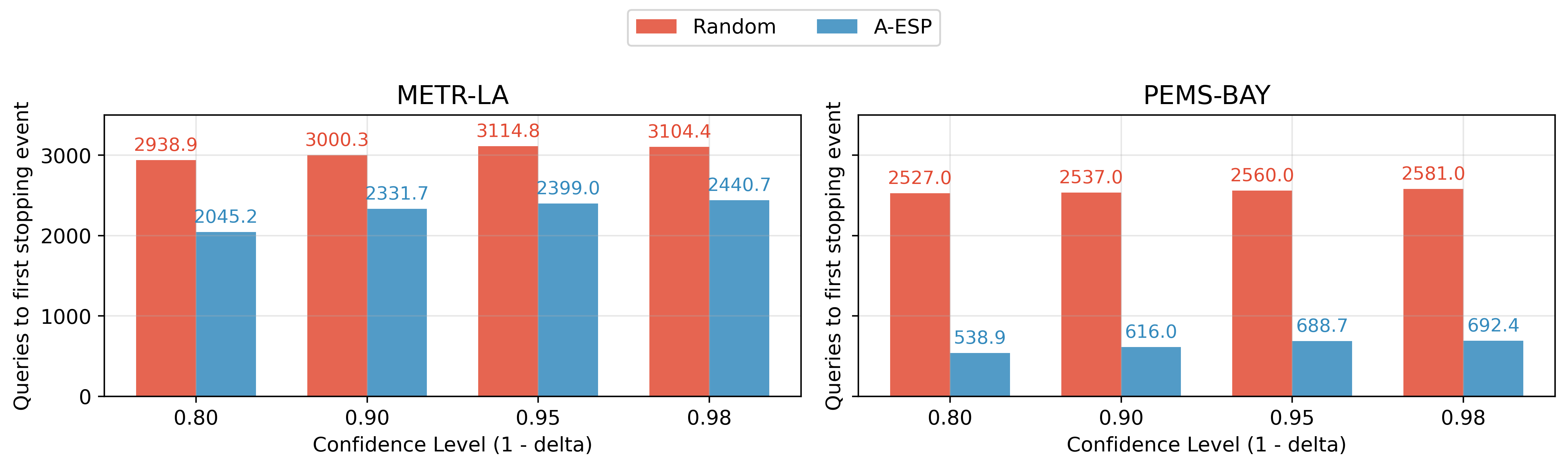}
    \caption{Average number of real queries required to recover the true shortest path with a given confidence level, comparing \textsc{A-ESP} with the Random baseline on METR-LA and PEMS-BAY.}
    \label{fig:easp-num-query}
\end{figure}

\section{Conclusion}

In this paper, we study a stochastic shortest-path problem in which the decision maker has access to a biased but inexpensive simulator and a limited number of costly real measurements. We model the simulator--real discrepancy as a smooth signal over an edge-similarity graph and develop a Laplacian-regularized estimator that calibrates synthetic costs with real observations. Our analysis provides finite-sample guarantees for edgewise estimation and path suboptimality, explicitly characterizing how performance depends on the smoothness and magnitude of the bias, the amount and quality of real data, and structural properties of the underlying network. Building on these estimators, we propose the Active Estimated Shortest Path algorithm, which uses pathwise confidence bounds to adaptively select edges to query and to certify an optimal (or near-optimal) path under the true costs, together with sample-complexity guarantees on the number of real measurements required.

We complement the theoretical results with numerical experiments on both synthetic graphs and real urban traffic networks. Across a range of noise levels and bias patterns, the proposed methods consistently outperform natural baselines that rely solely on real observations, solely on synthetic data, or on simple global calibration, especially when real data are sparse or highly noisy. The active learning procedure substantially reduces the number of real queries needed to identify high-quality routes compared to naive sampling rules, while adapting to heterogeneous edge uncertainties and varying simulator bias. 

Several directions for future research remain. One is to extend our framework to dynamic or time-varying networks, where both real and synthetic costs evolve over time. Another is to incorporate richer side information and learned similarity structures into the bias-regularization graph, and to study robust variants that relax the smoothness assumptions or account for misspecification of the simulator. More broadly, we believe the ideas of using biased simulators as structured priors and combining them with targeted real measurements can be applied to other network design and control problems beyond shortest-path routing.

\newpage
% Bibliography
\bibliographystyle{plainnat}
\bibliography{main}
\appendix
%\section{More Discussions}
%\section{Additional Discussions}
\section{Data-driven tuning for $\lambda$}
\label{appx:lambda_tuning}
Here we provide two simple data‑driven rules in practice.
\begin{enumerate}
\item \emph{K‑fold cross‑validation over edges with \(n_e>0\).}
Partition \(\{e:\,n_e>0\}\) into \(K\) folds.
For a candidate \(\lambda\), fit \(\hat b\) by solving \eqref{eqn:b_hat} after \emph{removing} the fidelity terms for one fold at a time (i.e., set \(w_e=0\) for \(e\) in the held‑out fold), and score that fold by
\[
\sum_{e\ \text{in fold}} w_e \,\bigl(\hat b_e - (\bar c_e-\bar c'_e)\bigr)^2 .
\]
Average the score over folds and pick the \(\lambda\) with the smallest average. (A logarithmic grid over \(\lambda\) is convenient.)

\item \emph{Discrepancy principle.}
If the noise level is roughly known from replicates, choose the smallest \(\lambda\) such that the in‑sample discrepancy
\[
\sum_{e:\,n_e>0} w_e \,\bigl(\hat b_e - (\bar c_e-\bar c'_e)\bigr)^2
\]
is on the order of the total expected noise (approximately the number of edges with \(n_e>0\) when the weights \(w_e\) are inverse‑variance weights). This guards against overfitting the noise.
\end{enumerate}

Other techniques, such as Stein’s Unbiased Risk Estimate (SURE) \citep{jorion1986bayes,stein1986lectures} or Generalized Cross Validation (GCV) \citep{golub1979generalized}, can also be adopted by the reader's discretion. 
\section{Proofs}
\label{appx:proofs}

\subsection{Proof of Lemma \ref{lem:uniqueM}}
\begin{proof}
Write the objective as
\[
F(b)\;=\;(b-y)^\top M(b-y) + \lambda\, b^\top L b,
\qquad b\in\mathbb{R}^{|E|}.
\]
Since $M=\mathrm{diag}(w_e)\succeq 0$ and $L$ is a (combinatorial) graph Laplacian, $L\succeq 0$ and
\[
\nabla^2 F(b)\;=\;2(M+\lambda L)\;\succeq\;0.
\]
To show strict convexity (hence uniqueness), it suffices to prove $M+\lambda L\succ 0$. Suppose $v\neq 0$ satisfies $v^\top(M+\lambda L)v=0$. Then necessarily $v^\top M v=0$ and $v^\top L v=0$. The first equality gives $v_e=0$ for every edge $e$ with $w_e>0$; the second gives (using the identity $x^\top L x=\frac12\sum_{e,e'}W_{e,e'}(x_e-x_{e'})^2$) that $v$ is constant on every connected component of the edge-edge graph induced by $W$. If, on each connected component, at least one edge has $w_e>0$ (which holds, e.g., when the component containing all source-to-sink edges is connected and at least one such edge is observed), the two conditions force $v=0$, a contradiction. Hence $M+\lambda L\succ 0$, and $F$ is strictly convex.

The unique stationary point solves
\[
\nabla F(\hat b)=2M(\hat b-y)+2\lambda L\hat b \;=\;0
\quad\Longleftrightarrow\quad
(M+\lambda L)\hat b \;=\; M y,
\]
and, since $M+\lambda L$ is invertible,
\[
\hat b \;=\; (M+\lambda L)^{-1}M\,y.
\]
\end{proof}

\subsection{Proof of Theorem \ref{thm:edgewise}}
\begin{proof}
By Lemma~\ref{lem:uniqueM}, recalling the notation
\[
H \;=\; M^{-1/2} L M^{-1/2}\succeq 0,
\qquad
S_\lambda \;=\; (I+\lambda H)^{-1},
\]
we can rewrite the estimator as
\[
\hat b \;=\; (M+\lambda L)^{-1}M\,y \;=\; M^{-1/2}S_\lambda M^{1/2} y.
\]
Decompose the empirical bias vector as $y=b^\star+\xi$, where $\xi$ is the centered noise defined in Assumption~\ref{ass:noiseM}. Then
\begin{align*}
\hat b - b^\star
&= M^{-1/2}S_\lambda M^{1/2}\xi
     \;-\; \lambda\,M^{-1/2}S_\lambda H M^{1/2} b^\star .
\end{align*}
Fix any $e\in E^+$, and denote by $e_e$ the $e$-th standard basis vector. Taking the $e$-th coordinate and applying the triangle and Cauchy-Schwarz inequalities,
\begin{align*}
\bigl| \hat b_e - b^\star_e \bigr|
&\le
\underbrace{\lambda\,\bigl\|H^{1/2}S_\lambda M^{-1/2} e_e\bigr\|_2
            \cdot \bigl\|H^{1/2}M^{1/2} b^\star\bigr\|_2}_{\text{(bias)}}
\;+\;
\underbrace{\bigl| \langle S_\lambda M^{-1/2} e_e,\; M^{1/2}\xi \rangle \bigr|}_{\text{(variance)}} .
\end{align*}
\paragraph{Bias term.}
Since $H\succeq 0$ has nonnegative eigenvalues $\{\eta_j\}_{j}$, the operator norm of $H^{1/2}S_\lambda$ equals
\[
\|H^{1/2}S_\lambda\|_{\mathrm{op}}
=\max_{j}\frac{\sqrt{\eta_j}}{1+\lambda\eta_j}
\;\le\;\frac{1}{2\sqrt{\lambda}},
\]
where the bound follows by maximizing $x\mapsto \sqrt{x}/(1+\lambda x)$ over $x\ge 0$. Therefore
\[
\lambda\,\bigl\|H^{1/2}S_\lambda M^{-1/2} e_e\bigr\|_2
\;\le\;
\frac{\sqrt{\lambda}}{2}\,\|M^{-1/2}e_e\|_2
\;=\; \frac{\sqrt{\lambda}}{2}\,w_e^{-1/2}.
\]
Assumption~\ref{ass:smoothM} gives $\|H^{1/2}M^{1/2} b^\star\|_2=\sqrt{{b^\star}^\top L b^\star}\le B$, hence
\[
\text{(bias)}\ \le\ \frac{\sqrt{\lambda}}{2}\,B\,w_e^{-1/2}.
\]

\paragraph{Variance term.}
Let $a_e=S_\lambda M^{-1/2}e_e$ and write
\[
\langle a_e,M^{1/2}\xi\rangle=\sum_{j} \underbrace{\bigl(a_{e,j}\sqrt{w_j}\bigr)}_{=:c_j}\,\xi_j.
\]
Since $\{\xi_j\}$ are independent, centered, and sub-Gaussian with proxies $\{\nu_j^2\}$ (i.e., $\|\xi_j\|_{\psi_2}\le \nu_j$), then the linear form $S:=\sum_j c_j\xi_j$ is sub-Gaussian with proxy
\[
\sigma^2 \;\le\; \sum_j c_j^2 \nu_j^2 \;\le\; \kappa_+ \sum_j a_{e,j}^2 \;=\; \kappa_+ \,\|a_e\|_2^2,
\]
where we used $w_j\nu_j^2\le \kappa_+$.
The standard sub-Gaussian tail bound gives, for every $\delta\in(0,1)$,
\[
\mathbb{P}\!\left(
|S|\;\le\; \sigma \sqrt{2\log(2/\delta)}
\right)\ \ge\ 1-\delta.
\]
Thus we have
\[
\mathbb{P}\!\left(
\bigl| \langle a_e, M^{1/2}\xi \rangle \bigr|
\;\le\; \underbrace{\sqrt{\kappa_+}}_{=\,\sigma/\|a_e\|_2}\,\|a_e\|_2\,\sqrt{2\log(2/\delta)}
\right)\ \ge\ 1-\delta.
\]
See, e.g., \cite{vershynin2018high,wainwright2019high} for the sub-Gaussian norm and tail inequality for linear forms.

\paragraph{Uniform (maximal) bound.}
Applying a union bound over $e\in E^+$ and using $\alpha_e(\lambda)\le \alpha_\infty(\lambda)$ yields
\[
\mathbb{P}\!\left(
\max_{e\in E^+}\bigl|\hat b_e-b^\star_e\bigr|
\;\le\;
\frac{\sqrt{\lambda}}{2}\,B\,w_{\min}^{-1/2}
\;+\;
\,\sqrt{\kappa_+}\,\alpha_\infty(\lambda)\,\sqrt{2\log\!\bigl(2|E|/\delta\bigr)}
\right)\ \ge\ 1-\delta ,
\]
which completes the proof.
\end{proof}

\subsection{Proof of Theorem \ref{thm:two_path_bound}}
\begin{proof}
For any path \(P\in\mathcal P\), write
\[
\hat\mu(P)\coloneqq \sum_{e\in E(P)}\hat\mu_e,
\qquad
\mu(P)\coloneqq \sum_{e\in E(P)}\mu_e .
\]
By definition of \(\widehat P\) as a shortest path under \(\hat\mu\), we have
\(\hat\mu(\widehat P)\le \hat\mu(P^\star)\).

Moreover, for any path \(P\),
\[
\mu(P)
=\sum_{e\in E(P)}\mu_e
=\sum_{e\in E(P)}\hat\mu_e + \sum_{e\in E(P)}(\mu_e-\hat\mu_e)
=\hat\mu(P) + \sum_{e\in E(P)}(\mu_e-\hat\mu_e).
\]
Therefore,
\begin{align*}
\mu(\widehat P)-\mu(P^\star)
&=
\bigl[\hat\mu(\widehat P)-\hat\mu(P^\star)\bigr]
+
\sum_{e\in E(\widehat P)}(\mu_e-\hat\mu_e)
-
\sum_{e\in E(P^\star)}(\mu_e-\hat\mu_e)
\\
&\le
\sum_{e\in E(\widehat P)}|\mu_e-\hat\mu_e|
+
\sum_{e\in E(P^\star)}|\mu_e-\hat\mu_e|
\\
&\le
\sum_{e\in E(\widehat P)}\beta_e
+
\sum_{e\in E(P^\star)}\beta_e,
\end{align*}
where the last inequality uses the assumed edgewise bounds
\(|\hat\mu_e-\mu_e|\le \beta_e\) for all \(e\in E\).
This proves the first claim.

For the uniform bound, if \(\beta_e\le \bar\beta\) for all \(e\), then
\[
\mu(\widehat P)-\mu(P^\star)
\le
\bar\beta\,|E(\widehat P)| + \bar\beta\,|E(P^\star)|
\le 2\,\bar\beta\,L_{\max},
\]
since \(|E(P)|\le L_{\max}\) for every $v_\text{src}$--$v_\text{sink}$ path \(P\in\mathcal P\).
\end{proof}

% \subsection{Proofs for Section \ref{subsec:active-guarantees}}
% \label{appx:active-proofs}

\subsection{Proof of Theorem \ref{thm:aesp-correctness}}

\begin{proof}
Fix \(t\ge1\).
The estimator at round $t$ is
\[
\hat b_t=(M_t+\lambda L)^{-1}M_t\,y_t,\qquad
\hat\mu_e(t)=\bar c'_e+\hat b_{t,e}.
\]
Using \(M_t^{1/2}(I+\lambda H_t)M_t^{1/2}=M_t+\lambda L\) with \(H_t=M_t^{-1/2}LM_t^{-1/2}\), we can rewrite
\[
\hat b_t
= (M_t+\lambda L)^{-1}M_t\,y_t
= M_t^{-1/2}S_{\lambda,t} M_t^{1/2}\,y_t,
\qquad
S_{\lambda,t}=(I+\lambda H_t)^{-1}.
\]
Decompose the empirical bias vector as \(y_t=b^\star+\xi_t\), where \(b^\star_e=\mu_e-\mu'_e\) and \(\xi_{t,e}\) is the centered noise from Assumption~\ref{ass:noiseM_new}.
As in the static analysis,
\[
\hat b_t-b^\star
= M_t^{-1/2}S_{\lambda,t} M_t^{1/2}\,\xi_t
\;-\; \lambda\,M_t^{-1/2}S_{\lambda,t}H_t M_t^{1/2}\,b^\star.
\]
Taking the \(e\)-th coordinate and applying the triangle inequality and Cauchy--Schwarz,
\[
\bigl|\hat b_{t,e}-b^\star_e\bigr|
\ \le\ 
\underbrace{\lambda\bigl\|H_t^{1/2}S_{\lambda,t}M_t^{-1/2}e_e\bigr\|_2\,\|H_t^{1/2}M_t^{1/2}b^\star\|_2}_{\text{bias}}
\;+\;
\underbrace{\bigl|\langle S_{\lambda,t}M_t^{-1/2}e_e,\ M_t^{1/2}\xi_t\rangle\bigr|}_{\text{variance}}.
\]

For the bias term, we use the same spectral argument as in Theorem~\ref{thm:edgewise}.
The eigenvalues of \(\sqrt{\lambda} H_t^{1/2}S_{\lambda,t}\) are \(\sqrt{\lambda h}/(1+\lambda h)\in[0,1/2]\) for \(h\ge0\), hence
\[
\lambda\bigl\|H_t^{1/2}S_{\lambda,t}\bigr\|_{\mathrm{op}}\ \le\ \frac{1}{2}\sqrt{\lambda}.
\]
Together with \(\|M_t^{-1/2}e_e\|_2=w_e(t)^{-1/2}\) and \(\|H_t^{1/2}M_t^{1/2}b^\star\|_2=\|b^\star\|_L\le B\) (Assumption~\ref{ass:smoothM}), this yields
\[
\text{(bias)}\ \le\ \frac{\sqrt{\lambda}}{2}\,B\,w_e(t)^{-1/2}.
\]

For the variance term, Lemma~\ref{lem:anytime-sg} gives, with probability at least \(1-\delta\), simultaneously for all \(t\ge1\) and all \(e\in E\),
\[
\bigl|\langle S_{\lambda,t}M_t^{-1/2}e_e,\ M_t^{1/2}\xi_t\rangle\bigr|
\ \le\ \sqrt{\kappa_+}\,\alpha_e(\lambda;M_t)\,\sqrt{2\log\!\Bigl(\tfrac{2|E|\,\pi^2 t^2}{3\delta}\Bigr)}.
\]
Combining the two bounds and recalling that \(\hat\mu_e(t)-\mu_e=\hat b_{t,e}-b^\star_e\), we obtain
\[
\bigl|\hat\mu_e(t)-\mu_e\bigr|\ \le\ \beta_e(t)
\]
for all $t$ and \(e\) on an event of probability at least \(1-\delta\).
Summing over \(e\in E(P)\) yields
\[
\bigl|\hat\mu(P,t)-\mu(P)\bigr|\ \le\ \sum_{e\in E(P)}\beta_e(t)\ =\ \beta(P,t),
\]
so \(\mu(P)\in[\mathrm{LCB}(P,t),\mathrm{UCB}(P,t)]\) simultaneously for all \(P\) and $t$, proving the first claim.

For correctness: if \(\mathrm{UCB}(\widehat P_t,t)\le \mathrm{LCB}(\widetilde P_t,t)\), then for any \(P\in\mathcal P\),
\[
\mathrm{UCB}(\widehat P_t,t)\ \le\ \mathrm{LCB}(\widetilde P_t,t)\ \le\ \mathrm{LCB}(P,t).
\]
On \(\mathcal E_\delta\) we have \(\mu(P)\in[\mathrm{LCB}(P,t),\mathrm{UCB}(P,t)]\), hence
\(\mu(\widehat P_t)\le \mathrm{UCB}(\widehat P_t,t)\le \mathrm{LCB}(P,t)\le \mu(P)\) for all \(P\), i.e., \(\widehat P_t=P^\star\).
\end{proof}

\subsubsection{Lemma \ref{lem:alpha-weight} with Proof}
\label{appx:active-proofs}

\begin{lemma}[Leverage and weight]
\label{lem:alpha-weight}
Let \(H_t=M_t^{-1/2}LM_t^{-1/2}\succeq 0\) and \(S_{\lambda,t}=(I+\lambda H_t)^{-1}\).
For any edge \(e\) with \(w_e(t)>0\),
\[
\alpha_e(\lambda;M_t)\;=\;\bigl\|S_{\lambda,t}M_t^{-1/2}e_e\bigr\|_2\ \le\ \frac{1}{\sqrt{w_e(t)}}.
\]
\end{lemma}
\begin{proof}
By definition, \(H_t\succeq 0\) implies \(0\preceq S_{\lambda,t}\preceq I\), hence \(\|S_{\lambda,t}\|_{\mathrm{op}}\le 1\).
Moreover \(M_t^{-1/2}e_e=e_e/\sqrt{w_e(t)}\), so
\[
\alpha_e(\lambda;M_t)\;=\;\bigl\|S_{\lambda,t}M_t^{-1/2}e_e\bigr\|_2
\ \le\ \|S_{\lambda,t}\|_{\mathrm{op}}\,\|M_t^{-1/2}e_e\|_2
\ \le\ \frac{1}{\sqrt{w_e(t)}}.
\]
\end{proof}

\subsubsection{Lemma \ref{lem:anytime-sg} with Proof}

\begin{lemma}[Anytime sub-Gaussian control]
\label{lem:anytime-sg}
Fix \(\delta\in(0,1)\) and an adapted process with per-edge noises as in Assumption~\ref{ass:noiseM_new}.
Let \(M_t=\mathrm{diag}(w_e(t))\), \(H_t=M_t^{-1/2}LM_t^{-1/2}\), and \(S_{\lambda,t}=(I+\lambda H_t)^{-1}\).
Write \(y_t=b^\star+\xi_t\), where \(b^\star_e=\mu_e-\mu'_e\) and \(\xi_{t,e}\) is the centered noise with conditional sub-Gaussian proxy \(\nu_e(t)\).
Then, with probability at least \(1-\delta\),
\[
\forall\,t\ge 1,\ \forall\,e\in E:\quad
\Bigl| \langle S_{\lambda,t}M_t^{-1/2}e_e,\ M_t^{1/2}\xi_t\rangle \Bigr|
\ \le\ \sqrt{\kappa_+}\,\alpha_e(\lambda;M_t)\,\sqrt{2\log\!\Bigl(\tfrac{2|E|\,\pi^2 t^2}{3\delta}\Bigr)}.
\]
\end{lemma}
\begin{proof}
The proof mirrors the static argument in Theorem~\ref{thm:edgewise} and uses a union bound over \((t,e)\).
Let \(\mathcal F_{t-1}\) be the \(\sigma\)-algebra generated by the past up to round \(t-1\).
Define
\[
a_{t,e}\;:=\;S_{\lambda,t}\,M_t^{-1/2}e_e\in\mathbb R^{|E|},
\qquad
Z_{t,e}\;:=\;\bigl\langle a_{t,e},\,M_t^{1/2}\,\xi_t\bigr\rangle
\;=\;\sum_{j\in E}\bigl(a_{t,e}\bigr)_j\sqrt{w_j(t)}\,\xi_{t,j}.
\]
Since \(M_t\) and \(S_{\lambda,t}\) are \(\mathcal F_{t-1}\)-measurable, the coefficient vector \(\{a_{t,e,j}\sqrt{w_j(t)}\}_j\) is fixed given \(\mathcal F_{t-1}\).
By Assumption~\ref{ass:noiseM_new}, each \(\xi_{t,j}\) is conditionally sub-Gaussian with proxy \(\nu_j(t)\), and
\[
\sigma_{t,e}^2\;:=\;\sum_{j\in E}\bigl(a_{t,e}\bigr)_j^2\,w_j(t)\,\nu_j(t)^2
\]
is a valid conditional variance proxy for \(Z_{t,e}\).
Using the calibration \(w_j(t)\nu_j(t)^2\le \kappa_+\) from Assumption~\ref{ass:noiseM_new}, we obtain
\begin{equation}\label{eq:proxy-upper}
\sigma_{t,e}^2\ \le\ \kappa_+\sum_{j}\bigl(a_{t,e}\bigr)_j^2
\ =\ \kappa_+\,\|a_{t,e}\|_2^2
\ =\ \kappa_+\,\alpha_e(\lambda;M_t)^2.
\end{equation}
Conditional on \(\mathcal F_{t-1}\), the standard sub-Gaussian tail bound therefore yields, for any \(\delta_{t,e}\in(0,1)\),
\begin{equation}\label{eq:fixed-t-e-final}
\Pr\!\left(\,\bigl|Z_{t,e}\bigr|
\ \le\ \sqrt{\kappa_+}\,\alpha_e(\lambda;M_t)\,\sqrt{2\log\frac{2}{\delta_{t,e}}}
\ \middle|\ \mathcal F_{t-1}\right)\ \ge\ 1-\delta_{t,e}.
\end{equation}

Now choose
\[
\delta_{t,e}\ :=\ \frac{3\delta}{\pi^2 t^2\,|E|}.
\]
We have \(\sum_{e\in E}\delta_{t,e}=3\delta/(\pi^2 t^2)\).
Applying a union bound over \(e\in E\) in~\eqref{eq:fixed-t-e-final} gives, for the whole vector at time $t$,
\begin{equation}\label{eq:union-over-e}
\Pr\!\left(\,\forall e\in E:\ \bigl|Z_{t,e}\bigr|
\ \le\ \sqrt{\kappa_+}\,\alpha_e(\lambda;M_t)\,\sqrt{2\log\!\Bigl(\tfrac{2|E|\,\pi^2 t^2}{3\delta}\Bigr)}
\ \middle|\ \mathcal F_{t-1}\right)\ \ge\ 1-\frac{3\delta}{\pi^2 t^2}.
\end{equation}

Finally, union bound over \(t=1,2,\dots\) using the weights \(3/(\pi^2 t^2)\):
\[
\sum_{t=1}^\infty \frac{3\delta}{\pi^2 t^2}\ =\ \frac{3\delta}{\pi^2}\sum_{t=1}^\infty \frac{1}{t^2}
\ =\ \frac{3\delta}{\pi^2}\cdot \frac{\pi^2}{6}\leq \delta.
\]
Thus, with probability at least \(1-\delta\), the event in \eqref{eq:union-over-e} holds simultaneously for all \(t\ge1\), which yields the stated bound.
\end{proof}

\subsection{Proof of Theorem \ref{thm:aesp-simple-sc}}

\begin{proof}[Proof of Theorem~\ref{thm:aesp-simple-sc}]
For simplicity of notation, we denote $m=|E|$ as the total number of edges in the given graph. Work on the event \(\mathcal E_\delta\) from Theorem~\ref{thm:aesp-correctness}, i.e.,
\(\mu(P)\in[\mathrm{LCB}(P,t),\mathrm{UCB}(P,t)]\) for all paths \(P\) and all rounds $t$.

At round $t$, let
\[
\widehat P_t\in\arg\min_{P\in\mathcal P} \hat\mu(P,t)
\quad\text{and}\quad
\tilde P_t\in\arg\min_{P\in\mathcal P\setminus\{\widehat P_t\}} \mathrm{LCB}(P,t)
\]
denote the empirical best path and the most optimistic challenger, and set
\(S_t:=E(\widehat P_t)\cup E(\tilde P_t)\).

The stopping test is \(\mathrm{UCB}(\widehat P_t,t)\le \mathrm{LCB}(\tilde P_t,t)\), equivalently
\[
\hat\mu(\tilde P_t,t)-\hat\mu(\widehat P_t,t)\ \ge\ \beta(\widehat P_t,t)+\beta(\tilde P_t,t).
\]
On \(\mathcal E_\delta\), we also have
\[
\hat\mu(\tilde P_t,t)-\hat\mu(\widehat P_t,t)
\ \ge\ \mu(\tilde P_t)-\mu(\widehat P_t)-\beta(\widehat P_t,t)-\beta(\tilde P_t,t),
\]
so it suffices to enforce
\begin{equation}
2\bigl[\beta(\widehat P_t,t)+\beta(\tilde P_t,t)\bigr]\ \le\ \mu(\tilde P_t)-\mu(\widehat P_t).
\label{A1-new}
\end{equation}

\medskip\noindent\textbf{Step 1: A sufficient \(\ell_2\) condition.}
Since each edge in \(S_t\) can appear in at most two of the two path-sums,
\[
\beta(\widehat P_t,t)+\beta(\tilde P_t,t)
=\sum_{e\in E(\widehat P_t)}\beta_e(t)+\sum_{e\in E(\tilde P_t)}\beta_e(t)
\ \le\ 2\sum_{e\in S_t}\beta_e(t).
\]
By Cauchy--Schwarz and \(|S_t|\le |E(\widehat P_t)|+|E(\tilde P_t)|\le 2L_{\max}\),
\[
\sum_{e\in S_t}\beta_e(t)
\ \le\
\sqrt{|S_t|}\,\Bigl(\sum_{e\in S_t}\beta_e(t)^2\Bigr)^{1/2}
\ \le\
\sqrt{2L_{\max}}\,\Bigl(\sum_{e\in S_t}\beta_e(t)^2\Bigr)^{1/2}.
\]
Hence,
\[
\beta(\widehat P_t,t)+\beta(\tilde P_t,t)
\ \le\
2\sqrt{2L_{\max}}\,
\Bigl(\sum_{e\in S_t}\beta_e(t)^2\Bigr)^{1/2}.
\]
Therefore a sufficient condition for~\eqref{A1-new} is
\begin{equation}
\sum_{e\in S_t}\beta_e(t)^2\ \le\ \frac{\bigl[\mu(\tilde P_t)-\mu(\widehat P_t)\bigr]^2}{32L_{\max}}.
\label{A2-new}
\end{equation}

\medskip\noindent\textbf{Step 2: Reducing to the gap \(\Delta\) without assuming \(\widehat P_t=P^\star\).}
We claim that, on \(\mathcal E_\delta\), the stronger condition
\begin{equation}
\sum_{e\in S_t}\beta_e(t)^2\ \le\ \frac{\Delta^2}{32L_{\max}}
\label{A3-new}
\end{equation}
already implies that the algorithm stops at round $t$.

Indeed, under~\eqref{A3-new}, we have
\[
\beta(\widehat P_t,t)+\beta(\tilde P_t,t)
\ \le\
2\sqrt{2L_{\max}}\cdot \frac{\Delta}{\sqrt{32L_{\max}}}
\ =\ \frac{\Delta}{2}.
\]
Also,
\[
\beta(\tilde P_t,t)
\le \sum_{e\in E(\tilde P_t)}\beta_e(t)
\le \sqrt{L_{\max}}\Bigl(\sum_{e\in S_t}\beta_e(t)^2\Bigr)^{1/2}
\le \frac{\Delta}{\sqrt{32}}
<\frac{\Delta}{2}.
\]

If \(\widehat P_t\neq P^\star\), then \(P^\star\) is feasible in the definition of \(\tilde P_t\), so
\[
\mathrm{LCB}(\tilde P_t,t)\ \le\ \mathrm{LCB}(P^\star,t)\ \le\ \mu(P^\star),
\]
where the last inequality uses \(\mathrm{LCB}(P,t)=\hat\mu(P,t)-\beta(P,t)\le \mu(P)\) on \(\mathcal E_\delta\).
On the other hand, still on \(\mathcal E_\delta\),
\[
\mathrm{LCB}(\tilde P_t,t)=\hat\mu(\tilde P_t,t)-\beta(\tilde P_t,t)\ \ge\ \mu(\tilde P_t)-2\beta(\tilde P_t,t).
\]
Combining yields
\[
\mu(\tilde P_t)-\mu(P^\star)\ \le\ 2\beta(\tilde P_t,t)\ <\ \Delta,
\]
which forces \(\tilde P_t=P^\star\) by the definition of the gap \(\Delta\).

But if \(\tilde P_t=P^\star\) and \(\widehat P_t\) is the empirical best, then
\(\hat\mu(\widehat P_t,t)\le \hat\mu(P^\star,t)\), and on \(\mathcal E_\delta\),
\[
\mu(\widehat P_t)-\beta(\widehat P_t,t)\ \le\ \hat\mu(\widehat P_t,t)
\ \le\ \hat\mu(P^\star,t)\ \le\ \mu(P^\star)+\beta(P^\star,t).
\]
Thus \(\mu(\widehat P_t)-\mu(P^\star)\le \beta(\widehat P_t,t)+\beta(P^\star,t)=\beta(\widehat P_t,t)+\beta(\tilde P_t,t)\le \Delta/2\),
contradicting \(\mu(\widehat P_t)-\mu(P^\star)\ge \Delta\) when \(\widehat P_t\neq P^\star\).
Therefore \(\widehat P_t=P^\star\).

With \(\widehat P_t=P^\star\), we have \(\mu(\tilde P_t)-\mu(\widehat P_t)\ge \Delta\) (since \(\tilde P_t\neq \widehat P_t\)), and since
\(\beta(\widehat P_t,t)+\beta(\tilde P_t,t)\le \Delta/2\), condition~\eqref{A1-new} holds and the stopping rule is satisfied.

Consequently, it suffices to drive~\eqref{A3-new}.

\medskip\noindent\textbf{Step 3: Upper bound \(\sum_{e\in S_t}\beta_e(t)^2\).}
Let \(L_t:=\log\!\bigl(2m\pi^2 t^2/(3\delta)\bigr)\) as in~\eqref{eq:edge-ci-active}.
Decompose \(\beta_e(t)=b_e(t)+v_e(t)\), where
\[
b_e(t)=\sqrt{\tfrac{\lambda}{2}}\,\frac{B}{\sqrt{w_e(t)}},
\qquad
v_e(t)=\sqrt{\kappa_+}\,\alpha_e(\lambda;M_t)\,\sqrt{2L_t}.
\]
Using \((x+y)^2\le 2x^2+2y^2\), Lemma~\ref{lem:alpha-weight}, and \(\alpha_e(\lambda;M_t)^2\le 1/w_e(t)\), we have
\[
\beta_e(t)^2\ \le\ 2b_e(t)^2+2v_e(t)^2
\ \le\ \frac{\lambda B^2}{w_e(t)} + \frac{4\kappa_+L_t}{w_e(t)}
\;=\;\bigl(\lambda B^2+4\kappa_+L_t\bigr)\,\frac{1}{w_e(t)}.
\]
Summing over \(e\in S_t\) gives
\begin{equation}
\sum_{e\in S_t}\beta_e(t)^2\ \le\
\Bigl(\lambda B^2+4\kappa_+L_t\Bigr)\,
\sum_{e\in S_t}\frac{1}{w_e(t)}.
\label{A4-new}
\end{equation}

Under Assumption~\ref{ass:negligible-sim}, Assumption~\ref{ass:noiseM_new} implies that for all \(e\) with \(w_e(t)>0\),
\[
\frac{1}{w_e(t)}\ \le\ \frac{1}{\kappa_-}\,\nu_e(t)^2
\ \le\ \frac{1}{\kappa_-}\,\frac{\sigma_e^2}{n_e(t)}.
\]
Plugging into~\eqref{A4-new} yields
\begin{equation}
\sum_{e\in S_t}\beta_e(t)^2\ \le\
\frac{\lambda B^2+4\kappa_+L_t}{\kappa_-}\,
\sum_{e\in S_t}\frac{\sigma_e^2}{n_e(t)}.
\label{A5-new}
\end{equation}

\medskip\noindent\textbf{Step 4: Bounding the allocation term via Lemma~\ref{lem:balance-Dt}.}
Let \(T_E(t):=\sum_{e\in E}n_e(t)\) be the total number of real samples collected by round $t$.
In Algorithm~\ref{alg:a-esp}, the global rule \(e_t\in\arg\max_{e\in E}\sigma_e^2/n_e(t)\) is exactly the greedy
variance-balancing allocation on \(D=E\).
Thus, Lemma~\ref{lem:balance-Dt} applied with \(D=E\) and \(K=m\) gives, for all rounds with \(T_E(t)\ge 2m\),
\[
\sum_{e\in E}\frac{\sigma_e^2}{n_e(t)}\ \le\ \frac{2m\,\Sigma^2}{T_E(t)}.
\]
Moreover, since the initialization queries each edge of \(P^{\mathrm{sim}}\) once and then exactly one additional edge per round,
we have \(T_E(t)=|E(P^{\mathrm{sim}})|+(t-1)\ge t\), hence
\begin{equation}
\sum_{e\in S_t}\frac{\sigma_e^2}{n_e(t)}
\ \le\
\sum_{e\in E}\frac{\sigma_e^2}{n_e(t)}
\ \le\
\frac{2m\,\Sigma^2}{T_E(t)}
\ \le\
\frac{2m\,\Sigma^2}{t}.
\label{A6-new}
\end{equation}

Combining~\eqref{A5-new} and~\eqref{A6-new} gives
\[
\sum_{e\in S_t}\beta_e(t)^2\ \le\
\frac{\lambda B^2+4\kappa_+L_t}{\kappa_-}\,\frac{2m\,\Sigma^2}{t}.
\]
Thus a sufficient condition for~\eqref{A3-new} is
\[
\frac{\lambda B^2+4\kappa_+L_t}{\kappa_-}\,\frac{2m\,\Sigma^2}{t}
\ \le\ \frac{\Delta^2}{32L_{\max}},
\]
that is,
\[
t\ \ge\ C_0\cdot \frac{m\,L_{\max}}{\kappa_-\,\Delta^2}\,
\Bigl(\lambda B^2+4\kappa_+L_t\Bigr)\,\Sigma^2,
\]
for a suitable universal constant \(C_0>0\).

Finally, recall \(L_t=\log(2m\pi^2 t^2/(3\delta))=O(\log(2m/\delta)+\log t)\).
A standard self-consistency argument yields a universal constant \(C>0\) such that
\[
T_\delta\ \le\
C\cdot
\frac{m\,L_{\max}}{\kappa_-\,\Delta^2}\,
\Bigl(\lambda B^2 + \kappa_+ \log\frac{2m}{\delta}\Bigr)\,
\Sigma^2.
\]
On \(\mathcal E_\delta\), this \(T_\delta\) suffices to ensure \(\mathrm{UCB}(\widehat P_t,t)\le\mathrm{LCB}(\tilde P_t,t)\) and thus termination with \(\widehat P_t=P^\star\) by Theorem~\ref{thm:aesp-correctness}.
This completes the proof.
\end{proof}

\subsubsection{Lemma \ref{lem:balance-Dt} with Proof}

\begin{lemma}[Greedy variance balancing on a fixed edge set]
\label{lem:balance-Dt}
Let \(D\subseteq E\) be a nonempty set of edges and write \(K\coloneqq|D|\).
For each \(e\in D\), let \(\sigma_e^2>0\) be the sub-Gaussian proxy of the real-data noise on edge \(e\).

Consider the greedy allocation on \(D\) defined as follows:
start from integer counts \(n_e(1)\in\{0,1\}\) (with the convention \(\sigma_e^2/0=+\infty\)) and for
\(s=2,3,\dots\) choose
\[
e_s \in \arg\max_{e\in D}\ \frac{\sigma_e^2}{\,n_e(s-1)\,},
\qquad
\text{and set }\,n_{e_s}(s)=n_{e_s}(s-1)+1,\ \ n_{e}(s)=n_e(s-1)\ \text{for }e\neq e_s.
\]
For each \(t\ge1\), let \(n_e(t)\) be the resulting counts and define the \emph{total} number of real samples
allocated to \(D\) up to time $t$ as
\[
T_D(t)\coloneqq \sum_{e\in D} n_e(t).
\]
Also define
\[
\Sigma_D^2\coloneqq\sum_{e\in D}\sigma_e^2,
\qquad
\Sigma^2\coloneqq\sum_{e\in E}\sigma_e^2.
\]

Then, for every integer \(t\ge1\) such that \(T_D(t)>K\),
\begin{equation}
\label{eq:Dt-alloc-bound}
\sum_{e\in D}\frac{\sigma_e^2}{n_e(t)}
\ \le\
\frac{K\,\Sigma_D^2}{\,T_D(t)-K\,}
\ \le\
\frac{K\,\Sigma^2}{\,T_D(t)-K\,}.
\end{equation}
In particular, if \(T_D(t)\ge 2K\) then
\[
\sum_{e\in D}\frac{\sigma_e^2}{n_e(t)}
\ \le\
\frac{2K\,\Sigma_D^2}{T_D(t)}
\ \le\
\frac{2K\,\Sigma^2}{T_D(t)}.
\]
\end{lemma}

\begin{proof}
For convenience, set \(r_e(s)\coloneqq \sigma_e^2/n_e(s)\) whenever \(n_e(s)>0\), with the convention
\(r_e(s)=+\infty\) if \(n_e(s)=0\).
The greedy rule picks at each step an index \(e_s\in\arg\max_{e\in D} r_e(s-1)\).

Fix $t$ with \(T_D(t)>K\). Under the initialization \(n_e(1)\in\{0,1\}\) and the convention
\(\sigma_e^2/0=+\infty\), the greedy rule samples any edge with \(n_e(s)=0\) before ever re-sampling an edge
with \(n_e(s)>0\). Hence \(T_D(t)>K\) implies \(n_e(t)\ge 1\) for all \(e\in D\), so all ratios below are finite.

\medskip\noindent\textbf{Step 1 (Pairwise dominance at termination).}
Take any two edges \(e,f\in D\) with \(n_f(t)\ge 2\).
We claim that
\begin{equation}
\label{eq:pairwise}
\frac{\sigma_e^2}{n_e(t)}\ \le\ \frac{\sigma_f^2}{\,n_f(t)-1\,}.
\end{equation}
Suppose, towards a contradiction, that
\(\sigma_e^2/n_e(t) > \sigma_f^2/(n_f(t)-1)\).
Let \(s^\star\) be the round when \(f\) was last incremented, so
\(n_f(s^\star-1)=n_f(t)-1\) and \(n_f(s^\star)=n_f(t)\).
Since counts only increase in $t$, we have \(n_e(s^\star-1)\le n_e(t)\), hence
\[
r_e(s^\star-1)
=\frac{\sigma_e^2}{n_e(s^\star-1)}
\ \ge\ \frac{\sigma_e^2}{n_e(t)}
\ >\ \frac{\sigma_f^2}{n_f(t)-1}
=\frac{\sigma_f^2}{n_f(s^\star-1)}
=r_f(s^\star-1).
\]
Thus \(r_e(s^\star-1) > r_f(s^\star-1)\), contradicting the fact that the greedy rule chose
\(f\) (rather than \(e\)) at round \(s^\star\).
Hence~\eqref{eq:pairwise} holds.

\medskip\noindent\textbf{Step 2 (Bounding the maximum ratio).}
Let
\[
\lambda_t\coloneqq \max_{e\in D} \frac{\sigma_e^2}{n_e(t)},
\]
and choose \(e^\star\in D\) such that
\(\lambda_t = \sigma_{e^\star}^2/n_{e^\star}(t)\).
Applying~\eqref{eq:pairwise} with \(e=e^\star\) and any \(f\in D\) with \(n_f(t)\ge 2\) yields
\[
\frac{\sigma_{e^\star}^2}{n_{e^\star}(t)}\ \le\ \frac{\sigma_f^2}{n_f(t)-1},
\]
which we rearrange as
\[
n_f(t)-1\ \le\ \frac{\sigma_f^2}{\sigma_{e^\star}^2}\,n_{e^\star}(t)
\qquad\text{for all }f\in D\text{ with }n_f(t)\ge 2.
\]
For edges with \(n_f(t)=1\) the inequality also holds trivially, since then \(n_f(t)-1=0\).
Therefore, for every \(f\in D\),
\[
n_f(t)-1\ \le\ \frac{\sigma_f^2}{\sigma_{e^\star}^2}\,n_{e^\star}(t).
\]

Summing this inequality over \(f\in D\) gives
\[
T_D(t)-K
=\sum_{f\in D}\bigl(n_f(t)-1\bigr)
\ \le\ \frac{n_{e^\star}(t)}{\sigma_{e^\star}^2}\sum_{f\in D}\sigma_f^2
\ =\ \frac{n_{e^\star}(t)}{\sigma_{e^\star}^2}\,\Sigma_D^2.
\]
Rearranging yields
\[
\lambda_t
=\frac{\sigma_{e^\star}^2}{n_{e^\star}(t)}
\ \le\ \frac{\Sigma_D^2}{T_D(t)-K}.
\]

\medskip\noindent\textbf{Step 3 (Bounding the objective).}
Finally,
\[
\sum_{e\in D}\frac{\sigma_e^2}{n_e(t)}
\ \le\ K\,\lambda_t
\ \le\ \frac{K\,\Sigma_D^2}{T_D(t)-K},
\]
which is the first inequality in~\eqref{eq:Dt-alloc-bound}.
The second inequality in~\eqref{eq:Dt-alloc-bound} follows from \(\Sigma_D^2\le\Sigma^2\).

If \(T_D(t)\ge 2K\), then \(T_D(t)-K\ge T_D(t)/2\), and we obtain
\[
\sum_{e\in D}\frac{\sigma_e^2}{n_e(t)}
\ \le\ \frac{K\,\Sigma_D^2}{T_D(t)-K}
\ \le\ \frac{2K\,\Sigma_D^2}{T_D(t)}
\ \le\ \frac{2K\,\Sigma^2}{T_D(t)},
\]
which yields the ``in particular'' statement.
\end{proof}
\section{Experiment Details}
\label{appx:experiment}
This appendix collects additional implementation details for the numerical experiments.

\subsection{Additional Details on \texttt{METR-LA} and \texttt{PEMS-BAY} \citep{j49q-ch56-25}}

These two datasets originate from California's transportation monitoring systems. \texttt{METR-LA} is collected from the traffic network in Los Angeles, while \texttt{PEMS-BAY} is derived from the Bay Area. Both datasets contain detailed traffic speed records aggregated at a 5-minute frequency. \texttt{METR-LA} covers data from March 1st to June 30th, 2012, across 207 sensors, whereas \texttt{PEMS-BAY} includes 325 sensors from January 1st to May 31st, 2017. In all, each dataset consists of a value matrix, which is in the shape of $[34272, 207]$ for \texttt{METR-LA} and $[52116, 325]$ for \texttt{PEMS-BAY}, and an adjacency matrix describing the edge connections. We defer the detailed data processing and generation steps to Appendix \ref{appx:real_data_gen}.

\subsection{Construction of the Similarity Matrix $W$}

Throughout the experiments, we apply the following constructions of the edge-similarity matrix \(W \in \mathbb{R}^{|E|\times |E|}\).

\begin{itemize}
    \item \textbf{Heat-kernel (diffusion) similarity.} 
    As a smoother alternative, we apply a diffusion process on the line-graph adjacency. Given the 1-hop adjacency \(A\) and its associated Laplacian \(L_{\mathrm{line}} = \mathrm{diag}(A\mathbf{1}) - A\), we define the heat kernel on the line graph as
    \[
        W_{\mathrm{diff}} = \exp\!\bigl(-t\, L_{\mathrm{line}}\bigr),
    \]
    where \(t>0\) controls the diffusion time. For small $t$, this kernel closely reflects the 1-hop structure, while larger $t$ spreads similarity across edges that are further apart in the line graph. After computing the matrix exponential, we symmetrize the result, zero out the diagonal, and threshold entries below a small cutoff (e.g., \(10^{-6}\)) to maintain sparsity. The diffusion kernel provides a continuous analogue of the multi-hop constructions and naturally weighs paths by both their length and multiplicity, making it well suited for graphs with complex or irregular topology. We adopt \(t = 0.5\) in all experiments.
\end{itemize}

We also test the following two hop-based alternatives. However,  empirically, the diffusion-based similarity matrix consistently outperforms these two constructions in our experiments. Thus, all results reported in the main paper use the diffusion kernel for \(W\).

\begin{itemize}
      \item \textbf{1-hop structural similarity.} 
    For two distinct edges \(e=(u,v)\) and \(e'=(u',v')\) in the underlying graph, we set
    \[
        W_{e,e'} = 
        \begin{cases}
        1, & \text{if } \{u,v\} \cap \{u',v'\} \neq \emptyset, \\
        0, & \text{otherwise}.
        \end{cases}
    \]
    Thus \(W_{e,e'} = 1\) whenever \(e\) and \(e'\) share an endpoint, encoding that incident edges are expected to have similar costs. We always set \(W_{e,e}=0\) on the diagonal.

    \item \textbf{2-hop structural similarity.} 
    To capture longer-range dependencies, we augment the 1-hop construction by including edges at distance two in the line graph. Let \(A\) denote the 1-hop adjacency matrix defined above, so that \(A_{e,e'} = 1\) if and only if \(e\) and \(e'\) share a node. The matrix \(A^2\) identifies pairs of edges connected by a 2-hop walk in the line graph, i.e., edges \(e\) and \(e'\) that do not share a node but each share a node with a common intermediate edge. We then define
    \[
        W_{e,e'} = 
        \begin{cases}
        1, & \text{if } A_{e,e'} > 0, \\
        \alpha, & \text{if } A_{e,e'} = 0 \text{ and } (A^2)_{e,e'} > 0, \\
        0, & \text{otherwise},
        \end{cases}
    \]
    where \(\alpha \in (0,1)\) is a decay parameter that discounts 2-hop connections relative to direct (1-hop) adjacency. We use \(\alpha = 0.5\) in all experiments. This construction allows biases to propagate over a wider neighborhood while still prioritizing immediate structural neighbors, improving imputation on edges with sparse or no real data.
\end{itemize}

\subsection{Additional Details on Real-Data Generation}
\label{appx:real_data_gen}

For the METR-LA and PEMS-BAY datasets, we work directly with the original traffic measurements and the provided sensor adjacency information to construct paired real and simulator cost distributions, as well as a road-network topology.

\paragraph{Real and simulated cost distributions.}
For each dataset, we first randomly select a contiguous one-week window of observations per seed. We focus on two time-of-day regimes whose traffic patterns differ substantially: a morning peak (6--9\,am) and an afternoon peak (3--6\,pm). This yields a sample size of 252 for each period. For every sensor \(e\) and each regime, we compute the empirical mean and variance of the recorded measurements over the selected weekly samples. We treat:
\begin{itemize}
    \item the afternoon (3--6\,pm) mean and variance as the real mean cost \(\mu_e\) and the real observation-noise variance, and
    \item the morning (6--9\,am) mean and variance as the simulator mean cost \(\mu'_e\) and the simulator observation-noise variance.
\end{itemize}
This construction induces systematic simulator bias from genuine day-night differences in the real traffic data, without imposing any explicit smoothness assumptions on the bias pattern. 

\paragraph{Reconstructing the road-network topology.}
The METR-LA and PEMS-BAY datasets do not provide a complete road-intersection graph, but they do include a sensor adjacency matrix based on pairwise road distances. We reconstruct an intersection-level graph with one edge per sensor as follows.
\begin{enumerate}
    \item \emph{Sensor-sensor adjacency.}  
    Let \(A^{\text{sens}}\) denote the provided sensor adjacency matrix, which encodes a Gaussian kernel of road distance between sensors. We first symmetrize this matrix and zero out entries below a small threshold (set to \(0.01\)). Two sensors are declared directly adjacent if their (symmetrized) weight exceeds this threshold; these adjacencies define an undirected ``sensor-sensor'' graph.
    \item \emph{Triangular loops (3-cliques).}  
    We enumerate all 3-cliques in the sensor-sensor graph and apply a triplet rule: whenever three sensors are mutually adjacent, we arrange their corresponding road segments to form a triangular loop, so that each pair of sensors in the triplet shares exactly one intersection node.
    \item \emph{Remaining adjacent pairs.}  
    For all remaining directly adjacent sensor pairs that are not part of any 3-clique, we enforce that the two sensors share exactly one endpoint node, so that they form a simple chain in the resulting intersection-level graph.
    \item \emph{Completion of endpoints.}  
    Any remaining free endpoints are then matched to newly created intersection nodes so that each sensor has exactly two endpoints. Finally, for each sensor \(e\), we connect its two assigned endpoints by an undirected edge labeled with that sensor’s index.
\end{enumerate}
The resulting intersection-level graph \(G = (V,E)\) is sparse, respects the sensor-adjacency structure, and contains exactly one undirected edge per sensor. This is the graph on which all experiments for METR-LA and PEMS-BAY are conducted.

\subsection{Implementation Details: Laplacian Calibration}
\label{appx:laplace_sim_impl}

This subsection provides the concrete numerical realization of the estimator in Section~\ref{sec:edge_cost_est} as used in all experiments.

\paragraph{Data-fidelity targets and weights.}
For each edge \(e\), we compute empirical means \(\bar c_e\) and \(\bar c'_e\) and (when available) unbiased sample variances \(s^2_{e}\), \(s^{\prime 2}_{e}\). The discrepancy vector is
\[
y_e=\bar c_e-\bar c'_e .
\]
We implement inverse-variance weights in \eqref{eqn:b_hat} as follows. For edges with \(n_e>0\),
\[
w_e \;=\;\min\!\left\{\frac{1}{\widehat{\mathrm{Var}}(y_e)},\ w_{\max}\right\},\qquad
\widehat{\mathrm{Var}}(y_e)=\widehat{\mathrm{Var}}(\bar c_e)+\widehat{\mathrm{Var}}(\bar c'_e),
\]
where \(\widehat{\mathrm{Var}}(\bar c_e)=s^2_e/n_e\) if \(n_e\ge 2\) and otherwise \(\sigma_e^2/n_e\); analogously for synthetic data. A small floor \(\widehat{\mathrm{Var}}(y_e)\leftarrow \max\{\widehat{\mathrm{Var}}(y_e),10^{-8}\}\) avoids division by zero. We set \(w_e=0\) when \(n_e=0\), so unobserved edges contribute only via Laplacian smoothing. We cap weights at \(w_{\max}=10^6\) for numerical stability.

\paragraph{Solving \((M+\lambda L)\hat b = My\).}
Let \(M=\mathrm{diag}(w)\). The estimator \(\hat b\) is obtained by solving the normal equations implied by \eqref{eqn:b_hat}:
\[
(M+\lambda L)\,\hat b = M y .
\]
Because some similarity components may have no real data, we solve component-wise. Let \(\{C\}\) be the connected components of the similarity graph induced by \(W\). For each component \(C\):
\begin{itemize}
    \item If \(w_e=0\) for all \(e\in C\), the component is uninformed; we set \(\hat b_e=0\) on \(C\), i.e., we fall back to the synthetic anchor.
    \item Otherwise, we restrict to \(C\) and solve \((M_C+\lambda L_C)\hat b_C=M_C y_C\) using a direct sparse linear solver. If the restricted system is numerically rank-deficient, we instead compute a least-squares solution via an iterative Krylov method \citep{li2025krylov}.
\end{itemize}
When \(\lambda = 0\), we skip smoothing and set \(\hat b_e=y_e\) for observed edges and \(0\) for unobserved edges. The calibrated mean estimate is always \(\hat\mu_e=\bar c'_e+\hat b_e\).

\paragraph{\(\lambda\) tuning in simulations.}
Unless a fixed \(\lambda\) is specified, we tune \(\lambda\) over a preset grid ranging from 0 to 100 and scaling in logarithmic order $(0, 0.0001, 0.001,\ldots, 1, 5, 10, 20, 50, 100)$, by minimizing Stein’s Unbiased Risk Estimate (SURE) score. Fix a candidate \(\lambda>0\). Let \(A_\lambda:=M+\lambda L\) and \(\hat b(\lambda)=A_\lambda^{-1}My\) denote the Laplacian-regularized bias estimator. The SURE score is computed as
\[
\mathrm{SURE}(\lambda)
\;=\;
(\hat b(\lambda)-y)^\top M(\hat b(\lambda)-y)
\;+\;
2\,\mathrm{df}(\lambda),
\qquad
\mathrm{df}(\lambda):=\mathrm{tr}\!\bigl(A_\lambda^{-1}M\bigr),
\]
where the first term is the weighted residual fit and \(\mathrm{df}(\lambda)\) is the degrees of freedom of the linear smoother \(\hat b(\lambda)=A_\lambda^{-1}My\). We evaluate \(\mathrm{SURE}(\lambda)\) for each \(\lambda\) in the grid and pick the minimizer \(\hat\lambda\). As the involved experiments in this paper have $|E|$ at a few hundreds, we form a dense representation of \(A_\lambda\) and use a Moore-Penrose inverse to obtain \(A_\lambda^{-1}M\), then take \(\mathrm{tr}(A_\lambda^{-1}M)\) exactly.

\subsection{Implementation Details: Active Learning}
\label{appx:active_esp_impl}

This subsection describes the numerical realization of Algorithm~\ref{alg:a-esp} and its confidence computations.

\paragraph{Weights and bias estimation.}
At round $t$, we form fidelity weights using the negligible-simulator-variance regime (Assumption~\ref{ass:negligible-sim}):
\[
w_e(t)=\kappa_-\,\frac{n_e(t)}{\sigma_e^2},
\]
The discrepancy vector is \(y_e(t)=\bar c_e(t)-\bar c'_e\), and we apply the Laplacian calibration procedure above to obtain \(\hat b(t)\) and \(\hat\mu(t)=\bar c'+\hat b(t)\).

\paragraph{Confidence radii and matrix inverses.}
To compute \(\beta_e(t)\), we implement the smoother
\[
S_\lambda(M_t)=(I+\lambda H_t)^{-1},\qquad
H_t=M_t^{-1/2} L M_t^{-1/2},
\]
where \(M_t=\mathrm{diag}(w(t))\). Since \(|E|\) is at most a few hundred in our datasets, we compute \(S_\lambda(M_t)\) explicitly from a dense representation of \(I+\lambda H_t\), using a Moore-Penrose inverse and standard numerical fallbacks if the matrix is ill-conditioned. This yields stable estimates of \(\alpha_e(\lambda;M_t)=\|S_\lambda(M_t)M_t^{-1/2}e_e\|_2\) and then \(\beta_e(t)\) via equation~(\ref{eq:edge-ci-active}). For unobserved edges with \(w_e(t)=0\), we avoid infinite radii by using a conservative finite fallback based on the smallest observed weight and largest observed \(\alpha_e\); we inflate this fallback by a factor of 10 to preserve correctness of the stopping test under finite precision.

\paragraph{Shortest paths with potentially negative weights.}
Because \(\hat\mu_e(t)\) and the lower-confidence costs \(\hat\mu_e(t)-\beta_e(t)\) may become negative (e.g., due to calibration noise), we explicitly handle negative weights. When all edge costs on a call are nonnegative, we compute shortest paths using a Dijkstra algorithm. If any edge cost is negative, we compute shortest paths using the Bellman--Ford algorithm; if a negative cycle is detected on a route, we conservatively treat the corresponding path cost as \(+\infty\).

\paragraph{Challenger path computation.}
The ``most optimistic challenger'' is computed as the minimum-LCB path distinct from the empirical best path. If all LCB edge costs are nonnegative, we enumerate simple paths in increasing LCB order and take the first path that differs from the empirical best path. If some LCB edge costs are negative, such enumeration is not valid and we avoid shifting all edges by a constant (which can change ordering when paths have different lengths). Instead, we (i) compute the LCB-shortest path via Bellman--Ford; (ii) if it differs from the best path we return it; else (iii) we compute the second LCB-shortest \emph{simple} path using the \(K=2\) case of Yen's algorithm with Bellman--Ford on spur graphs \citep{yen1970algorithm}. If no alternative simple path exists, the challenger defaults to the best path.

\paragraph{Query rule and stopping.}
After updating \(\hat\mu(t)\) and \(\beta(t)\), A-ESP checks the stopping condition \(\mathrm{UCB}(\widehat P_t,t)\le \mathrm{LCB}(\widetilde P_t,t)\). If it fails, we query the globally most uncertain edge,
\[
e_t\in\arg\max_{e\in E}\frac{\sigma_e^2}{n_e(t)}\quad\text{with}\quad \sigma_e^2/0:=+\infty,
\]
where unqueried edges are treated as having infinite ratio.

\end{document}